%% file: main.tex
\documentclass{article}

\usepackage[preprint]{neurips_2025}

\usepackage[utf8]{inputenc} 
\usepackage[T1]{fontenc}    
\usepackage{hyperref}       
\usepackage{url}            
\usepackage{booktabs}       
\usepackage{amsfonts}       
\usepackage{nicefrac}       
\usepackage{microtype}      
\usepackage{xcolor}         
\usepackage{subfig}
\usepackage{wrapfig}
\input{math_commands.tex}

\usepackage{graphicx}
\usepackage{wrapfig}

\usepackage{hyperref}
\usepackage{url}
\usepackage{amsmath}
\usepackage{amsfonts}
\usepackage{bbm}
\usepackage{algorithm}
\usepackage{nicefrac}
\usepackage{amsthm}
\usepackage{enumitem}
\usepackage{multirow}
\usepackage[normalem]{ulem}
\usepackage{algorithmic}

\usepackage{bm}
\usepackage{booktabs}
\input{input}

\usepackage[capitalize,noabbrev]{cleveref}

\usepackage[textsize=tiny]{todonotes}

\title{Sparse Linear Bandits with Blocking Constraints}
\author{Adit Jain \thanks{  Part of this work was done at Adobe Research.} \\
Electrical and Computer Engineering \\
Cornell University\\
Ithaca, NY, 14853 \\
\texttt{aj457@cornell.edu} 
\And Soumyabrata Pal\\
Adobe Research \\
 Bangalore, India \\
\texttt{soumyabratap@adobe.com} 
\And Sunav Choudhary\\
Adobe Research \\
 Bangalore, India \\
\And
Ramasuri Narayanam\\
Adobe Research \\
 Bangalore, India \\
 \And Harshita Chopra\\
University of Washington \\
 Seattle, WA, 98195\\
\And
Vikram Krishnamurthy\\
Electrical and Computer Engineering \\
Cornell University\\
Ithaca, NY, 14853
}

\begin{document}




\maketitle

\begin{abstract}
We investigate the high-dimensional sparse linear bandits problem in a data-poor regime where the time horizon is much smaller than the ambient dimension and number of arms. We study the setting under the additional \textit{blocking constraint} where each unique arm can be pulled only once. The blocking constraint is motivated by practical applications in personalized content recommendation and identification of data points to improve annotation efficiency for complex learning tasks. With mild assumptions on the arms, our proposed  online algorithm (\texttt{BSLB}) achieves a regret guarantee of $\widetilde{\mathsf{O}}((1+\beta_k)^2k^{\frac{2}{3}} \mathsf{T}^{\frac{2}{3}})$ where the parameter vector
has an (unknown) relative tail $\beta_k$ -- the ratio of $\ell_1$ norm of the top-$k$ and remaining entries of the parameter vector. To this end, we show novel offline statistical guarantees of the lasso estimator for the linear model that is robust to the sparsity modeling assumption. Finally, we propose a meta-algorithm (\texttt{C-BSLB})  based on corralling that does not need knowledge of optimal sparsity parameter $k$ at minimal cost to regret. Our experiments on multiple real-world datasets demonstrate the validity of our algorithms and theoretical framework.   
\end{abstract}

\section{Introduction} 

Sparse linear bandits are a rich class of models for sequential decision-making in settings where only a few features contribute to the outcome. They have been applied to various domains such as personalized medicine and online advertising~\citep{bastani2020online,yadkori12}. Recently, the theoretical properties of sparse linear bandits have been explored in data-poor regimes~\citep{hao2020high,efficientsparsehds}.
Formally, consider a time horizon of $\s{T}$ rounds, $\s{M}$ arms, each with a $d$-dimensional feature vector, and an unknown parameter vector with sparsity $k$, where $\numrounds \ll \dimarm \ll \numarms$. In each round, the learner selects an arm and observes a noisy reward whose expectation is the inner product between the parameter vector and the arm’s feature vector. The objective is to design an algorithm that sequentially selects arms to maximize cumulative reward.
Notably, ~\cite{hao2020high} established tight, dimension-free regret guarantees of $\orderof(\s{T}^{2/3})$ in this setting.

In this work, we study a more practical variant of sparse linear bandits that better models real-world, data-poor scenarios. Our first key contribution is the introduction of a novel constraint to the sparse bandit framework: each unique arm may be pulled only once, a restriction we refer to as the \textit{blocking constraint}. Second, we address the challenge of \textit{model misspecification} by providing regret guarantees that depend on how close the true parameter vector is to being $k$-sparse, thus broadening the applicability of our results beyond exactly sparse settings. Finally, unlike prior approaches, we develop an efficient algorithm that does not require prior knowledge of the sparsity parameter $\sparsityk$.

 The blocking constraint closely models the data acquisition process in data-poor regimes. For instance, consider personalized applications such as movie/book recommendations on edge devices. Users typically consume an item and rate it only once. There is hardly any point in recommending an item that has been previously consumed, at least not immediately - this is captured via the \textit{blocking constraint}. Existing theoretical approaches in recommendation literature~\citep{bresler2014latent, bresler2016collaborative, heckel2017sample, huleihel2021learning, pal2024blocked} use collaborative filtering (CF) via multiple users to address the blocking constraint. However, CF might affect user privacy or be limited on edge devices. Unlike CF-based methods, which rely on multiple users and pose privacy concerns, our approach applies to a single-user setting while still addressing the blocking constraint.

 A second application of our framework, in settings where labeling is expensive, is to create example banks comprising of \textit{high quality hard datapoints} for in-context learning via large language models (LLMs).
LLMs offer strong zero-shot capabilities, making it easier to prototype solutions for downstream tasks. However, they struggle with complex domain-specific queries, particularly when relevant training data is poor or evolving~\cite{farr2024llm}. In-context learning with a few-shot examples has emerged as a powerful approach, where a small set of high-quality examples improves model performance~\citep{dong2022survey}. 
 Crucially, hard examples provide better domain-specific information~\citep{baek2024revisiting, liu2024let, mavromatis2023examples}, but identifying them is challenging. Heuristic-based selection often leads to noisy, mislabeled, or outlier examples~\citep{mindermann2022prioritized}. 
 Alternatively, we can leverage domain experts to assign a hardness score while annotating. This data-poor problem can be framed in a bandit framework, where unlabeled datapoints act as arms and are sequentially annotated while hardness scores are modeled as a sparse linear function of embeddings. 
In domains with very few annotators—sometimes only one—it is impractical to re-query the same datapoint, naturally leading to a \textit{blocking constraint}. 

\noindent \textbf{Overview of our Techniques and Contributions. }  We propose \texttt{BSLB} (Blocked Sparse Linear Bandits), an efficient algorithm in our framework which is primarily an Explore-Then-Commit algorithm. In the exploration period, we sample a set of unique arms from a carefully chosen distribution and observe their rewards - the goal is to ensure that the expected covariance matrix of the sampled arms has a large minimum eigenvalue.
Such a well-conditioned covariance matrix ensures that the confidence ball of the estimate shrinks in all directions.
In the exploitation period, we use the Lasso estimator to estimate the unknown model parameters. The optimal exploration period depends on the correct sparsity level of the unknown parameter vector, which is difficult to set in practice. Therefore, we also present a meta-bandit algorithm based on corralling a set of base bandit algorithms \citep{agarwal2017corralling} - obtaining the same order-wise regret guarantees but without needing the knowledge of sparsity hyperparameter. Below, we summarize our \textit{main contributions}:    
\begin{enumerate}[noitemsep,nolistsep, leftmargin=*]
    \item We propose the high-dimensional sparse linear bandits framework with the novel blocking constraint to closely model sequential decision-making tasks in a data-poor regime. Importantly, the time horizon is much smaller than the ambient dimension. We also account for model misspecification where the unknown parameter vector is \textit{close to being sparse}  with relative tail magnitude $\restrictedconeconstant$ ($\ell_1$ norm of bottom $d-k$ entries divided by $\ell_1$ norm of top $k$ entries) at sparsity level $k$.
    \item We propose a computationally efficient ``explore then commit'' (ETC) algorithm \texttt{BSLB} for regret minimization in our framework (Theorem \ref{th:regret}) that achieves a regret guarantee of $\widetilde{\mathsf{O}}((1+\restrictedconeconstant)^2k^{\frac{2}{3}} \mathsf{T}^{\frac{2}{3}})$ for fixed known $k$ but unknown $\restrictedconeconstant$ - that is, the relative tail magnitude remains unknown to the algorithm.
    For the special case of exact sparsity that is $\restrictedconeconstant=0$, \texttt{BSLB} achieves a regret guarantee of $O(k^{2/3}\s{T}^{2/3})$ (Corollary \ref{coro:hard}) which is also tight in our setting (Theorem \ref{thm:lower}).  
    \item \texttt{BSLB} requires knowledge of the sparsity level $k$ (which also controls the tail magnitude $\restrictedconeconstant$) to set the length of the exploration period. This is challenging to establish in practice, especially for vectors that are not exactly sparse. Hence, we propose a meta-algorithm \texttt{C-BSLB} that combines base algorithms with different exploration periods. \texttt{C-BSLB} achieves same regret guarantees order-wise (Theorem \ref{th:corralling}) as \texttt{BSTB} but without knowing the optimal sparsity level.     
\end{enumerate}
To validate our theoretical results, we conduct experiments on both synthetic and real-world datasets. For personalized recommendations, we use MovieLens, Jester, and Goodbooks. For intelligent annotation in data-poor settings, we evaluate our method on the PASCAL VOC 2012 image classification dataset and the SST-2 text classification dataset. Our algorithm, \texttt{BSLB}, outperforms baselines in all tasks. Due to space constraints, detailed experiments are provided in Appendix~\ref{app:numerical}.

\textbf{Algorithmic Novelty:}
We consider the \textit{naive modification of the algorithm proposed in}~\citet{hao2020high} and illustrate why the blocking constraint makes the bandit instance non-trivially hard.
\newcommand{\smallnumberofarms}{l}
\label{disc:naive} 
Their \texttt{ESTC}  algorithm in the sparse linear bandits framework provides tight regret guarantees of $O(k^{2/3}\s{T}^{2/3})$ when the parameter vector is exactly $k$-sparse. Intuitively, one might consider modifying \texttt{ESTC} to incorporate the blocking constraint. In \texttt{ESTC}, a crucial initial step is to compute a distribution over the arms with the objective of maximizing the minimum eigenvalue of the expected covariance matrix. Subsequently, during the exploration period, arms are sampled from this computed distribution. It can be shown that the minimum eigenvalue of the sample covariance matrix is close to that of the expected one, ensuring that the sample covariance matrix is well-conditioned.

A naive modification to incorporate the blocking constraint involves using rejection sampling. If a previously pulled arm is sampled again, it is ignored and re-sampled from the same distribution. However, this significantly modifies the distribution over arms.
\begin{wrapfigure}{r}{0.5\textwidth}
    \centering
    \includegraphics[width=\linewidth]{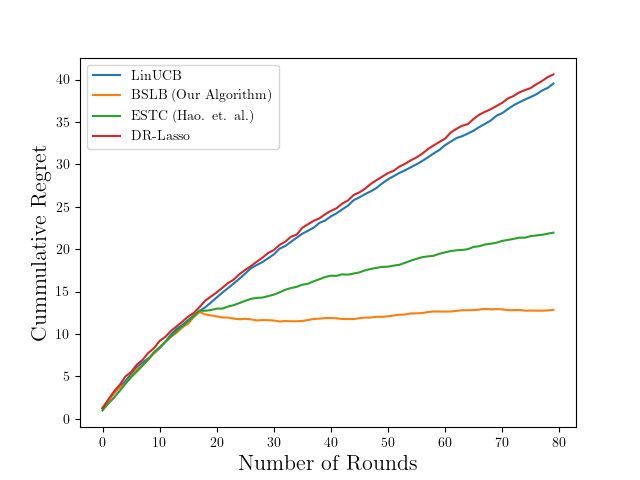}
   \caption{\small Simulation illustrating performance gap between our proposed algorithm \texttt{BSLB} and naive extensions of \texttt{ESTC}, LinUCB and DR-Lasso to incorporate blocking constraint. We consider an instance with $\numarms = 500$ arms ($\smallnumberofarms=5$ arms of unit norm and remaining arms $\ell_2$ norm of 0.5), $\dimarm = 100$, $\numrounds=80$ and $k=5$.}
    \label{fig:benchmarkbasealg}
    \vspace{-4mm}
\end{wrapfigure}
For instance, if the computed distribution in \texttt{ESTC} concentrates probability mass on a few $\smallnumberofarms$ arms (e.g., if $\smallnumberofarms$ arms have significantly higher norm than others), the blocking constraint forces us to sample arms with low probability mass once those $\smallnumberofarms$ arms are exhausted. This breaks the statistical guarantees proved for \texttt{ESTC} regarding the lasso estimator. We empirically illustrate and validate this special setting in Figure~\ref{fig:benchmarkbasealg} where our algorithm is also benchmarked with similar naive extensions of DR-Lasso and LinUCB~\citep{drlassobandits,NIPS2011_e1d5be1c} - the latter ones do not handle sparsity in any case. Hence, a naive modification of \texttt{ESTC} does not work, and the objective itself needs to incorporate the blocking constraint. However, this leads to a non-concave maximization problem, resulting in technical challenges. Finally, the authors in \cite{hao2020high} do not consider robustness to sparsity or unknown sparsity level, both of which are important to model in practice.


\textbf{Technical Challenges. } 
(A) \textit{Statistical guarantees:} Both in our setting and in \citep{hao2020high}, no assumptions are made on the arm vectors, unlike much of the existing sparsity literature, which requires the Gram matrix to satisfy the Restricted Isometry Property (RIP) \citep{boche2015survey}. A weaker parameterized condition, the Restricted Eigenvalue (RE) property, has been shown to provide strong statistical guarantees for the Lasso estimator in sparse linear regression \citep{bickel2009simultaneous,reconstructionitit}.  Specifically, if the expected covariance matrix of the arm vectors satisfies RE with parameter $K$, then with high probability the Gram matrix of sufficiently large i.i.d. samples will also satisfy RE with parameter $K/2$.
However, to the best of our knowledge, existing guarantees for the Lasso estimator hold only for exactly sparse parameter vectors when the Gram matrix of sampled arms just satisfies RE. In contrast, offline guarantees for soft sparse parameter vectors (where there is a non-zero tail) require the stronger RIP condition \citep{Wainwright_2019,boche2015survey}. Theorem \ref{lemma:fastsparselearning} of this paper bridges the gap and provides general statistical guarantees which (a) holds under the blocking constraint (sampling without replacement from a distribution on arms) (b) holds for an arm set that satisfies only the RE condition, and (c) for a parameter vector that is close to being sparse with a non-zero tail.

(B) \textit{Sampling in Exploration phase:}  In linear bandits with arbitrary arms, uniform random sampling during exploration fails to obtain a well-conditioned covariance matrix when arms are distributed non-uniformly. Moreover, as discussed earlier, a naive adaptation of \texttt{ESTC}'s sampling technique does not work. To address this, we optimize the probability distribution over arms to maximize the minimum eigenvalue of the expected covariance matrix while incorporating the blocking constraint. However, this leads to a subset selection problem, where the distribution must assign equal probability mass to a subset of arms while excluding others. The resulting optimization problem is discrete and non-concave. To tackle this, we introduce a concave relaxation of the objective and propose a randomized rounding procedure that yields a feasible solution with strong approximation guarantees on the minimum eigenvalue of the covariance matrix.

(C) \textit{Knowledge of hyper-parameters:} Our first proposed algorithm \texttt{BSLB} (similar to \texttt{ESTC}) requires as input a fixed sparsity $\sparsityk$. However, the algorithm does not know the tail magnitude $\beta_k$ at that sparsity level $k$. Importantly, if $k$ is set too low or too high, then the regret guarantee will not be reasonable. However, it is challenging to set the sparsity parameter without knowing the parameter vector itself. To resolve this challenge, we use a corralling algorithm based on the techniques of \citep{agarwal2017corralling} that combines several base algorithms and provides guarantees with respect to the optimal one. However, a naive application leads to an additional corralling cost with a linear dependence on dimension $\dimarm$, making the regret vacuous ($\numrounds \ll \dimarm$). We use a covering argument - we only choose a small representative subset (covering set) of base algorithms for input to corralling and show that for all remaining base algorithms, their regret is close enough to a base algorithm in the covering set. 

\textbf{Notation. }
\newcommand{\vectorb}[1]{\textbf{#1}}
We denote vectors by bold small letters (say $\vectorb{x}$), scalars by plain letters (say $x$ or $\s{X}$), sets by curly capital letters (say $\ca{X}$) and matrices by bold capital letters (say $\fl{X}$). We use $[m]$ to denote the set $\{1,2,\dots,m\}$, $\Vert \fl{x}\Vert_p$ to denote the $p$-norm of vector $\vectorb{x}$. For a set $\ca{T}$ of indices $\fl{v}_{\ca{T}}$ is used to denote the sub-vector of $\fl{v}$ restricted to the indices in $\ca{T}$ and $0$ elsewhere. $\lambda_{\min}(\fl{A})$ denotes the minimum eigenvalue of the matrix $\fl{A}$ and $\diag(\vectorb{x})$ denotes a diagonal matrix with entries such as $\vectorb{x}$. We use $\ca{B}^d$ and $\ca{S}^{d-1}$ to denote the unit ball and the unit sphere in $d$ dimensions, respectively. We write $\bb{E}X$ to denote the expectation of a random variable $X$. $\widetilde{O}(\cdot)$ notation hides logarithmic factors in $\s{T},\numarms,d$.

\newcommand{\armsetexploit}{\mathcal{D}}
\newcommand{\comments}[1]{{\color{blue} \ //  \ {#1}}}
\section{Problem Formulation}\label{sec:problem_formulation}
Consider a set (or more generally, a multi-set) of $\numarms$ arms $\armset \equiv \{\arm^{(1)}, \arm^{(2)}, \dots, \arm^{(\numarms)}\} \subseteq \R^\dimarm$. Let $\arm^{(\armidx)}\in\unitball$ denote the $\dimarm$-dimensional unit norm vector embedding associated with the $\armidx^{\s{th}}$ arm. We make the same mild boundedness assumption as~\citep{hao2020high}, that is, for all $\arm \in \armset$,  $\|\arm\|_\infty \leq 1$. We have a time horizon of $\numrounds$ rounds. 
In our high-dimension regime, we have $\numrounds \ll \dimarm \ll \numarms$; that is, the horizon is much smaller than the ambient dimension, which, in turn, is significantly smaller than the number of arms. This is a suitable model for the data acquisition process in data-poor regimes. At each round $\timeindex\in [\numrounds]$, an arm $\arm_\timeindex$  which has not been pulled in the first $\timeindex-1$ rounds is selected by the online algorithm (decision-maker). Note that such a selection mechanism respects the blocking constraint. 
Subsequently, the arm $a_t$ is pulled and the algorithm observes the stochastic reward $\reward_\timeindex$. 
We model the expected reward $\bb{E}r_t$  as a linear function of the arm embedding where $\parameter\in \R^\dimarm$ is an unknown parameter vector. In particular, the random variable $\reward_\timeindex$ is generated according to
 $\reward_\timeindex = \langle\parameter, \arm_\timeindex \rangle + \noise_\timeindex$ where $\{\noise_\timeindex\}_{t\in [T]}$ are zero-mean sub-gaussian random variables independent across rounds with variance proxy bounded from above by $\sigma^2$. For any sparsity level $\sparsityk \leq \dimarm$ we define the tail of the (unknown) parameter vector $\parameter$ as,  $\sparsitytail_\sparsityk:=  \frac{\Vert\parameter_{\topt_\sparsityk^\complement}\Vert_1}{\Vert\parameter_{\topt_\sparsityk}\Vert_1}$
where $\topt_\sparsityk$ denotes the set of $\sparsityk$ largest coordinates of $\parameter$ by absolute value and $\topt_\sparsityk^\complement=[d]\setminus \topt_\sparsityk$. Note that $\beta_k$ is unknown to the algorithm.

 Now, we formally define the regret objective in our online learning set-up, which also respects the \textit{blocking constraint}. Our regret definition captures the difference in the cumulative expected reward of arms selected by the online algorithm versus the cumulative expected reward of the $\numrounds$ unique arms with the highest mean rewards.
 Consider a permutation $\permutation:[\left|\armset\right|]\rightarrow [\left|\armset\right|]$ of arms such that for any $\armidxalt<\armidx$, we have $\langle \parameter, \arm^{(\permutation(\armidxalt))}\rangle \ge \langle\parameter,\arm^{\permutation(\armidx)}\rangle$. We can define the regret $\regret$ for our setting as,
 
\begin{align}\label{eq:regretcumm}
    \regret:= \sum_{t=1}^{\numrounds} \langle\parameter,\arm^{(\permutation(t))}\rangle- \sum_{t=1}^{\numrounds} \langle\parameter,\fl{a}_t\rangle.
\end{align}
 
We aim to design an algorithm that minimizes expected regret $\expectedregret$ in our setting where the expectation is over the randomness in the algorithm.
 
\begin{algorithm}
     \begin{algorithmic}[1]
        \REQUIRE Arms $\mathcal{A}$, time horizon $\numrounds$, Exploration Budget $\numroundsexplore$, Regularization Parameter $\lagrange$, Subset selection parameter $\searchbound$
        \STATE $\armsubset=$  \textsc{GetGoodSubset}$({\mathcal{A},\searchbound})$
        \comments{Compute good subset of arms}
        \FOR{$\timeindex \in [\numroundsexplore]$}
        \STATE Sample uniformly from $\arm_{\timeindex} \sim \armsubset$ and get reward $\reward_{\timeindex}$ 
        \STATE $\armsubset \gets \armsubset\setminus \{\arm_{\timeindex}\}$ \comments{Remove selected arm from remaining arms}
        \ENDFOR{}
        \STATE $\parameterestimate = \arg\min_{\parameter} \sum_{\timeindex \in [\numroundsexplore]}(\reward_\timeindex - \langle\parameter, \arm_\timeindex\rangle)^2 +  \lagrange||\parameter||_1$  
        \comments{Compute estimate using LASSO}
        \STATE $\armsetexploit = \armset \setminus \{\arm_\timeindex| \ \timeindex \in [\numroundsexplore]\}$ \comments{Arms available for exploit phase}
        \FOR{$\timeindex \in [\numroundsexplore+1,\numrounds]$} 
\STATE $\arm_{\timeindex} = \arg\max_{\arm \in \armsetexploit} \langle \parameterestimate, \arm \rangle$ 
        \STATE $\armsetexploit = \armsetexploit\setminus \{\arm_{\timeindex}\}$ 
        \ENDFOR
    \end{algorithmic}
    \caption{Explore then Commit for Blocked Sparse Linear Bandits (\texttt{BSLB})}
    \label{alg:greedy}
\end{algorithm}
\begin{algorithm}
    \begin{algorithmic}[1]
        \REQUIRE Set of Samples $\armset$, Subset selection parameter $\searchbound$
        \STATE \textbf{Output: } Subset $\armsubset$
        \STATE  Maximize the objective function defined in  (\ref{opt:relaxed}) with input  $\searchbound$ to obtain distribution $\hat{\samplingdist}$ over $\armset$.
       \FOR{$\armidx \in [\numarms]$} 
       \STATE $\armsubset = \armsubset  \cup \{\arm^{(\armidx)}\}$ with probability $\searchbound\hat{\samplingdist}_\armidx$ \comments{ Add sample $\armidx$ to $\armsubset$ with prob. $\searchbound \cdot \hat{\samplingdist}_\armidx$}
       \ENDFOR
       \STATE $\bar{\armsubset} = $ \texttt{SubsetSearch}($\armset,\dimarm,\searchbound$) \comments{Algorithm~\ref{alg:bruteforce} in Appendix~\ref{sec:bruteforce}}
        \STATE \textbf{return} $\argmax_{\mathcal{H} \in \armsubset,\bar{\armsubset}} \mineigvalue(\mathcal{H})$
    \end{algorithmic}
    \caption{GetGoodSubset: Subset selection for maximizing the minimum eigenvalue }
    \label{alg:getgoodsubset}
\end{algorithm}

\section{Our Algorithm and Main Results}\label{sec:algdesc}

\paragraph{Description of BSLB Algorithm. }Our main contribution is to propose an Explore-Then-Commit (ETC) algorithm named \texttt{BSLB} which is summarized in Algorithm~\ref{alg:greedy}. \texttt{BSLB} takes as input a set of arms $\ca{A}$, the time horizon  $\s{T}$, the exploration budget $\numroundsexplore$, subset selection parameter $\searchbound$ and the regularization parameter $\lagrange$. Steps 1-6 of \texttt{BSLB} correspond to the exploration component in the algorithm. In Step 1, we first compute a good subset of arms $\armsubset\subset \armset$ (using the function $\textsc{GetGoodSubset}(\armset,\searchbound)$) which comprises representative arms that cover the $\dimarm$-dimensional space reasonably well. Subsequently, in Steps 2-5, we sample arms without replacement from the set of arms $\armsubset$ for $\numroundsexplore$ rounds. The goal in the exploration component is to select a subset of arms such that the image of sparse vectors under the linear transformation by the Gram matrix of the selected set has a sufficiently large magnitude (see Definition \ref{def:re}). As we prove, such a result ensures nice statistical guarantees of the parameters estimated using observations from the subset of arms pulled until the end of the exploration phase. Since the set of arms $\armset$ can be arbitrary, note that sampling arms uniformly at random from the entire set might not have good coverage - especially when most arms are concentrated in a lower-dimensional subspace. Therefore, finding a good representative subset of arms leads to the following discrete optimization problem 
\begin{align}\label{eq:discreteopt}
   \mineigvalue^* := \underset{{\armsubset' \subseteq \armset}}{\max} \ \mineigvalue\left(|\armsubset'|^{-1}\sum_{\arm \in \armsubset'} \arm \arm^{\s{T}} \right).
\end{align}
\textit{Algorithm~\ref{alg:getgoodsubset}:} The function $\textsc{GetGoodSubset}(\armset,\searchbound)$ approximates the solution to this computationally intractable discrete optimization. We maximize a relaxed concave program in~\eqref{opt:relaxed} efficiently for a chosen input parameter $\searchbound$ to obtain a distribution $\hat{\samplingdist}$ on the set of arms $\ca{A}$ - subsequently, we construct the subset $\armsubset$ using randomized rounding (Step 4 of Algorithm~\ref{alg:getgoodsubset}) with $\hat{\samplingdist}$ to obtain a feasible solution to \eqref{eq:discreteopt}.  In addition, in Step 6, we include a search over all subsets of size $\orderof(\dimarm)$ and identify the one having the highest minimum eigenvalue. Step 6 allows us to achieve strong theoretical results but is computationally expensive - in practice, it can be skipped. As we explain in Remark~\ref{rem: point on complexity and tradeoff}, skipping Step 6 leads to a slightly weaker regret guarantee, but  significantly reduces computational complexity making Algorithm~\ref{alg:getgoodsubset} highly efficient.

In Step 7 of Algorithm~\ref{alg:greedy}, we use the Lasso estimator to get an estimate $\parameterestimate$ of the unknown parameter vector $\parameter\in \bb{R}^d$. Note that the number of samples used in obtaining the estimate $\parameterestimate$ is much smaller than the dimension $d$.  The second part of \texttt{BSLB} (Steps 8-11) corresponds to the exploitation component of the algorithm, we pull arms that are predicted to be the most rewarding according to our recovered estimate $\parameterestimate$.
At every round in \texttt{BSLB}, no arm is pulled more than once, thus respecting the \textit{blocking constraint}.
It is important to note that \texttt{BSLB} is two-shot. We change our data acquisition strategy only once after the exploration component - thus making our algorithm easy to implement in practice.  

Next, we move on to our main theoretical results.  In Section \ref{subsec:offline}, we provide offline guarantees of the Lasso estimator for sparse linear regression that are robust to sparsity and hold only with weak conditions. In Section \ref{subsec:subset}, we analyze an efficient algorithm for computing the sampling distribution in the exploration phase of \texttt{BSLB}. In Section \ref{subsec:online}, using the offline and approximation guarantees, we provide regret guarantees for our proposed algorithm \texttt{BSLB}.  In Section \ref{subsec:corralling}, we discuss guarantees on \texttt{C-BSLB} (Algorithm~\ref{alg:corral}) that does not need the knowledge of the sparsity level. 


\subsection{Offline Lasso Estimator Guarantees With Soft Sparsity and RE condition}\label{subsec:offline}
To the best of our knowledge, there does not exist in the literature offline guarantees for sparse linear regression that is (A) robust to sparsity modeling assumption and (B) holds only under the mild RE condition on the Gram matrix. Our first theoretical result fills this gap to a certain extent with an upper bound on error rate.
We will start by introducing the definition of Restricted Eigenvalue (RE) 
\begin{definition}\label{def:re} \textbf{Restricted Eigenvalue (RE): }
$\designmatrix\in\R^{\numsamples\times\dimarm}$ satisfies  Restricted Eigenvalue property $\RE(\sparsityk_0,\restrictedconeconstantvanilla,\designmatrix)$, if there exists a constant $\restrictedeigenvalue(\sparsityk_0,\restrictedconeconstantvanilla,\designmatrix)$ such that for all $\vectorsymalt \in \R^\dimarm$ and $\vectorsymalt\neq \fl{0}$,
\begin{align*}
0<{\restrictedeigenvalue(\sparsityk_0,\restrictedconeconstantvanilla,\designmatrix)} = \min_{J \subseteq \{1,\dots,\dimarm \} |J| \leq \sparsityk_0}\min_{\Vert \vectorsymalt_{J^\complement} \Vert_1 \leq \restrictedconeconstantvanilla
\Vert \vectorsymalt_{J} \Vert_1}\frac{\Vert \designmatrix \vectorsymalt \Vert_2 }{\Vert \vectorsymalt_J \Vert_2}.
\end{align*}
   \end{definition}
\citet{bickel2009simultaneous} showed that RE is among the weakest conditions imposed in the literature on the Gram matrix to ensure nice statistical guarantees on the Lasso estimator for sparse linear regression. We now state our first main result:


\begin{theorem}\label{lemma:fastsparselearning}
Let $\designmatrix \in \R^{\numsamples\times\dimarm}$ be the data matrix  satisfying $|\designmatrix_{ij}| \leq 1 \forall i,j$. Let $\rewardvector \in \R^{\numsamples}$ be the corresponding observations such that $\rewardvector = \designmatrix\parameter+\f{\noisevector}$, where $\f{\noisevector} \in \R^\numsamples$ is a zero-mean sub-gaussian random vector with i.i.d. components having bounded variance $\sigma^2=O(1)$. 
Let $\parameter$ have an (unknown) relative tail $\sparsitytail_\sparsityk$ at fixed sparsity level $\sparsityk$. Suppose $\designmatrix$ satisfies restricted eigenvalue property (Def.~\ref{def:re})  $\RE(\sparsityk,4(1+\restrictedconeconstant),\frac{\designmatrix}{\sqrt{\numsamples}})$ with constant $\restrictedeigenvalue>0$.
  An estimate $\parameterestimate$ of $\parameter$ recovered using Lasso (Line 9 in \texttt{BSLB}) with a regularization parameter $\lagrange=\sqrt{\frac{\log \dimarm}{\numsamples}}$, satisfies following with probability $1-o(d^{-2})$,
    \begin{align}\label{eq:oracleinequality}
    \begin{split}
       \Vert{\parameter-\parameterestimate}\Vert_1 =  \widetilde{\orderof} \left(\sparsityk(1+\restrictedconeconstant)^2\restrictedeigenvalue^{-2}\numsamples^{-1/2} \right).
       \end{split}
    \end{align} 
\end{theorem}

Due to space constraints, the proof of Theorem \ref{lemma:fastsparselearning} is deferred to Appendix~\ref{app:sparsereg}. 
\begin{insight}
Note that in Equation \ref{eq:oracleinequality}, for a fixed sparsity $k$, the estimator error guarantee decays with the number of samples $\numsamples$ (at the rate of $n^{-1/2}$) and RE constant $\restrictedeigenvalue$, while growing linearly with $\sparsityk$. 
 Existing error guarantees in literature focus only on exact sparse vectors $\f{\theta}$ - the data matrix $\designmatrix$ satisfies $\RE(\sparsityk,3,\frac{\designmatrix}{\sqrt{\numsamples}})$ with constant $\restrictedeigenvalue^{'}$ and $\restrictedconeconstant=0$ (see Theorem 7.13 ~\citet{Wainwright_2019}). However, with moderately stronger assumption of $\RE(\sparsityk,6,\frac{\designmatrix}{\sqrt{\numsamples}})$ on the data matrix, guarantees of  Theorem~\ref{lemma:fastsparselearning} hold for all $\restrictedconeconstant \leq 1/2$. As stated in Theorem \ref{lemma:fastsparselearning}, for a larger tail with $\restrictedconeconstant>1/2$, $\designmatrix$ needs to satisfy $\RE$ on a larger cone of vectors. 
\end{insight}
\begin{remark}
    Note that the statistical guarantee presented in Theorem \ref{lemma:fastsparselearning} is an offline error rate that is \textit{robust to sparsity modeling assumption} - similar to Theorem 7.19 in~\citet{Wainwright_2019} and Theorem 1.6 in~\citet{boche2015survey}. However, the former holds only for the special case when $\designmatrix$ has i.i.d. Gaussian rows, and the latter requires the stronger RIP condition on the data matrix. Our error guarantee is much more general and holds for deterministic data matrices $\fl{X}$ satisfying RE.
\end{remark}
\textit{Proof Outline: }The lasso inequality states, $ \frac{1}{\numsamples} \Vert \designmatrix \approximationerror \Vert_2^2 \leq  \frac{\errorvector^\transpose \designmatrix \approximationerror}{\numsamples} + \lagrange\|\approximationerror\|_1$.
   Using the RE condition, we show that the norm squared of the approximation error ($\approximationerror = \parameter-\parameterestimate$) is, ${\restrictedeigenvalue^2} \Vert \approximationerror_{\topt_k} \Vert_2^2 =  \orderof(\Vert \approximationerror\Vert_{1}\Vert \frac{ \designmatrix^\transpose\errorvector }{\numsamples}\Vert_{\infty})$.  Further by decomposing the approximation we show that $\restrictedeigenvalue^2\Vert \approximationerror\Vert_{1} = \orderof(\Vert \frac{ \designmatrix^\transpose\errorvector }{\numsamples}\Vert_{\infty}\sparsityk(1+\restrictedconeconstant)^2)$. Finally, we apply a concentration inequality for the noise vector $\noise$, $\Vert \frac{ \designmatrix^\transpose\errorvector }{\numsamples}\Vert_{\infty}$ scales as $\orderof(\frac{1}{\sqrt{\numsamples}})$. 
   
 Below we derive a corollary for the case when the rows of the design matrix are sampled without replacement from a set whose empirical covariance matrix has a minimum eigenvalue. 
\begin{coro}\label{coro:empiricallasso}
        Let $\designmatrix \in \R^{\numsamples\times\dimarm}$ be the data matrix with $\numsamples$ samples and dimension $\dimarm$, whose rows are sampled uniformly without replacement from a set $\armsubset\subset \bb{R}^d$. Let   $\mineigvalueempirical= \lambda_{\min}(\left|\armsubset\right|^{-1}\sum_{\fl{a}\in \armsubset}\arm\arm^{\transpose})$. Consider the same setup for observations $\rewardvector$ as in Theorem \ref{lemma:fastsparselearning}. 
        Provided $\numsamples = \Omega(\sparsityk\mineigvalueempirical^{-4})$, an estimate $\parameterestimate$ of $\parameter$ recovered using Lasso (Line 9 in \texttt{BSLB}), will satisfy with probability $1-\exp(-\Omega(\numsamples))-o(d^{-2})$,
    \begin{align}
\begin{split}
   \Vert{\parameter-\parameterestimate}\Vert_1 =  \widetilde{\orderof} \left(\sparsityk(1+\restrictedconeconstant)^2\mineigvalueempirical^{-1} \numsamples^{-1/2} \right).
       \end{split}
    \end{align} 
\end{coro}
\begin{remark}
    In comparison to the existing statistical guarantees (Theorem A.3 of ~\citet{hao2020high} which is a restatement of Theorem 7.13 of~\citet{Wainwright_2019}), the multiplicative factor $(1+\restrictedconeconstant)^2$ arise due to non-zero tail of the parameter vector and reduce to $1$ for the exact sparsity case ($\restrictedconeconstant=0$). 
   Note in particular that we do not have the RE assumption in Corollary \ref{coro:empiricallasso}. Instead, it is replaced by a lower bound on $n$ - arms sampled without replacement from the set $\ca{G}$ whose Gram matrix has a sufficiently large minimum eigenvalue. This is possible because a lower bound on minimum eigenvalue for a positive semi-definite matrix implies a lower bound on RE with arbitrary parameters - concentration guarantees imply that the RE condition remains satisfied when sufficiently large (yet smaller than $\left|\ca{G}\right|$) number of samples are sampled from $\ca{G}$. 
\end{remark}




\subsection{Subset Selection for Maximizing the Minimum Eigenvalue}\label{subsec:subset}

Recall that in Step 4 of \texttt{BSLB}; we sample from a carefully chosen subset of arms that has good coverage - more precisely, our goal is to solve the optimization problem in equation \ref{eq:discreteopt} to obtain a representative set of arms. Although \citet{hao2020high} had a similar objective, the absence of blocking constraint in their framework implied that they could solve for a probability distribution on the set of arms such that the minimum eigenvalue of the expected covariance matrix is maximized. Since their solution space was the probability simplex, the objective was continuous and concave - implying that a solution can be found efficiently.

However, in our setting, due to the blocking constraint, we need to identify a subset of representative arms from which to sample uniformly at random without replacement in the exploration component - this leads to the objective in \eqref{eq:discreteopt} being discrete and therefore non-concave. Note that a brute force solution to our objective implies a search over all subsets of $[\numarms]$ and will take time $\Omega(\exp(\numarms))$.
To design an efficient algorithm for obtaining a good feasible solution to the non-concave objective in \ref{eq:discreteopt}, our first step is to obtain a concave relaxation as described in equation \ref{opt:relaxed} - in particular, instead of optimizing over a subset, we optimize over probability distributions over the set of arms such that the probability mass over any arm is bounded from above. 
\begin{align}\label{opt:relaxed}
\begin{split}
      \hat{\samplingdist}(\searchbound) = &\underset{{\samplingdist \in \probabilityspace(\armset)}}{\arg\max} \ \mineigvalue\left(\armmatrix\diag({\samplingdist}) \armmatrix^{\s{T}}\right)  \ \  \text{ such that } \ \Vert\samplingdist\Vert_{\infty} \leq \frac{1}{\searchbound},
\end{split}
\end{align}
Note that  $\armmatrix = [\arm^1,\dots,\arm^\numarms]^\transpose \in \bb{R}^{\numarms\times\dimarm}$ denotes the matrix with all arms and $\searchbound$ is an additional parameter to the relaxed objective. Since the solution to equation \ref{opt:relaxed} might not be a feasible one for equation \ref{eq:discreteopt}, we use a randomized rounding (Step 19 in \texttt{BSLB}) procedure to obtain a feasible solution. In the randomized rounding procedure, each arm $j\in[\numarms]$ is sampled into our feasible output set $\ca{G}$ (used in the exploration component) independently with probability $\searchbound\samplingdist_j$. 
Let $\ca{X}\subseteq \ca{A}$ be the optimal subset for which the RHS in Equation \ref{eq:discreteopt} is maximized and let  $\mineigvalue^*$ be the corresponding objective value (minimum eigenvalue of corresponding normalized Gram matrix).  We present the following theorem on the approximation guarantees of the solution achieved by our procedure \texttt{GetGoodSubset} of Algorithm~\ref{alg:greedy} - the theorem says that the minimum eigenvalue of the Gram matrix associated with arms in $\ca{G}$ (obtained post randomized rounding procedure) is close to $\mineigvalue^*$.
\begin{theorem}\label{th:randomizedrounding}
Let $\armmatrix = [\arm^1,\dots,\arm^\numarms]^\transpose \in \bb{R}^{\numarms\times\dimarm}$ denote the matrix of all arms. Consider the concave optimization of~\eqref{opt:relaxed} solved at $\searchbound = \orderof({\frac{\dimarm}{{(\mineigvalue^l)}^{2/3}}})$. Let $\armsubset$ be the output of the randomized rounding procedure (Algorithm~\ref{alg:getgoodsubset}) and $\widehat{\lambda}_{\min}$ be the minimum eigenvalue of the corresponding covariance matrix that is, $\widehat{\lambda}_{\min}=\lambda_{\min}(\left|\ca{G}^{-1}\right|\sum_{\fl{a}\in \ca{G}}\fl{a}\fl{a}^{\s{T}})$. Then under the assumption  $\mineigvalue^* \geq \mineigvalue^l$, we must have $\widehat{\lambda}_{\min} \geq \frac{1}{4}\mineigvalue^*$ with probability $1 - o(1)$.
 
\end{theorem}

\begin{remark}\label{rem: point on complexity and tradeoff}
Note that the time complexity is polynomial in the number of arms $\numarms^{\orderof\left(\dimarm{\mineigvalue^l}^{-2/3}\right)}$ (refer to Appendix~\ref{sec:bruteforce} for details) which is significantly improved than the trivial brute force algorithm which has a running time of $\orderof(\exp{\numarms})$. One can further reduce the time complexity to $\orderof(\numarms)$ and remove the exponential dependence on $\dimarm$, by choosing $\searchbound = \orderof(({\mineigvalue^l})^{-2}\dimarm)$ at the cost of a slightly worse approximation guarantee of $\widehat{\lambda}_{\min} = \Omega\left({(\mineigvalue^l)^2\mineigvalue^*}\right)$~\footnote{Note that $\mineigvalue^l \leq \mineigvalue^* \leq 1$ since $\Vert\arm\Vert_\infty\leq 1$.}  (See Appendix~\ref{sec:appendix on avoiding brute force}).
\end{remark}

    Several existing techniques in experimental design deal with maximizing objectives such as minimum eigenvalue; however, they assume submodularity or matroid constraints~\citep{Allen-Zhu2021-yl}, which the average minimum eigenvalue of~\eqref{eq:discreteopt} does not satisfy (see Appendix~\ref{sec:matroid} for a detailed discussion).

\textit{Proof Outline: } 
 We first show in Lemma~\ref{lemma:probboundeig} (using concentration guarantees) that the following objective values are close: (A) value of the maximized concave objective with distribution $\hat{\samplingdist}\in \ca{P}(\ca{A})$ and parameter $\searchbound$ in Equation \ref{opt:relaxed} (B) objective value of the set $\ca{G}$ (equation \ref{eq:discreteopt}) obtained via randomized rounding procedure from $\hat{\samplingdist}$  at $\searchbound$ (line 4 in Algorithm~\ref{alg:getgoodsubset}). 
Note that the value of the maximized concave objective in Equation \ref{opt:relaxed} with parameter $g_1$ is larger than the value with parameter $g_2$ provided $g_1 \le g_2$. Therefore, we show our approximation guarantees with respect to objective in Equation \ref{opt:relaxed} with parameter $\searchbound=\orderof(\dimarm)$. We show that the approximation guarantee holds with high probability if $|\mathcal{X}|>\searchbound$ and do a brute-force search otherwise.
Finally, given that the concave objective with parameter $\dimarm$ in Equation \ref{opt:relaxed} is a relaxation of the discrete objective in Equation \ref{eq:discreteopt}, the objective value of the former is going to be larger than the objective value of the latter.
\subsection{Online Guarantees - Regret Bound for \texttt{BSLB}}\label{subsec:online}
Our next result is the expected regret incurred by \texttt{BSLB} in the online setting. 
The key ingredient in the regret analysis lies in appropriately accounting for the blocking constraint in the exploitation component of \texttt{BSLB}. Below, we present the result detailing the regret guarantees of \texttt{BSLB} when the exploration period $\numroundsexplore$ is set optimally using a known sparsity level $k$. 
\begin{theorem}\label{th:regret}
   (\textbf{Regret Analysis of \texttt{BSLB}}) 
Consider the $\dimarm$-dimensional sparse linear bandits framework with blocking constraint having a set $\armset \subset \unitball$ of $\numarms$ arms spanning $\R^\dimarm$ and $\numrounds$ rounds ($\numrounds \ll \dimarm \ll \numarms$). Let $\|\arm\|_\infty\leq 1 \forall \arm \in \armset$. In each round $\timeindex\in [\numrounds]$, we choose arm $\arm_\timeindex\in \armset$ and observe reward $\reward_\timeindex = \langle\parameter,\arm_\timeindex \rangle + \noise_\timeindex$ where $\parameter\in \R^\dimarm$ is unknown and $\noise_t$ is zero-mean independent noise random variable with variance $\sigma^2=O(1)$. Let $\rmax = \max_\arm \langle\parameter,\arm \rangle$. Suppose $\parameter$ has (unknown) tail magnitude $\sparsitytail_\sparsityk$ at fixed sparsity level $\sparsityk$. Let $\mineigvalue^*$ for the set $\armset$ be as defined in~\eqref{eq:discreteopt} and let $\mineigvalue^l$ be its known lower bound. Let $\widehat{\lambda}_{\min}$ be the minimum eigenvalue of the normalized Gram matrix of the sampled subset from Step 1 of \texttt{BSLB}.
In this framework, \texttt{BSLB} with
exploration period $\numroundsexplore = \widetilde{\orderof}(\rmax^{-2/3}\widehat{\lambda}_{\min}^{-2/3}\sparsityk^{\frac{2}{3}}\numrounds^{\frac{2}{3}})$, regularization parameter $\lagrange=\sqrt{\frac{\log \dimarm}{\numsamples}}$ and subset selection parameter $\searchbound= \orderof(\frac{\dimarm}{(\mineigvalue^l)^{2/3}})$, achieves a regret guarantee    
   \begin{align}\label{eq:regretbound}
   \begin{split}
        \expectedregret = \widetilde{\orderof}&\left( \rmax^{1/3} {(\mineigvalue^*)}^{-2/3}(1+\restrictedconeconstant)^2\sparsityk^{\frac{2}{3}}\numrounds^{\frac{2}{3}}  \right).
           \end{split}
   \end{align}
\end{theorem}
\textit{Proof Outline: }
In exploration component of the regret decomposition, we use Corollary~\ref{coro:empiricallasso} to obtain error guarantees of Lasso and Theorem~\ref{th:randomizedrounding} to bound the minimum eigenvalue of the Gram matrix of the sampled arms. We optimize the exploration period to obtain our stated result (Proof in App.~\ref{app:corollaryproof}).

Note that when $\restrictedconeconstant=0$ (when the parameter vector is exactly $\sparsityk$-sparse), we obtain regret guarantee which has the same order in $\numrounds$, $\sparsityk$, $\mineigvalue^*$ and $\rmax$ as that of~\citep[Theorem 4.2]{hao2020high}.
\begin{insight}
    \texttt{BSLB} enables diversity in selected arms by performing Step~2 in Algorithm~\ref{alg:greedy} - this step ensures that $\mineigvalue$ of the covariance matrix of the subset used in exploration is approximately optimal.
    The exploration period in Theorem \ref{th:regret} is optimized to maximize cumulative reward. However, in practice, the exploration period of \texttt{BSLB} can be increased at the cost of higher regret if diversity assumes higher importance in the sampled data-points. 
\end{insight}
\begin{remark}\label{rem: point on complexity and tradeoff in regret}
The optimization in Step 9 of \texttt{BSLB} (LASSO) has a runtime $\mathsf{Poly}$($\numarms,\dimarm,\numrounds$). The runtime of the \textsc{GetGoodSubset}($\armset$, $\searchbound$) is $\mathsf{Poly}$($\numarms$) but exponential in $\dimarm$ (see Appendix \ref{sec:bruteforce}  details). Note that the stated runtime in the latter part is significantly lower than the trivial brute force search for the optimal subset having a runtime of $\orderof(\exp(\numarms))$. This is possible because, as a result of the approximation guarantees of Theorem~\ref{th:randomizedrounding}, we restrict the size of the subset for the brute-force search. Further one can remove the exponential dependence on $\dimarm$ (See Appendix~\ref{sec:appendix on avoiding brute force}), however this worsens the regret guarantee by a multiplicative factor of $(\mineigvalue^l)^{-2}$.
\end{remark}
We provide the following theorem on the lower bound of regret for the high-dimensional linear bandit setting with blocking constraint. 
\begin{theorem} \label{thm:lower}
     Consider the $d$-dimensional sparse linear bandits framework with blocking constraint having a set $\mathcal{A} \subseteq \mathcal{B}^d$ of M arms spanning $\mathbb{R}^d$ and $\mathsf{T}$ rounds ($\mathsf{T} \ll d \ll \numarms)$. In each round $t \in [\mathsf{T}]$, we choose arm $a_t \in \mathcal{A}$ and observe reward $r_t = \langle \theta, a_t \rangle + \eta_t$ where $\theta \in \mathbb{R}^d$, is unknown and $\eta_t$ is zero-mean independent noise random variable with variance $\sigma^2 = 1$. Assume that the parameter vector is $k$-sparse, $\Vert \theta \Vert_0 = k$. Then for any bandit algorithm the worst case regret is lower bounded as follows,
     \[
      \mathbb{E}[\mathsf{R}] = \Omega(\min((\mineigvalue(\mathcal{A}))^{-1/3}k^{1/3}\mathsf{T}^{2/3}),\sqrt{d\mathsf{T}})).
     \]
\end{theorem}

For the case when the true parameter satisfies the exact sparsity condition (the tail $\sparsitytail_\sparsityk$ is $0$),  our regret guarantee (see Corollary~\ref{coro:hard} in Appendix) achieves the same $\s{T}^{2/3}$ regret dependence as in as~\citet{hao2020high} without the \textit{blocking constraint}. Note that similar to \citet{hao2020high}, our regret guarantees with a $\s{T}^{2/3}$ scaling and respecting the blocking constraint  are also tight.

\subsection{Corralling when optimal sparsity level is not known}\label{subsec:corralling}

   Note that for any unknown parameter vector $\f{\theta}$, we can fix the sparsity level $k$ and therefore the corresponding tail magnitude $\restrictedconeconstant$ - subsequently, we can obtain the guarantees of Theorem \ref{th:regret} by setting the exploration period optimally for the fixed $k$. However, if $k$ is set too low, then $\restrictedconeconstant$ will be too high, and therefore, the multiplicative term containing $\restrictedconeconstant$ in the regret (Equation \ref{eq:regretbound}) dominates. There is a trade-off, and therefore, there is an optimal choice of sparsity $\sparsityk$. Therefore, we propose a meta-algorithm \texttt{C-BSLB} that exploits coralling~\citep{agarwal2017corralling}  multiple versions of the \texttt{BSLB} algorithm~\ref{alg:greedy} with different values of $\sparsityk$ used to set the exploration period $\numroundsexplore$ - the meta-algorithm gradually learns to choose the best base algorithm. 
   
   However, naively applying CORRAL with all distinct base algorithms leads to a linear dependence on dimension $d$ in the regret making it vacuous. Therefore we carefully choose $\log \dimarm$ base algorithms for search within CORRAL with corresponding sparsity parameters set on exponentially spaced points - such a restriction ensures that the overhead in regret is minimal (logarithmic dependence on dimension $\dimarm$). However, we still prove our regret guarantee with respect to the base algorithm with optimal sparsity - although the optimal base algorithm may not be in the set of carefully chosen base algorithms in the meta-algorithm.  

\begin{theorem}\label{th:corralling} 
Consider the $\dimarm$-dimensional sparse linear bandits framework with blocking constraint as described in Theorem \ref{th:regret}. Let the \texttt{C-BSLB} algorithm (Algorithm~\ref{alg:corral} in Appendix \ref{sec:corralling}) run with an appropriate learning rate on multiple versions of \texttt{BSLB}, using distinct sparsity parameter  $\sparsityk$ taking values in the set $\{2^i\}_{i=0}^{\lfloor \log_2(\dimarm)\rfloor+1}$. Let the optimal sparsity parameter in Theorem \ref{th:regret} that achieves minimum regret be $\sparsity \in \{1,2,\dots,\dimarm-1,\dimarm\}$, and let $\bestregret$ be the corresponding regret. Then the meta-algorithm \texttt{C-BSLB} achieves the following regret guarantee,  
\begin{align}\label{eq:corral_guarantee}
        \expectedregret &= \orderof(\sqrt{\numrounds\log_2(\dimarm)}+\sqrt{\sparsity}\log_2(\dimarm)\bestregret). 
    \end{align} 
\end{theorem}
Note that the first term in Equation \ref{eq:corral_guarantee} and the multiplicative factor of $\sqrt{\sparsity}\log_2(\dimarm)$ corresponds to the additional cost in combining the input base algorithms by  Algorithm~\ref{alg:corral}. We stress that the dependence on dimension $\dimarm$ from the additional cost is only logarithmic.

\textit{Proof Outline: } We use ~\citep[Theorem 5]{agarwal2017corralling} and our regret bound from Theorem~\ref{th:regret} to obtain Theorem~\ref{th:corralling}. The key novelty in our proof is to establish the following - when searching with the small curated set of base algorithms in CORRAL, we do not suffer a significant loss in the regret even if the base algorithm with the optimal sparsity parameter does not lie in the curated set. The crux of our proof lies in a covering argument. By using a recursive telescoping argument, we can bound the regret incurred between any base algorithm not used for the search while corraling versus the base algorithm with the closest sparsity parameter used in the search. 

\section{Limitations and Future Work}


This paper addresses the challenge of regret minimization in high-dimensional, sparse linear bandits with blocking constraints, without imposing exact sparsity constraints on the parameter. We establish theoretical guarantees demonstrating the sub-linear regret of the proposed \texttt{BSLB} algorithm, ensuring these results hold under minimal assumptions on the set of arms and exhibit robustness to sparsity. Furthermore, we introduce a meta-algorithm, \texttt{C-BSLB}, which combines multiple \texttt{BSLB} algorithms with varying input parameters. This meta-algorithm achieves the same regret bounds as \texttt{BSLB}, eliminating the need for prior knowledge of the optimal sparsity hyperparameter. This paper assumes a linear relationship, which cannot capture binary or other non‐linear rewards; one could extend BSLB to a generalized linear bandit by using a logistic link function for binary rewards, thereby modeling non‐linear decision boundaries.
To make the approach robust against heavy‐tailed or adversarial noise, we can replace Lasso with a GLM estimator or incorporate robust covariance estimation in the exploration phase, which would mitigate the impact of outliers and model deviations.

\bibliography{iclr2025_conference}
\bibliographystyle{plainnat}
\newpage

\appendix

\newpage 

\section{Numerical Experiments}\label{app:numerical}

\subsection{Simulation Study}\label{app:simulation}

\textit{Performance of C-BSLB (without knowledge of true parameters)}\label{app:corrallingexp}

\begin{figure}[ht]
    \centering
    \includegraphics[width=0.5\linewidth]{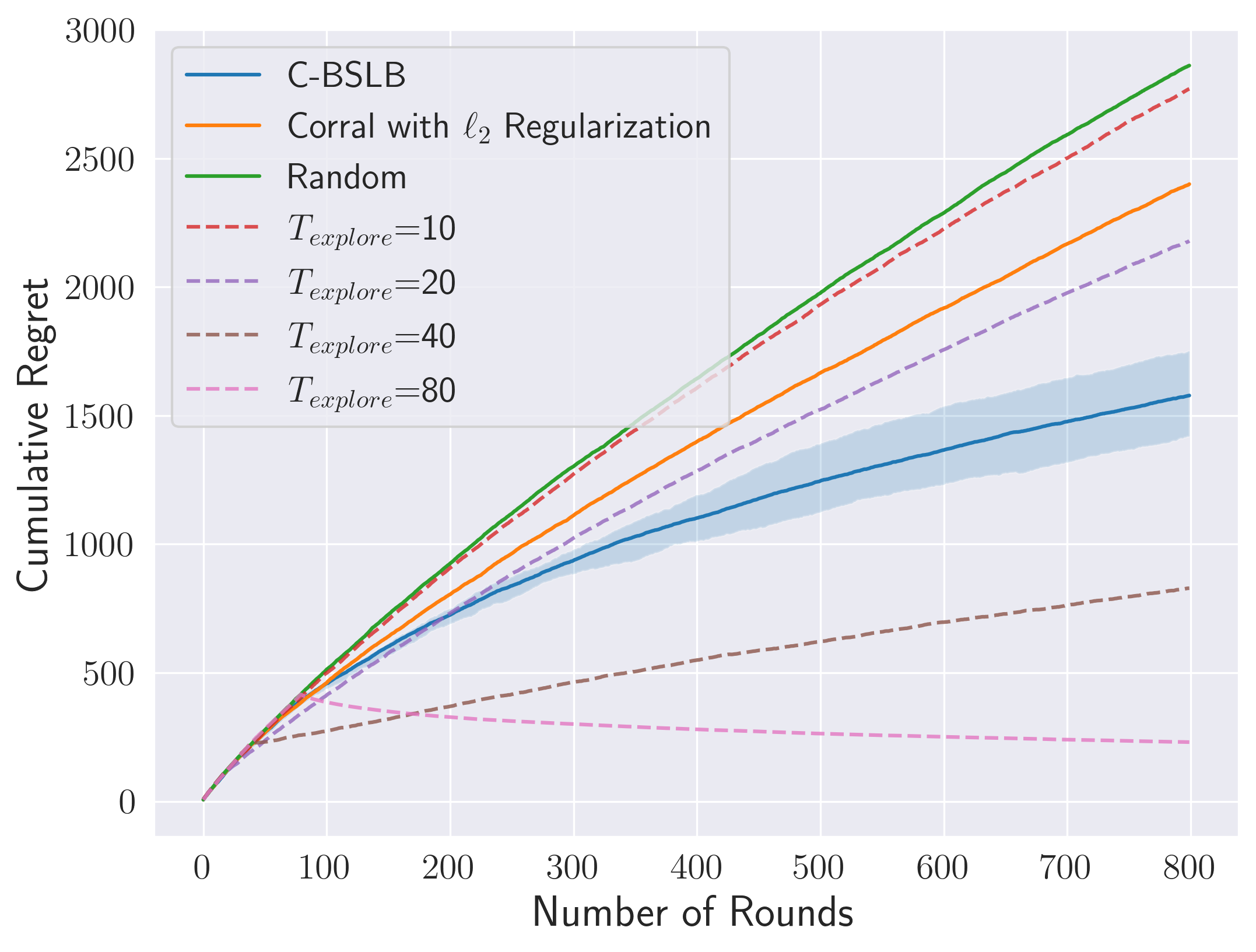}
    \caption{Regret of different algorithms in a Simulated Blocked Sparse Linear Bandit Setup. }
    \label{fig:cummulativesimulated}
\end{figure}
Finally, we also run a simulation study to study the efficacy of our \texttt{BSLB} and \texttt{C-BSLB} algorithm and demonstrate how CORRAL can be used to achieve a sub-linear regret without the knowledge of the optimal parameters.  We compute the cumulative regret at time $\timeindex$ compared to the top$-\timeindex$ arms, and unlike the standard bandit setting, in a blocked setting, the cumulative regret need not be monotonic. To highlight how our method exploits the sparsity of the parameter, we also run CORRAL with multiple versions of our algorithm but with a simple linear regression estimator. We simulate the experimental set-up with the following parameters $\numarms = 10000, \dimarm=1000$ and $\numrounds = 300$. At sparsity level $\sparsityk=10$, the tail parameter is $\restrictedconeconstant=3$. The experiment is repeated with $100$ different random parameter initialization.  We plot the cumulative regret in Fig.~\ref{fig:cummulativesimulated}, for algorithms run with different exploration period and two versions of the CORRAL algorithm. Our C-BSLB performs better than corralling with $\ell_2-$regularization, showing that our method exploits the sparsity and does not require true knowledge of the hyperparameters. We also benchmark against a random policy and show that our method performs significantly better showing that the upper bound on regret is not \textit{vacuous}. 

\subsection{Personalized recommendation with Single Rating per Item}\label{app:recommendation}
Next we demonstrate our bandit algorithm on real-world recommendation datasets. However, we construct the recommendation task such that a) each item receives only a single rating from a user and b) we can only use previous recommendations of a user for recommending content. This is in contrast to the standard collaborative filtering setting where an item can where the ratings of the other users is used to recommend content to you. Our setting makes this possible by exploiting the additional information from the embeddings obtained from a pre-trained network for the text (or image) features of the different items. We argue that our setting is be more relevant in recommendation scenarios where privacy is a concern. 
\newcommand{\numcopies}{\text{copies}}

We perform two sets experiments. 

\textit{Experiment 1 on MovieLens and Netflix data  (Movie Ratings) :}

The MovieLens 100K dataset contains 100,000 5-star ratings from 1,000 users on 1,700 movies . For our analysis, we selected the 100 most active users (those with the highest rating counts) and the 100 most-rated movies. Using rating data from 50 users, we first applied matrix factorization to complete the user-item rating matrix, then derived 40-dimensional movie embeddings. To create synthetic high-dimensional embeddings, we extended these to 120 dimensions by randomly sampling additional coordinates uniformly from [0,1]. Results averaged across 5 test users (shown in Figure~\ref{fig:movielensbenchmark}) demonstrate that our algorithm and ESTC outperform LinUCB and DR-Lasso, with the performance gap attributable to our method's effective exploitation of high-dimensional sparsity patterns . This aligns with theoretical expectations for sparse learning scenarios in recommendation systems.
\begin{figure}[h]
    \centering
    \includegraphics[width=0.5\linewidth]{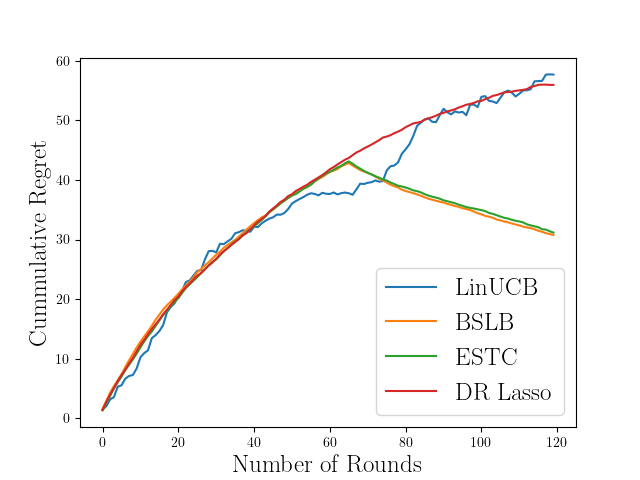}
    \caption{Numerical experiment on MovieLens illustrating performance gap between our proposed algorithm \texttt{BSLB} and naive extensions of LinUCB and DR-Lasso to incorporate blocking constraint. The performance of extended \texttt{ESTC} remains competitive.}
    \label{fig:movielensbenchmark}
\end{figure}
\begin{figure}[h]
    \centering
    \includegraphics[width=0.5\linewidth]{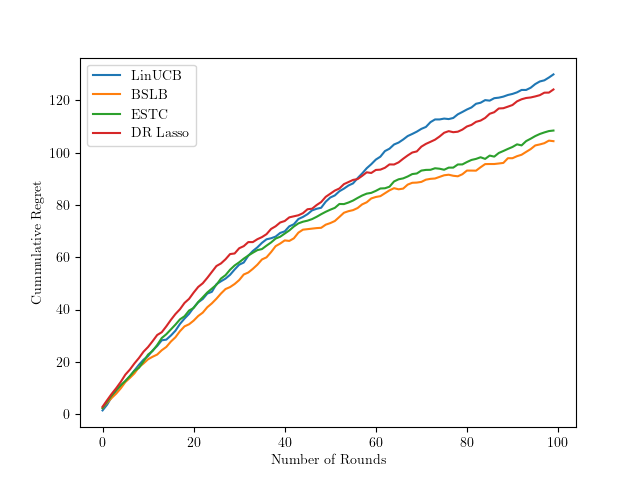}
    \caption{Numerical experiment on Netflix dataset illustrating performance gap between our proposed algorithm \texttt{BSLB} and naive extensions of LinUCB and DR-Lasso to incorporate blocking constraint. The performance of extended \texttt{ESTC} remains competitive.}
    \label{fig:netflixbenchmark}
\end{figure}
The Netflix Prize dataset contains 100 million 5-star ratings from 480,189 users across 17,770 movies, collected between 1998-2005. For analysis, we selected the 200 most active users (those with the highest rating counts) and 400 most-rated movies. Using ratings from 100 users, we applied matrix factorization to complete the user-item interaction matrix and derived 40-dimensional latent movie embeddings. To simulate high-dimensional sparse representations, we extended these embeddings to 120 dimensions by randomly sampling additional coordinates from [0,1].
 Results averaged across 10 test users (shown in Figure~\ref{fig:netflixbenchmark}) demonstrate that our algorithm and ESTC outperform baseline methods like LinUCB and DR-Lasso. This performance gap highlights the advantage of exploiting sparsity patterns in high-dimensional latent factor models.

\textit{Experiment 2 using Embeddings from Content Information:} We run the corralling algorithm using $\numcopies=4$ copies of Algorithm~\ref{alg:greedy} each with different exploration periods, $\numroundsexplore$. Each of the instance, we first give $\numroundsexplore$ random recommendations by sampling uniformly without replacement from a suitably constructed subset $\armsubset$ to each user. Given the ratings obtained, we estimate the parameter $\parameter_{\text{user}}$ specific to the user using only their recommendations. For the remaining $\numroundsexploit=\numrounds-\numroundsexplore$ rounds, we give the top $\numroundsexploit$ recommendations based on the estimated parameter.  To benchmark we run the algorithms independently and also against a random policy which randomly recommends. We next describe the two tasks that we report our results on for experiment $2$,

\textbf{Goodbooks-10k (Book Reviews): }
We use the Goodbooks-10k for a personalized book recommendation tasks~\citep{goodbooks2017}. For each book we use the title and the author to obtain embeddings using the MiniLM-L6-v2 sentence transformer which we use as the feature vectors for the arms. There are $\numarms = 1500$ books and we consider $10$ users which have more than $600$ ratings. The ratings are between $1$ to $5$. We consider the exploration periods as $[100,150,200,300]$.


\textbf{Jester (Joke Ratings):}
We use the Jester joke dataset 1 which has ratings on $100$ jokes by $24,983$ users~\citep{Goldberg2001-ur}. We obtain embedding for the jokes using the same transformer as above. For experimental purposes we filter out users which do not have ratings on all the jokes and are left with $7200$ users. We run our algorithm with $10$ different random seeds for each of the $7200$ users and report the results averaged across all users. The joke ratings range from $-10$ to $10$. For different algorithm instances $\numroundsexplore$ is taken to be $[20,40,60,80]$.
\begin{figure}
    \centering
    \includegraphics[width=0.5\columnwidth]{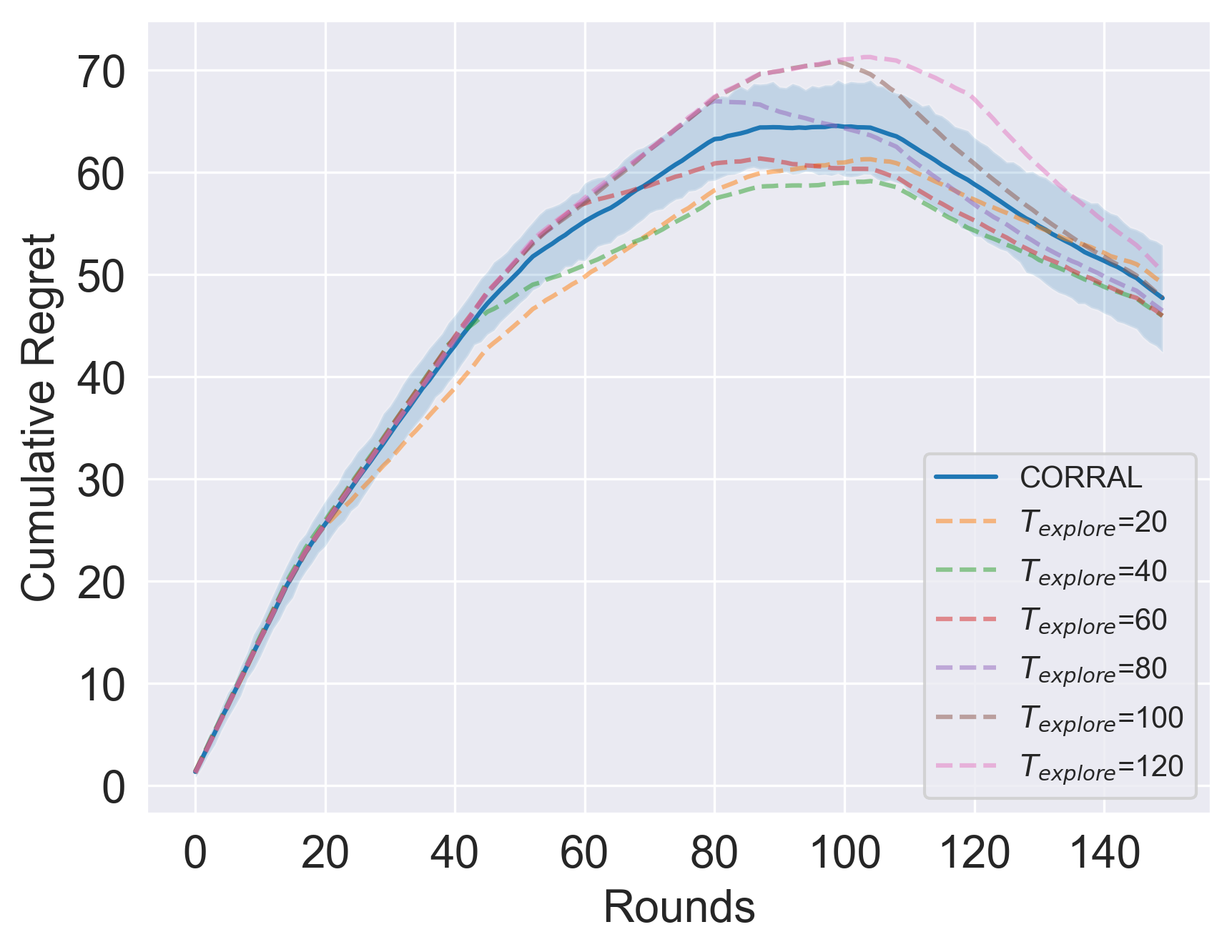}
    \caption{Cumulative Regret for recommendation using only single ratings using \texttt{BSLB} with different exploration periods and when run with CORRAL~\citep{agarwal2017corralling} in Books Dataset.}
    \label{fig:subfig1benchmark}
\end{figure}
\begin{figure}
    \centering
    \includegraphics[width=0.5\columnwidth]{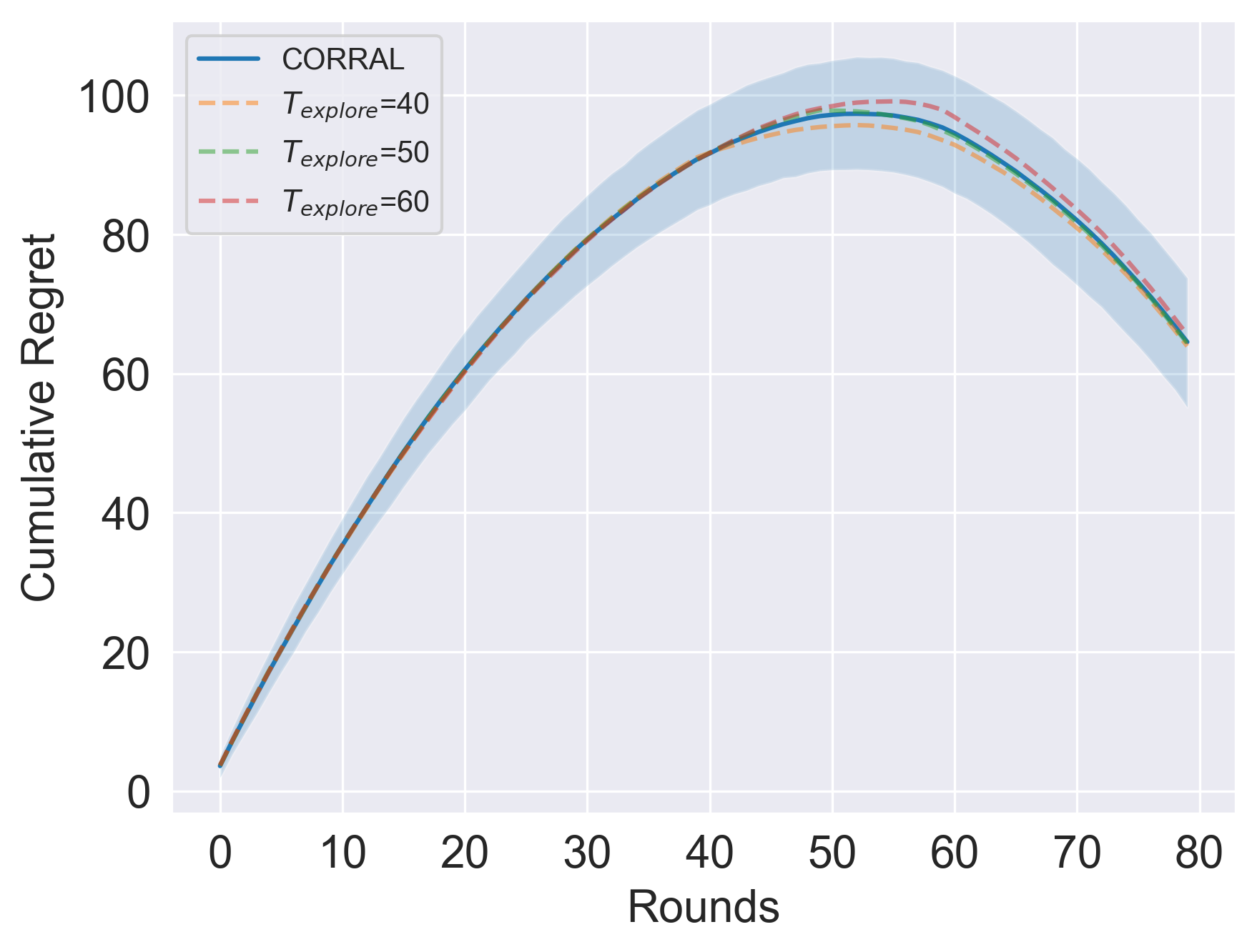}
    \caption{Cumulative Regret for recommendation using only single ratings using \texttt{BSLB} with different exploration periods and when run with CORRAL~\citep{agarwal2017corralling} in Jester.}
    \label{fig:subfig2benchmark}
\end{figure}
\textbf{Results: }We summarize the cummulative regret from equation~\ref{eq:regretcumm} of the algorithms in Figure~\ref{fig:subfig2benchmark}. We add the random policy as a reference. We see that for the different dataset our the algorithm achieves a sub-linear regret. The reason the cumulative regret is not monotonic is due to the fact that the regret is with respect to the top-$\numrounds$ arms. 
It can be seen that our algorithm with Corral achieves a performance close to the performance of the algorithm with the exploration period, $\numroundsexplore$ out of the $5$.

\subsection{Adaptive Annotation using Difficulty Feedback from Annotators}\label{sec:annotation}
Below, we demonstrate our methods for annotation in a label-scarce setting for image classification on the PASCAL VOC 2012 dataset. Additional experimental results on SST-2 (text dataset) can be found in Appendix \ref{app:sst-2}. Finally, experiments on the MovieLens, Netflix, and GoodBooks datasets in the context of personalized recommendation with few labeled data using our theoretical framework are in Appendix \ref{app:recommendation}. Finally, we provide detailed simulations in Appendix \ref{app:simulation}.

We consider the setting where we have a total of $\numarms$ unlabelled samples (with $\numrounds \ll \numarms$) and only $\numrounds$ datapoints can be annotated (sequentially). For each unlabeled datapoint sent for annotation to the expert(s),  we receive the ground truth label and the \textit{difficulty score} $\reward_\timeindex$ corresponding to the difficulty in annotating the datapoint. 
We showcase the effectiveness of  \texttt{BSLB} (Algorithm~\ref{alg:greedy}) in our experimental set-up with real-world datasets. Given a model $\ca{M}$ to be trained on a downstream task, to benchmark \texttt{BSLB}, we consider the following set of baselines (to compare against) to choose subset of datapoints for annotation and subsequent training of the aforementioned model $\ca{M}$:

\begin{enumerate}[noitemsep, nolistsep, leftmargin=*]
     \item \textbf{Random}: Subset of $\numrounds$ unlabeled datapoints chosen uniformly at random 
    \item \textbf{All}: All the samples in training data (except the validation fold) 
    \item \textbf{AnchorAL}~\citep{activelearning2024}: an anchoring based active learning baseline ($\numrounds$ samples). 
    \item \textbf{SEALS}~\citep{activelearning2022}: a KNN based sub-sampling active learning baseline ($\numrounds$ samples). 
\end{enumerate}

$\tau_{\text{easy}},\tau_{\text{hard}}$ (thresholds on difficulty score to determine easy/hard samples) and $\numroundsexplore$ (exploration rounds) 
are relevant hyper-parameters 
specified for the corresponding experiments~\footnote{We consider the AL setup initialized with $\numroundsexplore$ samples and $\numrounds-\numroundsexplore$ samples queried in a batch.}. We benchmark learning performance on 2 datasets: a) $N_{\text{valid}}$ hard samples (samples with difficulty $>\tau_{\text{hard}}$) (\textbf{hard-valid}) b) $N_{\text{valid}}$ easy samples (samples with difficulty ratings $<\tau_{\text{easy}}$) (\textbf{easy-valid}).

\textbf{AnchorAL} and \textbf{SEALS} are state-of-the-art active learning (AL) algorithms. In general, for a label-scarce complex task, AL might not be immediately applicable (see cold-start problem in \cite{li2024survey}) - especially for datapoints close to the decision boundary with noisy/less informative confidence intervals. This is because AL requires a reliably trained model on an even smaller subset of labeled datapoints - however, on datapoints far from the decision boundary (easy datapoints), noisy confidence signals are still useful. As we show in our experiment, this intuition holds, and the AL models, along with the \textbf{random} baseline, perform well on the \textbf{easy-valid} dataset. It is worth noting that complex (hard) datapoints often tend to be the main challenge in industrial applications. This is because it is \textit{easy} to improve performance on easy data (cheaper to obtain) by simply increasing samples during training, but hard datapoints are difficult to generalize on~\citep{pukowski2024investigating}.



 

\subsubsection{Image Classification on PASCAL VOC 2012}
Our main result is for the image classification
task on the public image dataset, PASCAL VOC 2012~\citep{Everingham15}. The dataset has $11,540$ unique images and comprises segmentations for $20$ objects. In addition to the image dataset, we use difficulty scores of annotations from~\citep{img-difficulty-CVPR-2016} - the authors have provided the visual search difficulty by measuring the time taken to annotate in a controlled environment. The annotation task here was to identify if an image contains a particular object, e.g. ``Does this image contain a car''. The authors derive a difficulty score between $0$ and $8$ by normalizing the time to annotate.  

In our experiment, the goal is to train a learning model for image classification - $\ca{M}$ is a support vector machine (SVM) head attached to a frozen pre-trained vision transformer (ViT) model pre-trained on ImageNet-21k dataset~\citep{wu2020visual}. We present results on the classification task - given an input image, predict if the image has an \textit{object} or not. We consider $4$ different objects, namely chair, car, bottle, and (bottle or chair). The last object is an \texttt{OR} conjunction of two labels. We consider the thresholds as $\tau_{\text{easy}}=3.1$ and $\tau_{\text{hard}} = 3.9$ since the distribution of the difficulty scores in the dataset is heavy-tailed as shown in Figure~\ref{fig:distdifficultyscore}. The (image, question) tuple with difficulty scores in the range $[3.1,3.8]$ are highly noisy and therefore have been excluded. Table~\ref{tab:resultshyperparameters} contains the hyperparameters $\numrounds$, $\numroundsexplore(\approx 0.6\numrounds)$ used and the number of samples in the \textbf{all} dataset for the different object classification tasks, along with the size of the validation datasets \textbf{hard-valid} and \textbf{easy-valid} and aggregaged accuracies. Table~\ref{tab:resultsextra} contains results on the effect of varying $\numroundsexplore$.

\begin{table*}[t]
\centering
\resizebox{\columnwidth}{!}{%
\begin{tabular}{ccccccc}
\hline
\begin{tabular}[c]{@{}c@{}}Validation\\ Type\end{tabular} &
  \multicolumn{1}{c}{\begin{tabular}[c]{@{}c@{}}Object \\ Annotated\end{tabular}} &
  \multicolumn{1}{c}{AnchorAL} &
  \multicolumn{1}{c}{SEALS} &
  \multicolumn{1}{c}{Random} &
  \multicolumn{1}{c}{All} &
  \multicolumn{1}{c}{\textbf{Our} (\texttt{BSLB})} \\ \hline
\multirow{4}{*}{easy-valid}          & chair           & 94.0 $\pm$ 1.67 &  90.6 $\pm$1.8  & 96.4 $\pm$1.0  & \textbf{96.0$\pm$ 1.1}   & 94.6 $\pm$1.6            \\
                               & car             & 94.5$\pm$ 1.6 & 94.7$\pm$ 4.0  & 97.7$\pm$2.1  & \textbf{98.7 $\pm$0.1} & 96.5$\pm$ 1.8            \\
                               & bottle          & 93.0$\pm$2.5 & 92.8$\pm$2.3 & 96.8$\pm$1.1 & \textbf{96.8$\pm$1.1} & 94.8$\pm$2.0          \\
                               & bottle or chair & 91.5$\pm$1.1 & 92.3$\pm$1.1 & 94.8$\pm$0.97 & \textbf{94.6$\pm$2.1 }& 91.7$\pm$2.2         \\ \hline
\multirow{4}{*}{\textbf{hard-valid}} & chair           & 69.3 $\pm$ 3.1 & 69.6$\pm$ 6.1  & 66.0$\pm$ 3.8  & 71.3$\pm$ 3.2           & \textbf{73.3$\pm$ 3.3}   \\
                               & car             & 70.3 $\pm$4.0   & 70.0$\pm$ 5.7    & 60.0$\pm$ 5.4  & 65.4$\pm$ 4.0          & \textbf{74.0$\pm$ 3.4} \\
                               & bottle          & 63.1$\pm$2.9 & 63.4$\pm$3.4 & 59.7$\pm$4.4 & 64.8$\pm$1.9 & \textbf{66.8$\pm$2.6} \\
                               & bottle or chair & 67.1$\pm$3.5 & 66.3$\pm$1.4 & 68.0$\pm$4.0 & 72.3$\pm$2.0 & \textbf{73.0$\pm$1.7} \\ \hline
\end{tabular}%
}

\caption{Test accuracy of model $\ca{M}$ trained on different subsets of data annotated for 4 distinct object detection tasks in an image (PASCAL-VOC): The test performance of \texttt{BSLB} approach on the easy and hard validation dataset is at par with the $\ca{M}$ trained on all samples. We perform significantly better on the hard validation dataset compared to random sampling and active learning baselines.}
\label{tab:results}
\end{table*}


We present our results in Table~\ref{tab:results} averaged over $5$ validation folds.
For this classification task, our method (\texttt{BSLB})  efficiently selects datapoints (to be annotated) compared to baselines with an equal number of samples. Regarding the quality of the final trained model $\ca{M}$, the learning performance of \texttt{BSLB} on \textbf{easy-valid} is within $2\%$ of that obtained by the baseline \textbf{random}. However, there is an improvement of $5-14\%$ on the hard validation data \textbf{hard-valid}. When compared to the active learning baselines (\textbf{AnchorAL} and \textbf{SEALS}), \texttt{BSLB} performs better by $1-4\%$ on \textbf{easy-valid} and by $3.5-7\%$ on \textbf{hard-valid}. Finally, when compared to $\ca{M}$ trained on all datapoints (\textbf{all} baseline), which has $6\times$ to $12\times$ more samples, our method does better ($0.7\%$ to $8.6\%$) on the \textbf{hard-valid} and does decently on \textbf{easy-valid} ($<3\%$ difference).  These results validate our theory - in particular, we find that performance on \textbf{easy-valid} improves if the model $\ca{M}$ is trained on more samples (randomly chosen to improve coverage). However, improving the performance on \textbf{hard-valid} dataset is the main challenge where our simple approach \texttt{BSLB} with theoretical guarantees does reasonably well.



\subsubsection{Validation on Different Hyperparameters }
The hyperparameters are presented in Table~\ref{tab:resultshyperparameters}. Note that the reason for selecting different $\numrounds$ across different objects was because the validation datasets had to be big enough (so that the variance of accuracy is informative) and different objects had different total number of samples. The number of exploration rounds are set with respect to $\numrounds (\sim 0.5\numrounds - 0.7\numrounds)$ so that the approximation error after exploration is small enough.
The active learning methods are also run with random initialization of $\numroundsexplore$ explorations and one round of querying with $\numrounds-\numroundsexplore$ queries. We present the results of a study where we vary only the $\numroundsexplore$ for the same value of $\numrounds$ in Table~\ref{tab:resultsextra} and observe that with decreasing $\numroundsexplore$ the estimation of difficulty scores deteriorates, and the performance on hard-valid deteriorates. The performance on easy-valid improves since samples are randomly chosen if the estimation is unsuccessful. Note that our method still performs better than AL baselines and random sampling. 
\subsubsection{Text Classification on SST-2}\label{app:sst-2}

Next we present a result on text classification task on SST-2~\citep{socher-etal-2013-recursive}. However since there were no human difficulty ratings available for this task, we use rarity of text~\citep{DBLP:journals/corr/abs-1811-00739} as a heuristic for the difficulty ratings. The learning model is a SVM which classifies sentence embeddings obtained from the MiniLM-L6-v2  transformer. We consider  $\numroundsexplore=100$ samples and $\numroundsexploit=200$. The normalized rarity ranges from $0$ to $1$ and we set $\tau_{\text{hard}} = 0.5$ and $ \tau_{\text{easy}}=0.2$. 

We observe a similar trend as the previous task where \textbf{BSLB} method performs better than a \textbf{random} subset by $3\%$ and as good as the \textbf{random-large} subset on the \textbf{hard-valid} dataset. There is no regression on the \textbf{easy-valid}. The results on both the validation sets are comparable with~\textbf{mixed} dataset which require all the difficulty ratings (which can be computed for the heuristic but not otherwise). \textbf{BSLB} performs better than both active learning methods on both the validation sets by $2\%$.  However, since this is a standard sentiment analysis task, the embeddings are more informative, thereby improving the baseline performance for a \textbf{random} subset.  

\subsubsection{Correlation between model difficulty and human annotation difficulty}
In Figure~\ref{fig:correlationdifficulty}, we show that for the chair images, if a model $\mathcal{M}$ is trained on all images, then the fraction of difficult samples at a certain distance from the classifier boundary goes down as the distance from the classifier body increases; especially after a certain distance from the decision boundary. 

\begin{figure}[ht]
    \centering
    \includegraphics[width=0.5\linewidth]{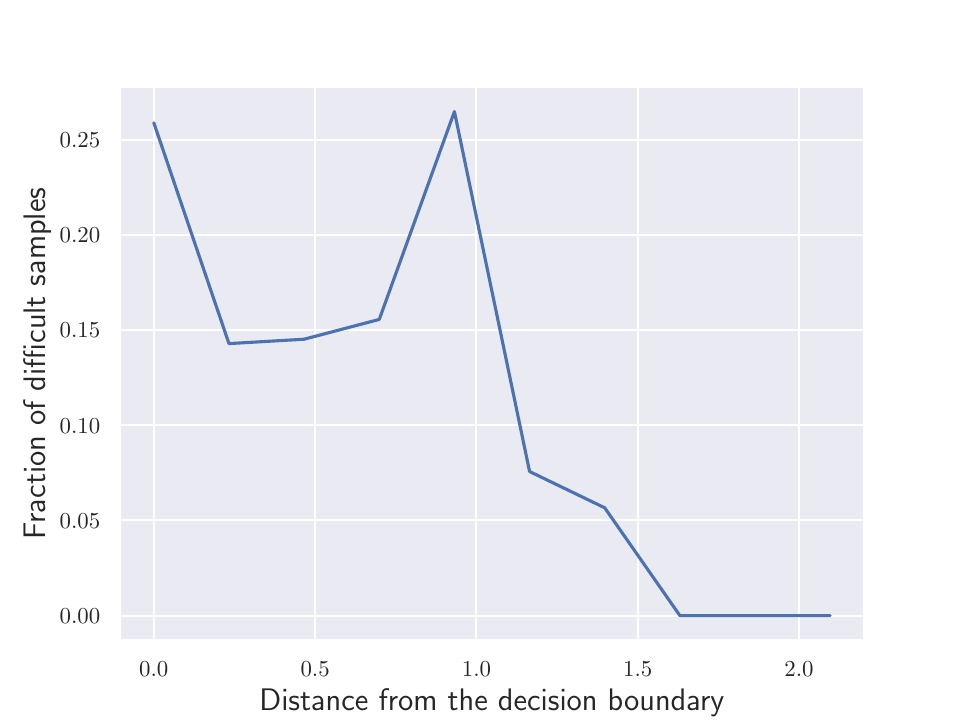}
    \caption{  Fraction of difficult samples (labelled by humans) against the distance from decision boundary for SVM trained on all \textit{chair} images. As the distance from the decision boundary increases the fraction of difficult samples (difficulty rating from humans $>3.5$) decays to 0.  }
    \label{fig:correlationdifficulty}
\end{figure}
\begin{figure}[ht]
    \centering
    \includegraphics[width=0.5\linewidth]{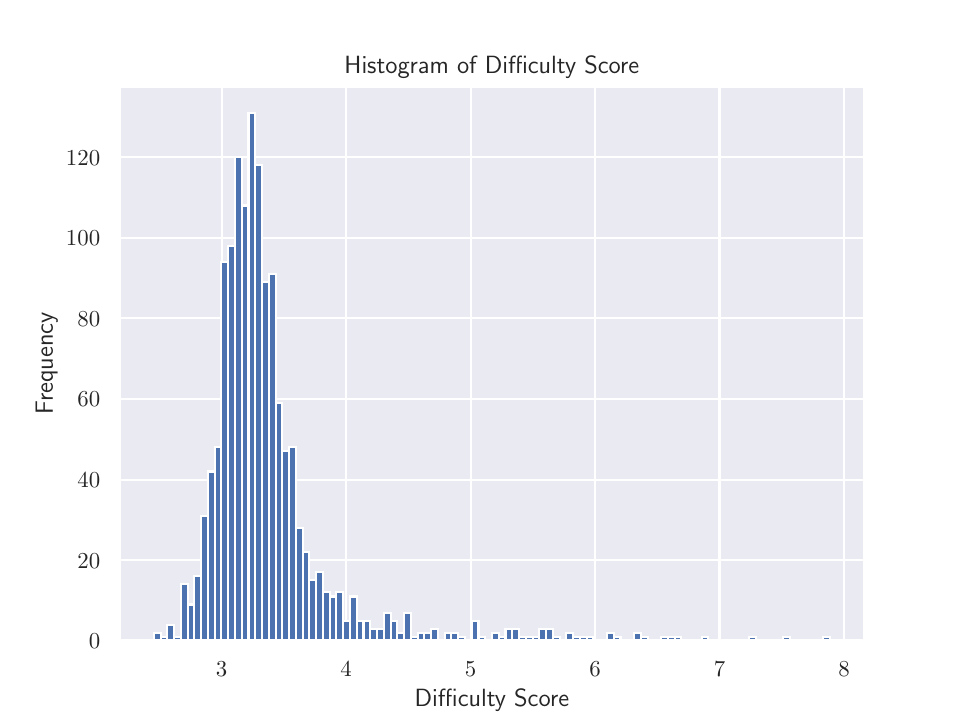}
    \caption{Histogram shows the heavy-tailed distribution of the difficulty score from~\citet{img-difficulty-CVPR-2016} of the \textit{chair} object of the PASCAL-VOC dataset. We clip the entries from the middle since they make the difficult estimation noisier, in practical implementation, one would need to develop a mechanism to flag samples with \textit{ambiguous} difficulty and this is left for future work.  }
    \label{fig:distdifficultyscore}
\end{figure}
\begin{figure}[ht]
    \centering
        \includegraphics[width=0.5\columnwidth]{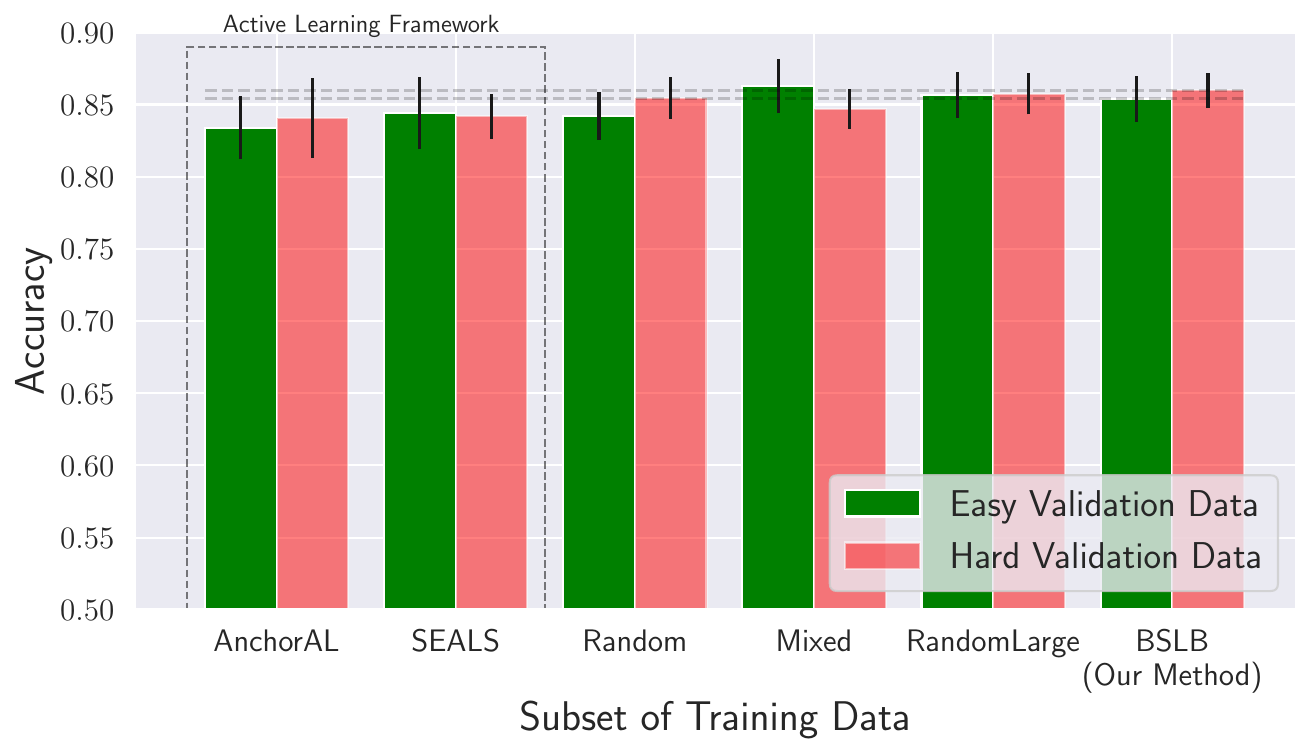}
        \caption{Text Classification on SST-2: The gains are not substantial on the text-classification, but show that our methods are task agnostic. Although conceptually active learning also does adaptive annotation, our method performs better (especially on \textbf{hard-valid}) in the label-scarce setting when $\numrounds\ll\dimarm$ and the hardness of the samples considered.}
        \label{fig:textclassification}
\end{figure}
\begin{table}[ht]
\centering
\resizebox{\columnwidth}{!}{%
\begin{tabular}{@{}ccccc|cc|c|@{}}
\toprule
\multicolumn{1}{c}{\textbf{\begin{tabular}[c]{@{}c@{}}Object\\ Being \\ Annotated\end{tabular}}} &
  \textbf{$N_{\text{valid}}$ (hard)} &
  \textbf{$N_{\text{valid}}$ (easy)} &
  \textbf{$\numrounds$} &
  \textbf{$\numroundsexplore$} &
  \multicolumn{2}{|c|}{\textbf{All}} &
  \multicolumn{1}{c}{\textbf{BSLB}} \\ \midrule
 &
   &
   &
   &
   &
  \multicolumn{1}{c}{\textbf{Num Samples}} &
  \multicolumn{1}{c}{\textbf{Averaged Accuracies}} &
  \multicolumn{1}{c}{\textbf{Averaged Accuracies}} \\
chair           & 60  & 80  & 100 & 80  & 960 (10x)  & 83.65          & \textbf{83.95} \\
car             & 70  & 100 & 90  & 60  & 1227 (13x) & 82.05          & \textbf{85.25} \\
bottle          & 70  & 100 & 120 & 60  & 822 (6x)   & 80.8           & \textbf{80.8}  \\
bottle or chair & 120 & 120 & 140 & 100 & 1807 (13x) & \textbf{83.45} & 82.35          \\ \bottomrule
\end{tabular}%
}
\caption{Different hyperparameters used for the experiment of Sec$\sim$\ref{sec:annotation}. The num samples show how our method achieves a similar accuracy ($-1\%$ to $4\%$ improvement over all) by considering substantially less samples.}
\label{tab:resultshyperparameters}
\end{table}
\begin{table}[ht]
\centering
\resizebox{\columnwidth}{!}{%
\begin{tabular}{@{}lccccccc@{}}
\toprule
 &
  \begin{tabular}[c]{@{}c@{}}Validation\\ Type\end{tabular} &
  \begin{tabular}[c]{@{}c@{}}Object \\ Annotated\end{tabular} &
  AnchorAL &
  SEALS &
  Random &
  All &
  \textbf{Our} (\texttt{BSLB}) \\ \midrule
\multirow{6}{*}{$\numroundsexplore=80$} &
  easy-valid &
  chair & 94.5$\pm$1.0 & 93.2$\pm$1.5 & 95.7$\pm$1.0 &\textbf{ 96.0$\pm$1.4} & 92.3$\pm$1.4
   \\
 &
   &
  car &95.8$\pm$2.5 & 95.3$\pm$2.0 & 97.3$\pm$1.2 & \textbf{98.2$\pm$1.1} & 95.7$\pm$1.7
   \\
 &
   &
  bottle &
   95.0$\pm$0.9 & 95.0$\pm$1.1 & 96.2$\pm$1.9 & 96.7$\pm$1.8 & \textbf{96.8$\pm$1.2}
   \\ \cmidrule(l){2-8} 
 &
  \textbf{hard-valid} &
  chair &72.0$\pm$3.0 & 71.0$\pm$1.4 & 67.0$\pm$5.0 & 69.4$\pm$4.1 & \textbf{73.4$\pm$2.4}\\
 &
   &
  car &66.4$\pm$7.9 & 69.2$\pm$5.6 & 52.6$\pm$8.2 & 58.8$\pm$6.8 & \textbf{74.0$\pm$2.7} \\
 &
   &
  bottle &61.2$\pm$1.5 & 61.8$\pm$2.0 & 51.2$\pm$2.9 & 52.6$\pm$3.1 & \textbf{63.2$\pm$2.9} \\ \midrule
\multirow{6}{*}{$\numroundsexplore=60$} &
  easy-valid &
  chair &94.8$\pm$2.1 & 93.7$\pm$1.7 & 95.5$\pm$2.0 & \textbf{95.5$\pm$1.4} & 94.8$\pm$1.8

   \\
 &
   &
  car &96.2$\pm$1.9 & 95.5$\pm$1.9 & 97.5$\pm$1.2 & \textbf{98.3$\pm$1.1} & 97.0$\pm$2.3
\\
 &
   &
  bottle &95.0$\pm$1.2 & 94.8$\pm$1.0 & 97.2$\pm$1.7 & \textbf{97.5$\pm$1.7} & 96.0$\pm$1.9
\\ \cmidrule(l){2-8} 
 &
  \textbf{hard-valid} &
  chair &70.0$\pm$5.5 & 70.0$\pm$3.8 & 68.0$\pm$4.2 & 69.8$\pm$4.0 & \textbf{72.4$\pm$1.7}
 \\
 &
  \textbf{} &
  car &65.8$\pm$8.7 & 70.0$\pm$7.5 & 53.2$\pm$5.4 & 60.6$\pm$8.1 & \textbf{72.6$\pm$3.9}
 \\
 &
 &
  bottle & 61.6$\pm$1.0 & 61.6$\pm$2.3 & 53.4$\pm$2.6 & 54.0$\pm$2.6 & \textbf{62.6$\pm$2.8}
  \\ \midrule
\multirow{6}{*}{$\numroundsexplore=30$} &
  easy-valid &
  chair &94.8$\pm$1.5 & 94.5$\pm$1.2 & 96.3$\pm$1.4 & \textbf{96.7$\pm$1.9} & 94.5$\pm$3.2
 \\
 &
   &
  car &95.3$\pm$2.3 & 95.8$\pm$1.9 & 97.3$\pm$1.2 & \textbf{98.2$\pm$1.1} & 92.0$\pm$11.6
 \\
 &
   &
  bottle &95.3$\pm$0.8 & 95.0$\pm$0.5 & 96.2$\pm$1.9 & 96.7$\pm$1.8 & \textbf{96.7$\pm$1.2}
\\ \cmidrule(l){2-8} 
 &
  \textbf{hard-valid} &
  chair &69.4$\pm$2.1 & \textbf{71.2$\pm$2.4} & 67.6$\pm$4.9 & 69.8$\pm$4.0 & 70.8$\pm$4.5
\\
 &
  \textbf{} &
  car &64.2$\pm$9.5 & 67.6$\pm$7.7 & 52.6$\pm$8.2 & 58.8$\pm$6.8 & \textbf{70.8$\pm$6.5}
\\
 &
  \multicolumn{1}{l}{} &
  bottle &60.2$\pm$1.9 & 62.0$\pm$2.1 & 51.2$\pm$2.9 & 52.6$\pm$3.1 & \textbf{63.0$\pm$4.9} \\ \bottomrule
\end{tabular}%
}
\caption{Learning Accuracies on Different Methods for Image Classification in PASCAL-VOC 2012: Effect of $\numroundsexplore$ with the number of rounds fixed at $\numrounds=120$ and with 120 \textbf{easy-valid} and 100 \textbf{hard-valid} samples. }
\label{tab:resultsextra}
\end{table}

\subsection{How does the tail of the parameter matter?}
\begin{figure}[ht]
    \centering
    \includegraphics[width=0.5\linewidth]{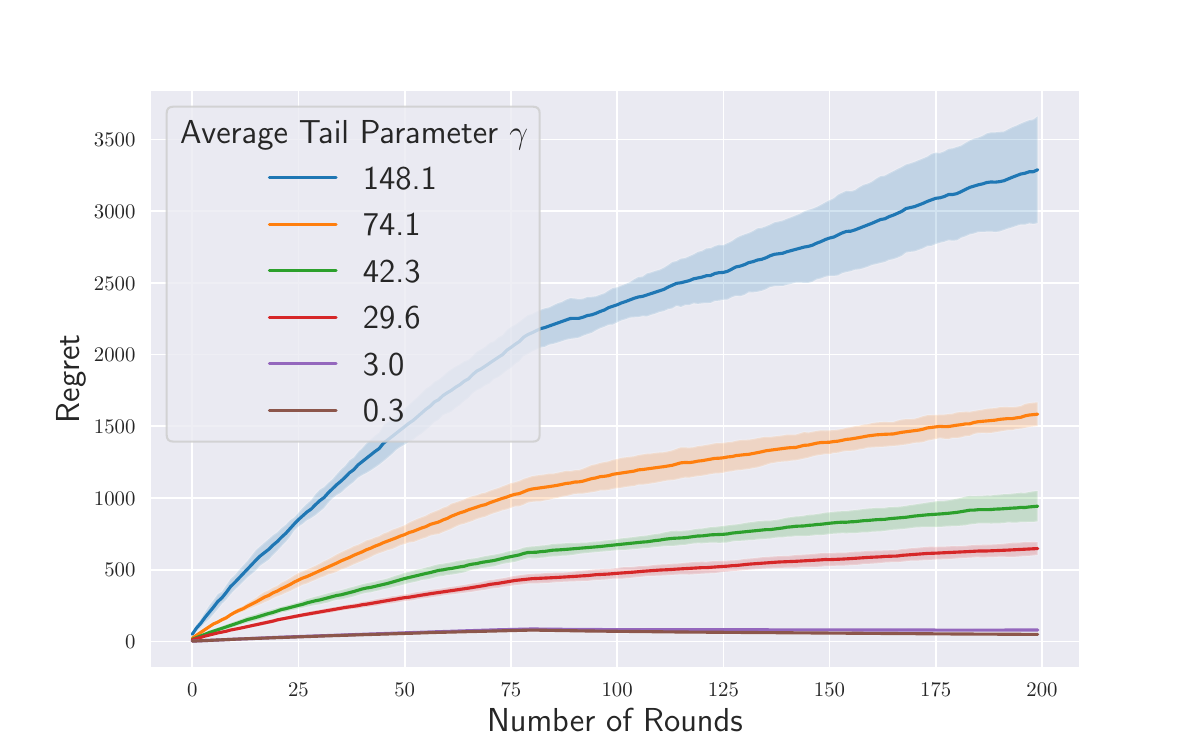}
    \caption{Effect of the tail parameter $\restrictedconeconstant$ on the performance of the \texttt{BSLB} algorithm with $\numroundsexplore=80$. As the tail increases in magnitude the cummulative regret worsens (increases). However observe that our algorithm is still robust to reasonably large tail $\restrictedconeconstant=75$.  }
    \label{fig:taildependence}
\end{figure}
In Figure~\ref{fig:taildependence}, we investigate the effect of the tail parameter in the performance of \texttt{BSLB} with a fixed exploration period $\numroundsexplore=80$ and different sizes of the tail in the same setup as the simulation study of Appendix~\ref{app:simulation}. We observe that as the tail parameter $\restrictedconeconstant$ grows, the regret worsens, however we remark that even for a decent $\restrictedconeconstant = 75$, the performance is reasonable.

\subsection{Convergence of CORRAL parameters in C-BSLB}
We plot the convergence of the CORRAL parameters of \texttt{C-BSLB} in Figure~\ref{fig:convergencecorral} for the simulated experiment of Appendix~\ref{app:corrallingexp}. We observe that the probability of sampling the best algorithm ($\numroundsexplore=T_3=80$) increases with rounds. Note that since the experiments were run on a limited resource machine, we could only do $\dimarm=1000$, and for our setup $\numrounds\ll \dimarm$ has to be sufficiently low ($500$ in this case). This is not enough for the CORRAL algorithm to truely exploit the best possible algorithm in \texttt{C-BSLB} but as we see in Figure~\ref{fig:cummulativesimulated}, however it still achieves the a decent performance.
\begin{figure}[ht!]
    \centering
    \includegraphics[width=0.5\linewidth]{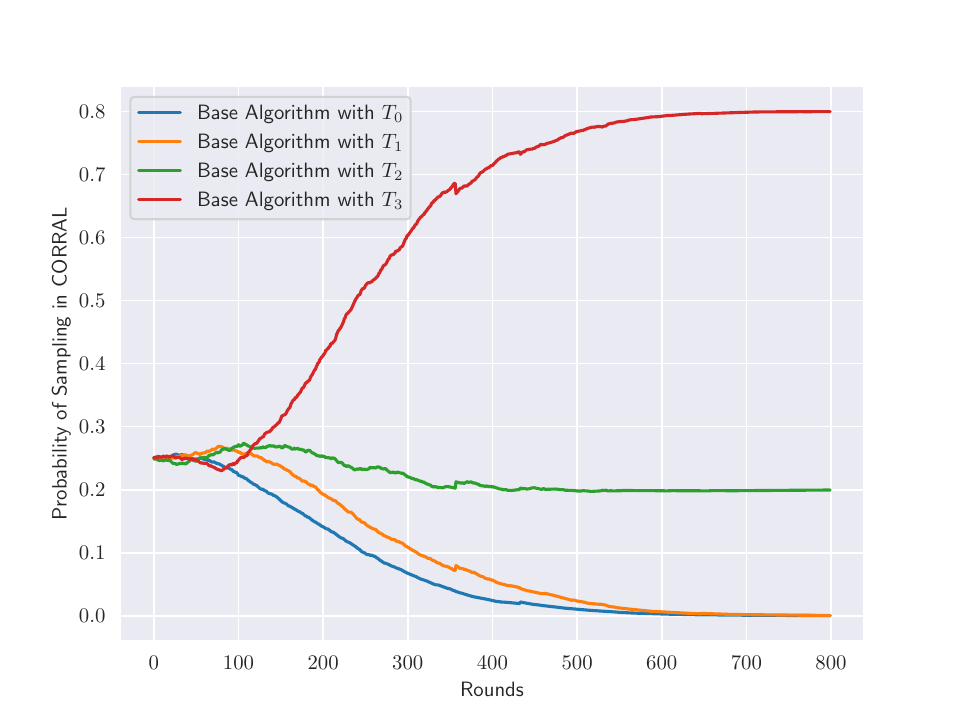}
    \caption{Convergence of the different sampling probabilities for the base algorithms of the \texttt{C-BSLB} (Algorithm~\ref{alg:corral}). This plot is with respect to the simulation study parameters. We can observe that the probability for the best algorithm ($\numrounds_3$) improves with each iteration and for the worst performing algorithm ($\numrounds_0$) decays to $0$.  }
    \label{fig:convergencecorral}
\end{figure}

\begin{figure}[ht!]
    \centering
        \includegraphics[width=0.5\columnwidth]{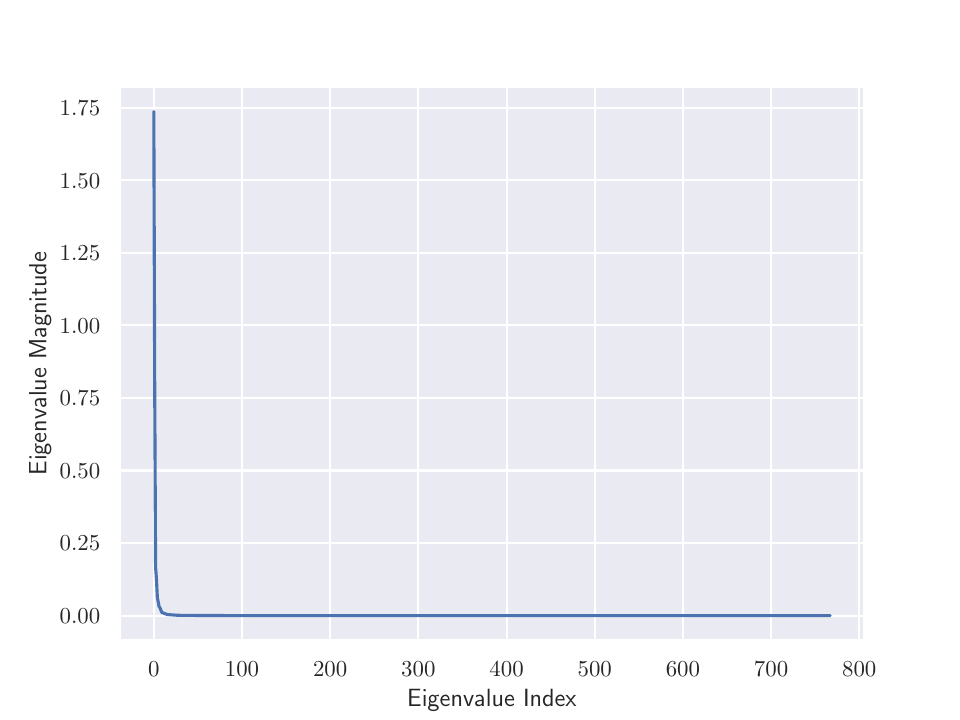}%
    \caption{Eigenvalue spectrum of the embeddings of the two dataset show exponential decay in the eigenvalues, which implies that a uniformly random sample covers the set optimally with high probability because the data is primarily shaped by a few directions. For the PASCAL-2012 on object \textbf{chair} with ViT Base Patch16-224 embeddings.}
    \label{fig:eigenvaluespectrum}
\end{figure}

\begin{figure}[ht!]
    \centering
        \includegraphics[width=0.5\columnwidth]{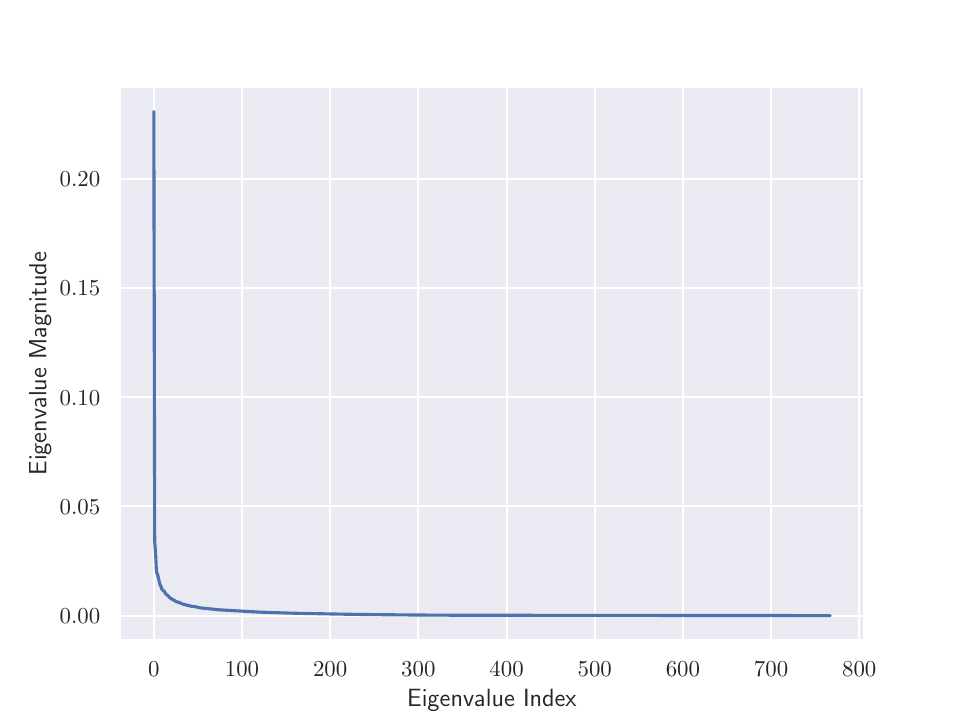}%
    \caption{Eigenvalue spectrum of the embeddings of the two dataset show exponential decay in the eigenvalues, which implies that a uniformly random sample covers the set optimally with high probability because the data is primarily shaped by a few directions. Balanced Sample ($2500$ datapoints) of SST-2 with All MPNet Base V2 embeddings.}
    \label{fig:enter-label}
\end{figure}

\section{Related Work}\label{app:relatedwork}
The Appendix comprises three sections: Section~\ref{app:relatedwork} discusses the related work, Section~\ref{sec:proofs} details the detailed proofs of the technical results and Section~\ref{app:numerical} presents additional numerical results.

\subsection{High Dimensional Sparse Bandits}Recent research on high-dimensional sparse linear bandits has focused on developing algorithms that can effectively handle the challenges posed by high-dimensional feature spaces while exploiting the underlying sparsity. \citet{hao2020high} made significant contributions by establishing a novel $\Omega(\numsamples^{2/3})$ dimension-free minimax regret lower bound for sparse linear bandits in the data-poor regime, where the horizon is smaller than the ambient dimension. They complemented this with a nearly matching upper bound for an explore-then-commit algorithm, demonstrating that $\orderof(\numsamples^{2/3})$ is the optimal rate in the data-poor regime. Building upon these foundations, ~\citet{pmlr-v162-li22a} proposed a simple unified framework for high-dimensional bandit problems, including sparse linear bandits. Their Explore-the-Structure-Then-Commit (ESTC) algorithm achieved comparable regret bounds in the LASSO bandit problem and provided a general framework for the contextual bandit setting using low-rank matrices and group sparse matrices. ~\citet{efficientsparsehds} further improved the algorithm of \citet{hao2020high} for computational complexity using ideas from random projection. On the other hand, ~\citet{thompsonhds} has recently looked at contextual high dimensional bandits with exact sparsity constraints using Thompson sampling with a prior on the sparsity parameter and ~\citet{hdswosparsity} have studied high dimensional bandits without a strict sparsity constraint. However both these work give asymptotic bounds on the regret whereas our work provides finite-sample bounds on the regret. Further the existing work does not explicitly quantify the effect of the parameter tail on the regret or consider a \textit{blocking} constraint, which are the two primary novel considerations in our paper.

\subsection{Personalized Recommendation} Related work on privacy-preserving recommender systems highlights the tension between personalization and user privacy. Traditional recommender systems, such as collaborative filtering and matrix factorization, often rely on centralized data collection and processing, which creates privacy risks for users. These risks include potential re-identification of anonymized data and unauthorized access to sensitive information~\citep{privacyrisks}. To address these concerns, researchers have explored various privacy-preserving techniques, including differential privacy, secure multi-party computation, homomorphic encryption, and federated learning~\citep{privacymatrixfactor}. These approaches aim to protect user data while maintaining recommendation accuracy. However, challenges remain in balancing privacy protection with system efficiency and recommendation quality. Our work is most closely related to recent studies on context-aware and hybrid recommendation systems that incorporate privacy-preserving mechanisms explicitly using \textit{blocking} constraints~\citep{pal2024blocked}. 

\subsection{Annotation in Niche Applications with limited expert annotators}

 A second application of our framework is identifying \textit{hard datapoints} for annotation in expensive labeling tasks. Large Language Models (LLMs) offer strong zero-shot capabilities, making it easier to prototype solutions for downstream tasks. However, they struggle with complex domain-specific queries when relevant training data is scarce or evolving~\cite{farr2024llm}. In-context learning with few shot examples has emerged as a powerful approach, where a small set of high-quality examples improves model performance~\citep{dong2022survey}. Crucially, hard examples provide better domain-specific information~\citep{baek2024revisiting, liu2024let, mavromatis2023examples}, but identifying them is challenging. Heuristic-based selection often leads to noisy, mislabeled, or outlier examples~\citep{mindermann2022prioritized}.
Alternatively, we can leverage domain experts to assign a hardness score while annotating. This data-poor problem can be framed in a bandit framework, where unlabeled datapoints act as arms and are sequentially annotated while hardness scores are modeled as a sparse linear function of embeddings. 
In domains with very few annotators—sometimes only one—it is impractical to re-query the same datapoint, naturally leading to a \textit{blocking constraint}. Beyond annotation, high-quality hard examples are also valuable in model training, where they improve generalization and efficiency ~\citep{sorscher2022beyond, maharana2023d2}.

\citet{sorscher2022beyond} demonstrated that selecting high-quality data can reduce the power-law association of test error with dataset size to an exponential law. 
In annotation-expensive tasks with large volumes of unlabeled data, the challenge is to select a representative subset of datapoints for labeling. 
Conceptually our work is similar to \textit{active learning} (AL)~\citep{settles2009active,activelearning2024} where unlabeled samples are annotated adaptively, based on the confidence of a trained model~\citep{activelearning2022}. 
Active learning works well with good initialization and informative confidence intervals. 
However, in our label-scarce setting, AL is particularly challenging with complex data due to the absence of a reliably trained model in the first place - this is more pronounced for difficult datapoints for which prediction is hard. AL needs an initial set of high-quality labeled samples to reasonably train a model  - also known as the \textit{cold-start} problem - when labels are scarce, uncertainty based sampling techniques are unsuitable \citep{li2024survey}. 
Our goal is to identify informative samples with the help of the expert annotator(s), whom the final model aims to emulate. 
\textit{Coreset selection} \citep{guo2022deepcore,albalak2024survey,sener2018active} aims to select a subset of datapoints for training. However, coreset selection assumes that a large amount of \textit{labeled data} already exists, and the focus is on reducing computational costs. In contrast, our setting deals with the lack of labeled data, making existing coreset selection approaches, which rely on the entire labeled dataset, inapplicable. 
Our work also aligns with \textit{curriculum learning}~\citep{bengio2009curriculum}, where a model is trained on samples of increasing difficulty/complexity. Due to the ambiguity in hardness definition, often heuristics are used to infer the difficulty of samples~\citep{soviany_curriculum_2022} and can turn out unreliable and not generalizable.
For niche tasks where an expert annotator is available, the difficulty ratings from the annotator are more informative than heuristics since the goal is to train a model to mimic the expert~\citep{img-difficulty-CVPR-2016}. 
In computer vision, there has been recent work regarding estimating the difficulty of a dataset for the model using implicit difficulty ratings of annotation. For NLP tasks, \citet{pmlr-dataset-difficulty} constructs information theoretic metrics to estimate the difficulty of datapoints.

\section{Technical Proofs}\label{sec:proofs}

\input{proofs}


\end{document}

%% file: math_commands.tex
\usepackage{amsmath,amsfonts,bm}

\def\eqref#1{equation~\ref{#1}}

\def\1{\bm{1}}

\DeclareMathAlphabet{\mathsfit}{\encodingdefault}{\sfdefault}{m}{sl}
\SetMathAlphabet{\mathsfit}{bold}{\encodingdefault}{\sfdefault}{bx}{n}

\newcommand{\R}{\mathbb{R}}

\DeclareMathOperator*{\argmax}{arg\,max}

%% file: input.tex
\newtheorem{theorem}{Theorem}
\newtheorem{definition}{Definition}

\newtheorem{lemma}{Lemma}

\newtheorem{coro}{Corollary}

\theoremstyle{remark}
\newtheorem{remark}{Remark}
\newtheorem{insight}{Insight}
\newcommand{\f}[1]{\boldsymbol{#1}}
\newcommand{\bb}[1]{\mathbb{#1}}
\newcommand{\fl}[1]{\mathbf{#1}}
\newcommand{\ca}[1]{\mathcal{#1}}

\newcommand{\s}[1]{\mathsf{#1}}

\newcommand{\orderof}{\s{O}}

\newcommand{\probabilityexploresucc}{\nu}
\newcommand{\armidx}{j}
\newcommand{\armidxalt}{i}
\newcommand{\numarms}{\s{M}}
\newcommand{\dimarm}{d}
\newcommand{\armset}{\ca{A}}
\newcommand{\armsubset}{\ca{G}}
\newcommand{\arm}{\fl{a}}
\newcommand{\numrounds}{\s{T}}
\newcommand{\numroundsexploit}{\numrounds_{\text{exploit}}}
\newcommand{\timeindex}{t}
\newcommand{\reward}{r}

\newcommand{\indicator}{\mathbbm{1}}
\newcommand{\sparsity}{k^{\star}}
\newcommand{\parameter}{\f{\theta}}

\newcommand{\noise}{\eta}
\newcommand{\expectation}{\mathbb{E}}

\newcommand{\parameterestimate}{\widehat{\parameter}}

\newcommand{\permutation}{\pi}
\newcommand{\permutationapprox}{\widehat{\pi}}

\newcommand{\indicatorvector}{e}

\newcommand{\numroundsexplore}{\numrounds_{\text{explore}}}

\newcommand{\lagrange}{\lambda}

\newcommand{\transpose}{\mathsf{T}}

\newcommand{\constant}{\zeta}
\newcommand{\covariancematrix}{\Sigma}
\newcommand{\designmatrix}{\mathbf{X}}
\newcommand{\samplingdist}{\bm{\mu}}
\newcommand{\tightness}{\Lambda}
\newcommand{\RE}{\mathsf{RE}}

\newcommand{\probability}{\mathbb{P}}
\newcommand{\randomvector}{Y}

\newcommand{\numsamples}{n}

\newcommand{\mineigvalue}{\lambda_{\min}}

\newcommand{\sparsityaux}{m}
\newcommand{\subgaussiansym}{\psi}
\newcommand{\vectorsym}{\fl{v}}
\newcommand{\cone}{\mathsf{C}}

\newcommand{\approximationerror}{\fl{h}}
\newcommand{\errorvector}{\fl{w}}

\newcommand{\sparsityk}{k}
\newcommand{\topt}{\ca{T}}
\newcommand{\sparsitytail}{\beta}

\newcommand{\armbound}{M}

\newcommand{\randomvectoralt}{\fl{z}}

\newcommand{\diag}{\mathsf{diag}}
\newcommand{\probabilityspace}{\mathcal{P}}

\newcommand{\randomvariable}{Z}
\newcommand{\function}{f}

\newcommand{\functionvar}{x}
\newcommand{\vectorsymalt}{z}

\newcommand{\event}{\mathcal{E}}
\newcommand{\empericalcovariancematrix}{\hat{\covariancematrix}}

\newcommand{\unitball}{\mathcal{B}^\dimarm}

\newcommand{\indicatorprob}{p}

\newcommand{\sparsemin}{\sparsity_1}

\newcommand{\radamacher}{\xi}
\newcommand{\subgaussianbound}{\Psi}

\newcommand{\indicatorelement}{p}

\newcommand{\randomvariablealt}{W}
\newcommand{\complement}{c}
\newcommand{\bestregret}{\expectedregret_*}

\newcommand{\rmax}{R_{\max}}
\newcommand{\regret}{\s{Reg}(\numrounds)}
\newcommand{\expectedregret}{\expectation[\regret]}
\newcommand{\armmatrix}{\mathbf{A}}

\newcommand{\searchbound}{\hat{u}}
\newcommand{\optsubsetsize}{u^*}
\newcommand{\restrictedeigenvalue}{K}
\newcommand{\restrictedconeconstantvanilla}{\gamma}
\newcommand{\restrictedconeconstant}{\beta_\sparsityk}
\newcommand{\rewardvector}{\mathbf{r}}
\newcommand{\noisevector}{\fl{\noise}}
\newcommand{\mineigvalueempirical}{\Lambda}

\newcommand{\zerovector}{\mathbf{0}}
\newtheorem{fact}{Fact}

%% file: proofs.tex
 \subsection{Proof of Lower Bound for Regret (Theorem~\ref{thm:lower})}

\begin{proof}
We consider the hard-sparsity instance of the high-dimensional linear bandit setting with the \textit{blocking} constraint. We prove this in three steps. First we show how one can construct an equivalent bandit problem with a \textit{blocking} constraint for any problem without the \textit{blocking} constraint. Next we use results from~\cite{hao2020high} to show that the result holds true for this transformation. Finally, we show that this is the best  
\begin{enumerate}
    \item For any instance of the bandit problem without the \textit{blocking}  constraint, we can construct a bandit problem with the \textit{blocking} constraint. This can be done by considering $\mathsf{T}$ copies of each of the arms, i.e. for the original arm set $\mathcal{A}=\{a^{(1)},\dots,a^{(\mathsf{M})}\}$ we construct the arm multiset $\mathcal{A}^\prime$ as follows,
    \[
    \mathcal{A}^\prime = \cup_{i=1}^{\mathsf{M}} \cup_j^{\mathsf{T}} \{a^{(i)}_j\}
    \]
    where $a^{(i)}_j$ denotes the $j^\text{th}$ copy of the $i^\text{th}$ arm (such that it is a different arm with the same arm vector). Since it is a multiset, the union operation double-counts duplicate values. Now the bandit setting with blocking constraint with arm set $\mathcal{A}^\prime$ is identical to the bandit setting without blocking constraint with arm set $\mathcal{A}$. Further the regret decomposition becomes identical in the first term, i.e. the top $\mathsf{T}$ arms have the same arm vector.
    \item Now for any arm set $\mathcal{A}$ without the blocking constraint, the lower bound from Theorem 3.3 of~\cite{hao2020high} holds for any algorithm, and the following bound holds
    \[
    \mathbb{E}[\mathsf{R}] = \Omega(\min(\mineigvalue^{-1/3}(\mathcal{A})k^{1/3}\mathsf{T}^{2/3}),\sqrt{d\mathsf{T}})).
    \]
    \item Now if any algorithm operating with the blocking constraint could achieve the regret of order lesser than $\min(k^{1/3}\mathsf{T}^{2/3}),\sqrt{d\mathsf{T}})$, then the algorithm would solve the bandit problem with arm multiset $\armset^\prime$ (with the blocking constraint) with regret lower than $\min(k^{1/3}\mathsf{T}^{2/3}),\sqrt{d\mathsf{T}})$. But then we can solve the original problem with arm set $\mathcal{A}$ with the same regret, hence arriving at a contradiction. 
\end{enumerate}
In the data-poor regime $d\geq k^{1/3}\mathsf{T}^{2/3}$, which is the regime considered in the paper, this bound reduces to $\Omega(\mineigvalue^{-1/3}(\mathcal{A})k^{1/3}\mathsf{T}^{2/3})$, which is the order that our upper bound achieves. 

\end{proof}

   \subsection{Proof of Theorem~\ref{lemma:fastsparselearning}}\label{app:sparsereg}
 \textit{Preliminaries and Basis Pursuit Program}
   
   We need the following definition of a restricted eigenvalue of a matrix $\designmatrix$,
   \begin{definition}\textbf{(Restricted Eigenvalue) }
       If $\designmatrix$ satisfies  Restricted Eigenvalue property $\RE(\sparsityk_0,\restrictedconeconstantvanilla,\designmatrix)$, then there exists a constant $\restrictedeigenvalue(\sparsityk_0,\restrictedconeconstantvanilla,\designmatrix)$ such that for all $\vectorsymalt \in \R^\dimarm$ and $\vectorsymalt\neq \fl{0}$,
\begin{align*}
{\restrictedeigenvalue(\sparsityk_0,\restrictedconeconstantvanilla,\designmatrix)} = \min_{J \subseteq \{1,\dots,\dimarm \} |J| \leq \sparsityk_0}\min_{\Vert \vectorsymalt_{J^\complement} \Vert_1 \leq \restrictedconeconstantvanilla\Vert \vectorsymalt_{J} \Vert_1}\frac{\Vert \designmatrix \vectorsymalt \Vert_2 }{\Vert \vectorsymalt_J \Vert_2}.
\end{align*}
   \end{definition}
This definition implies, 
\begin{align*}\restrictedeigenvalue&(\sparsityk_0,\restrictedconeconstantvanilla,\designmatrix)\Vert \vectorsymalt_J \Vert_2 \leq \Vert \designmatrix \vectorsymalt \Vert_2  \ \forall \vectorsymalt,\quad  \text{such that} \quad \Vert \vectorsymalt_{J^\complement} \Vert_1 \leq \restrictedconeconstantvanilla\Vert \vectorsymalt_{J} \Vert_1 \ \forall J \subseteq \{1,\dots,\dimarm \} , |J| \leq \sparsityk_0.
\end{align*}


   We now prove the theorem. 
\begin{proof}
\newcommand{\topk}{\mathcal{Y}}
  Let  $\approximationerror = \parameterestimate - \parameter$ denote the error in estimation of the parameter, $\errorvector =\rewardvector - \designmatrix\parameter$ denote the noise vector, $\lagrange$ is the Lagrange parameter in Line 9 of Algorithm~\ref{alg:greedy} and $\numsamples$ the number of samples. 

 For the vector $\approximationerror$, let $\topk_0$  denotes the top-$\sparsityk$ coordinates by absolute value , $\topk_1$ the next top $\sparsityk$ coordinates and so on  . Let $\vectorsym_{\topk_i}$ denote the vector where the $\topk_i$ coordinates are equal to the coordinates of $\vectorsym$ and all the other coordinates are 0. Let $\alpha = \Vert \approximationerror_{({\topk_0 \cup \topk_1})}\Vert_{2}$.

    \newcommand{\lagrangian}{\lambda}
    The following inequality holds for a Lagrangian Lasso program (w/o any assumptions on sparsity) (See Eq.7.29 Proof of Theorem 7.13 Part (a) of ~\citet{Wainwright_2019}; note this holds true without any assumptions on the sparsity or the design matrix but just by optimality argument of the Lasso program), 
    \begin{align}\label{eq:bpptwo}
       \frac{1}{\numsamples} \Vert \designmatrix \approximationerror \Vert_2^2 \leq  \frac{\errorvector^\transpose \designmatrix \approximationerror}{\numsamples} + \lagrange(\|\parameter\|_1 - \|\parameterestimate\|_1)
    \end{align}
    Now by triangle inequality $\|\parameter\|_1 - \|\parameterestimate\|_1\leq \|\parameter-\parameterestimate\|_1 = \|\approximationerror\|_1$  and therefore,
     \begin{align*}
       \frac{1}{\numsamples} \Vert \designmatrix \approximationerror \Vert_2^2 \leq  \frac{\errorvector^\transpose \designmatrix \approximationerror}{\numsamples} + \lagrange\|\approximationerror\|_1 \leq \left( \Vert \frac{\designmatrix^\transpose\errorvector}{\numsamples} \Vert_\infty + \lagrange \right)\|\approximationerror\|_1
    \end{align*}
(The last step follows as a result of using Hölder's inequality). 
Let $B= \left( \Vert \frac{\designmatrix^\transpose\errorvector}{\numsamples} \Vert_\infty + \lagrange \right)$, which gives us,
    \begin{align}\label{eq: lagrangian re condition}
       \frac{1}{\numsamples} \Vert \designmatrix \approximationerror \Vert_2^2  \leq B\|\approximationerror\|_1
    \end{align}

We lower bound $\frac{1}{\numsamples} \Vert \designmatrix \approximationerror \Vert_2^2$ using the RE condition and in order to do that we need to show that $\approximationerror$ lies in the cone $\{\vectorsymalt| \vectorsymalt\in\R^\dimarm, \|\vectorsymalt_\topk^0\|\leq 4(1+\restrictedconeconstant)\|\vectorsymalt_{\topk_0^C}\|\}$, which is the subset of the cone the RE condition is satisfied on.
Note that for the vector $\approximationerror$ the following condition holds, 
\begin{align}\label{eq: bound on topt complement}
\begin{split}
  \Vert \approximationerror_{{\topk_0}^C}\Vert_{1}  &= \Vert \approximationerror_{{\topk_1}}\Vert_{1} + \Vert \approximationerror_{({\topk_0 \cup \topk_1})^C}\Vert_{1}  \overset{(a)}{\leq}\Vert \approximationerror_{{\topk_1}}\Vert_{1} + \Vert \approximationerror_{({\topk_0 \cup \topk_1})}\Vert_{1} + 2 \sparsitytail_\sparsityk \\&\overset{(b)}{\leq} \Vert \approximationerror_{{\topk_1}}\Vert_{1}  + \Vert\approximationerror_{({\topk_0 \cup \topk_1})}\Vert_{1} + 2\restrictedconeconstant\Vert \parameter^*_{\topk_0}\Vert_1  \\ &\overset{(c)}{\leq} 2\Vert \approximationerror_{{\topk_0}}\Vert_{1}  + 2\Vert\approximationerror_{{\topk_0 }}\Vert_{1} + 4\restrictedconeconstant\Vert \approximationerror_{\topk_0}\Vert_1  \leq 4(1+\gamma)\Vert \approximationerror_{\topk_0}\Vert_1,
  \end{split}
\end{align}
where (a) follows from the decomposition available in Theorem 1.6 in~\citet{boche2015survey}, 
\begin{align*}
     \Vert \approximationerror_{(\topk_0\cup\topk_1)^\complement}\Vert_1 \leq   \Vert \approximationerror_{(\topk_0)^\complement} \Vert_1 \leq \Vert \approximationerror_{(\topk_0)}  \Vert_1 + 2\sparsitytail_\sparsityk \leq \| \approximationerror_{(\topk_0\cup \topk_1)}  \Vert_1 + 2\sparsitytail_\sparsityk
    \end{align*}
(b) follows from the definition of the tail $\sparsitytail_\sparsityk = \restrictedconeconstant \| \parameter^*_{\topk_0}\|_1$

and (c) follows from $\| \parameter^*_{\topk_0}\|_1 \leq \Vert \parameter^*_{\topk_0} - \parameterestimate_{\topk_0} + \parameterestimate_{\topk_0} \Vert_1 \leq \Vert \parameter^*_{\topk_0} - \parameterestimate_{\topk_0}\Vert + \Vert \parameterestimate_{\topk_0} \Vert_1 \leq \Vert \approximationerror_{\topk_0} \Vert + \Vert \parameterestimate_{\topk_0} \Vert_1 \leq 2\Vert \approximationerror_{\topk_0} \Vert_1  $. The last inequality holds since $\parameterestimate$ is a solution of~\eqref{eq:bpptwo} and $\Vert \parameterestimate_{\topk_0}\Vert_1 \leq \Vert \parameterestimate \Vert_1 \leq \Vert\approximationerror_{\topk_0} + \zerovector_{\topk_0^\complement} \Vert_1= \Vert\approximationerror_{\topk_0} \Vert_1$ otherwise $\approximationerror_{\topk_0}+ \zerovector_{\topk_0^\complement}$ would be the solution instead. 

Now we can use the RE condition on the design matrix $\designmatrix$ using the vector $\approximationerror$, which is such that $\|\approximationerror_{\topk_0^\complement}\| \leq (4+4\restrictedconeconstant)\|\approximationerror_{(\topk_0)}\|$ the following holds,
\begin{align}\label{eq:bppthree}
 {\restrictedeigenvalue^2(\sparsityk_0,4+4\restrictedconeconstant,\frac{\designmatrix}{\sqrt{n}})} \Vert \approximationerror_{\topk_0 }\Vert_2^2 \leq \frac{\Vert \designmatrix \approximationerror\Vert_2^2}{\numsamples} .
\end{align}
Putting this into~\eqref{eq: lagrangian re condition}, and letting $\restrictedeigenvalue=\restrictedeigenvalue(\sparsityk_0,4+4\restrictedconeconstant,\frac{\designmatrix}{\sqrt{n}})$,
\begin{align*}
&{\restrictedeigenvalue^2} \Vert \approximationerror_{\topk_0} \Vert_2^2 \leq B \Vert \approximationerror\Vert_{1} = B(\Vert \approximationerror_{({\topk_0 })}\Vert_{1} + \Vert \approximationerror_{({\topk_0})^C}\Vert_{1})  \overset{(a)}{\leq}B(5+4\restrictedconeconstant)\|\approximationerror_{({\topk_0 })}\Vert_{1} 
\end{align*}
where (a) follows from~\eqref{eq: bound on topt complement}. Therefore,
\begin{align*}
    \frac{\restrictedeigenvalue^2}{\sparsityk} \Vert \approximationerror_{\topk_0} \Vert_1^2 \overset{(a)}{\leq} {\restrictedeigenvalue^2} \Vert \approximationerror_{\topk_0} \Vert_2^2 \leq B(5+4\restrictedconeconstant)\|\approximationerror_{({\topk_0 })}\Vert_{1},
\end{align*}
where (a) is due to Cauchy Schwarz. So we have, 
\begin{align*}
    \Vert \approximationerror_{\topk_0} \Vert_1 \leq \frac{B\sparsityk(5+4\restrictedconeconstant)}{\restrictedeigenvalue^2}
\end{align*}
which we combine with $\|\approximationerror\|_1{\leq}(5+4\restrictedconeconstant)\|\approximationerror_{({\topk_0 })}\Vert_{1} $ to obtain,
\begin{align}\label{eq: proof final l1 bound in B}
    \Vert \approximationerror \Vert_1 \leq \frac{B\sparsityk(5+4\restrictedconeconstant)^2}{\restrictedeigenvalue^2}
\end{align}
Now we bound $\Vert \frac{\designmatrix^\transpose\errorvector}{\numsamples} \Vert_\infty$ using the properties of a bounded zero-mean noise from~\citet{Wainwright_2019}.
Consider the random variables, $\frac{(\designmatrix^\transpose\errorvector)_j}{n}$ each of which is a weighted sum of independent random variables. Each of the term of the sum is subgaussian with parameter bounded by $\frac{\sigma\armbound}{n}$, where $\armbound$ is the $\ell_\infty$ bound. We have $\armbound = 1$.  Using a standard sub-Gaussian concentration bound $\probability( \frac{ (\designmatrix^\transpose\errorvector)_j }{\numsamples}\geq t) \leq 2\exp(-\frac{nt^2}{2\sigma^2})$. In Set $t= {2\constant_2\sigma}(\sqrt{\frac{\log d}{n}})$, then $\probability( \frac{ (\designmatrix^\transpose\errorvector)_j }{\numsamples}\geq {2\constant_2}\sigma(\sqrt{\frac{\log \dimarm}{n}}) \leq 2\exp(-2\constant_2^2\log d) = 2\dimarm^{-2\constant_2^2}$.
Using a union bound on $\dimarm$ on the probability $\probability(\Vert \frac{ \designmatrix^\transpose\errorvector }{\numsamples}\Vert_{\infty}\geq{\constant_2}(\sqrt{\frac{\log d}{n}})) = \probability(\max_{j\in[\dimarm]} \frac{ (\designmatrix^\transpose\errorvector)_j }{\numsamples}\geq{2\constant_2}\sqrt{\frac{\log d}{n}}) \leq \sum_{j\in[\dimarm]} \probability( \frac{ (\designmatrix^\transpose\errorvector)_j }{\numsamples}\geq{2\constant_2}\sqrt{\frac{\log d}{n}}) \leq \dimarm^{1-2\constant_2^2}$. 
Therefore for a large enough constant $\constant_2$, $B \leq 2\constant_2\sigma (\sqrt{\frac{2\log{\dimarm}}{\numsamples}})$ with high probability $1-\dimarm^{1-2\constant_2^2}$. 

Combining this with~\eqref{eq: proof final l1 bound in B} and putting $\lagrange=\sqrt{\frac{\log{\dimarm}}{\numsamples}}$, we have, 
\begin{align*}
      \Vert \approximationerror \Vert_1 \leq (2\sqrt{2}\constant_2\sigma+1) \frac{(5+4\restrictedconeconstant)^2\sparsityk}{\restrictedeigenvalue^2}\sqrt{\frac{\log{\dimarm}}{\numsamples}}
\end{align*}
with probability $1-\dimarm^{1-2\constant_2^2}$.
   \end{proof}

\subsection{Proof of Corollary~\ref{coro:empiricallasso}}\label{app:corollaryproof}
The statement of the Corollary basically extends the estimation guarantees of  Theorem~1 hold with high probability.

First let us denote the expected covariance matrix as (assume $\hat{\armsubset}$ is the sampled subset without replacement from $\armset$),
\begin{align*}
    \covariancematrix = \expectation\left\{\frac{1}{|\hat{\armsubset}|}\sum_{\arm\in\hat{\armsubset}} \arm^\transpose \arm\right\} = \frac{1}{|\hat{\armsubset}|} \sum_{\arm\in\hat{\armsubset}} \expectation \{ \arm^\transpose \arm \} \overset{(a)}{=} \frac{1}{|\armsubset|}\sum_{\arm\in\armsubset} \arm^\transpose \arm, \end{align*}
    where (a) follows from sampling without replacement. 
The proof then hinges on the following key fact and using Theorem~\ref{th:reconcentration} stated later. 
\begin{fact}The minimum eigenvalue defined by $  \mineigvalueempirical=\mineigvalue\left(\frac{1}{|\armsubset|}\sum_{\arm \in \armsubset}\arm^\transpose\arm\right)$ is such that,
\begin{align}\label{eq: bound on min eig value}
\begin{split}
  \mineigvalueempirical&=  \mineigvalue\left(\frac{1}{|\armsubset|}\sum_{\arm \in \armsubset}\arm^\transpose\arm\right) =\min_{\vectorsymalt\in\R^\dimarm} \frac{1}{|\armsubset|}\frac{\|\designmatrix\vectorsymalt\|_2^2}{\|\vectorsymalt\|^2_2} 
\leq \restrictedeigenvalue^2(\sparsityk,4(1+\restrictedconeconstant),\frac{\designmatrix}{
  \sqrt{|\armsubset|}
  }) 
  \end{split}
\end{align}
\end{fact}
We further show using results from~\cite{reconstructionitit} (Proved in Theorem~\ref{th:reconcentration} below.) that the following holds with probability $1-\exp(-\constant_5\numsamples)$,
\begin{align}\label{eq: re inequality hp}
  \frac{1}{2}  \restrictedeigenvalue&(\sparsityk,4(1+\restrictedconeconstant),\Sigma^{1/2})\leq \restrictedeigenvalue(\sparsityk,4(1+\restrictedconeconstant),\empericalcovariancematrix^{1/2})
\end{align}

Therefore the following holds, 
\begin{align*}
  \frac{1}{2}  \restrictedeigenvalue&(\sparsityk,4(1+\restrictedconeconstant),\frac{\designmatrix}{\sqrt{\numsamples}})\geq \mineigvalueempirical^{1/2}
\end{align*}
with probability $1-\exp(-\constant_5\numsamples)$. 

Therefore, 
\begin{align*}
      \Vert \parameter-\parameterestimate \Vert_1 \leq 2\constant_2\sigma \frac{(5+4\restrictedconeconstant)^2\sparsityk}{\restrictedeigenvalue^2}\sqrt{\frac{2\log{\dimarm}}{\numsamples}} = \orderof\left( \frac{(5+4\restrictedconeconstant)^2\sigma\sparsityk}{\mineigvalueempirical}\sqrt{\frac{2\log{\dimarm}}{\numsamples}}\right),
\end{align*}
with probability $1-\dimarm^{1-2\constant_2^2}-\exp(-\constant_5\numsamples)$.

The only additional ingredient needed to prove the corollary is the following concentration inequality of~\eqref{eq: re inequality hp}, which shows that if the RE of the covariance matrix of the set which is used for sampling is bounded from below, then so is the RE of the sampled covariance matrix with high probability. However although the following theorem is adapted directly from~\citet{reconstructionitit}; to prove it we need to make changes accordingly for the sampling without replacement case. This introduces an additional multiplicative $\mineigvalue$  term in the exponential in the probability but does not change the order with respect to any other variable. 

\textit{Concentration Inequality for RE condition to hold on sample covariance matrix given that it holds on the expected covariance matrix. }

We use the following Theorem which is an extension of Theorem 8 from~\citet{reconstructionitit}.
\begin{theorem}\label{th:reconcentration} Let $0 < \delta < 1$ and $0 < \sparsityk_0 < \dimarm$. Let $Y \in 
\mathbb{R}^\dimarm$ be a random vector such that $\|Y\|_{\infty} \leq \armbound$ a.s and denote $\Sigma = \expectation YY^T$. Let $\designmatrix$ be an $\numsamples \times \dimarm$ matrix, whose rows 
$Y_1,\ldots,Y_n$ are sampled without replacement from a set $\armsubset$. Let $\Sigma$ satisfy the 
$\RE(\sparsityk_0, 3\restrictedconeconstantvanilla, \Sigma^{1/2})$ condition as in Definition~\ref{def:re}. Set $\sparsemin = \sparsityk_0 + \sparsityk_0\max_{j}\Vert\covariancematrix^{1/2}\indicatorvector_j \Vert_2^2\times \left(\frac{16\RE(\sparsityk_0,3\restrictedconeconstantvanilla,\covariancematrix^{1/2})^2(3\restrictedconeconstantvanilla)^2(3\restrictedconeconstantvanilla+1)}{\delta^2}\right)$. Assume that $\sparsemin\leq \dimarm$ 
and $\rho = \rho_{\min}(\sparsemin, \Sigma^{1/2}) > 0$, the ($\sparsemin-$sparse) minimum eigenvalue of $\Sigma^{1/2}$. Suppose the sample size satisfies 
for some absolute constant $C$,
\[
\numsamples \geq \frac{C\armbound^2\sparsemin\cdot \log \dimarm}{\rho \delta^2} \cdot \log^3 \left(\frac{C\armbound^2\sparsemin \cdot \log \dimarm}{\rho \delta^2}\right).
\]

Then, with probability at least $1 - \exp(-\delta \rho^2 \numsamples/(6\armbound^2\sparsemin))$,
$\RE(\sparsity_0, \restrictedconeconstantvanilla, \designmatrix)$ condition holds for matrix $\frac{1}{\sqrt{\numsamples}}\designmatrix$ with,
\[  0\leq{(1-\delta)\restrictedeigenvalue(\sparsityk_0, \restrictedconeconstantvanilla, \Sigma^{1/2})}{} \leq \restrictedeigenvalue\left(\sparsityk_0, \restrictedconeconstantvanilla, \frac{1}{\sqrt{\numsamples}}\designmatrix\right).
\]
\end{theorem}
(The inequality is reverse because our definition of Restricted Eigenvalue has the 1/$\restrictedeigenvalue$ compared to the definition of~\citet{reconstructionitit} ). We use the Theorem with $\delta = \frac{1}{2}$

The proof of Theorem~\ref{th:reconcentration} is dependent on Theorem~23, which is reproduced below, and Theorem~10 (Reduction Principle), which is in the paper. Out of these two, only Theorem~23 has elements related to the randomness of the design. 
In the proof of the above theorem, authors use concentration inequality to extend a RIP-like condition on a general cone (rather than sparse vectors). This concentration inequality results from the following theorem: an augmented version of Theorem 22 of~\citet{reconstructionitit} to the sampling without replacement case.

\textit{The original proof of Theorem 22 is extremely involved (and mathematically rich). Reproducing the entire proof would have surmounted to reproducing the entire paper. We only highlight the key difference,  it is recommended that the reader goes through the proof beforehand/side-by-side.}
\begin{theorem}
    Set $1 > \delta > 0$, $0 < \sparsity_0 \leq \dimarm$ and $\tightness_0 > 0$. Let $\armsubset$ be a subset of vectors  such that $||\randomvector||_{\infty} \leq \armbound$, with $\covariancematrix = \sum_{\randomvector\in\armsubset}\frac{1}{|\armsubset|}\randomvector\randomvector^\transpose$. $\covariancematrix$ satisfies $\RE(\sparsity_0,3\tightness_0,\covariancematrix^{1/2})$. Rows of $\designmatrix$ are drawn uniformly without replacement from $\armsubset$.
    Set $\sparsemin = \sparsity_0 + \sparsity_0\max_{j}\Vert\covariancematrix^{1/2}\indicatorvector_j \Vert_2^2\times \left(\frac{16\RE(\sparsity_0,3\tightness_0,\covariancematrix^{1/2})^2(3\tightness_0)^2(3\tightness_0+1)}{\delta^2}\right)$. Assume $\sparsemin \leq \dimarm$ and $\mineigvalue(\sparsemin,\covariancematrix^{1/2})>0$. If for some absolute constant $\constant_{12}$,
    \begin{align*}
        \numsamples \geq \frac{\constant_{12}\armbound^2\sparsemin\log\dimarm}{\mineigvalue(\sparsemin,\covariancematrix^{1/2})\delta^2}\log^3\left(\frac{\constant_{12}\armbound^2\sparsemin\log\dimarm}{\mineigvalue(\sparsemin,\covariancematrix^{1/2})\delta^2}\right)
    \end{align*}
    then with probability $1-\exp(\frac{-\delta\mineigvalue^2(\sparsemin,\covariancematrix^{1/2})\numsamples}{6\armbound^2\sparsemin})$, for all $\vectorsym \in \cone(\sparsity_0,\tightness_0), \ \vectorsym \neq 0$
    \begin{align*}
        1 -\delta \leq \frac{1}{\sqrt{\numsamples}}\frac{\Vert \designmatrix \vectorsym \Vert_2}{\Vert \vectorsym \Vert_2} \leq 1 + \delta.
    \end{align*}
\end{theorem}

\begin{proof}
In the proof of Theorem 22 of~\citet{reconstructionitit}, two arguments require the sampling with replacement (i.i.d. samples), namely symmetrization and Talagrand's concentration inequality. We use a sampling with replacement version of the symmetrization argument and a sampling with replacement version of McDiarmid's concentration inequality to obtain comparable bounds. Therefore, to prove this argument, the following two lemmas.
    \begin{lemma}(Symmetrization without Replacement)
        \begin{align*}
            \expectation &\sup_{\functionvar \in F} \left| \expectation\function_j(\functionvar,\randomvariable_j) - \frac{1}{n}\sum  \function_j(\functionvar,\randomvariable_j)\right| \leq \frac{2}{n} \expectation \sup_{\functionvar \in F} \left| \sum \radamacher_j \function_j(\functionvar,\randomvariable_j)\right| 
        \end{align*}
            where are i.i.d. Rademacher random variables and $\randomvariable_j$ are random variables sampled uniformly without replacement from some set. 
    \end{lemma}
    \begin{proof}
        Let $\randomvariable_1,\dots,\randomvariable_n$ be the random variables sampled uniformly without replacement from set $\armsubset$. 
        $\randomvariable_1-$
        Consider $\randomvariable_1^{'},\dots,\randomvariable_n^{'}$ be an independent sequence of random variables sampled uniformly without replacement from set $\armsubset$. 
        Then $\frac{1}{n}\sum  \function_j(\functionvar,\randomvariable_j)- \expectation\function_j(\functionvar,\randomvariable_j)$ and  $\frac{1}{n}\sum  \function_j(\functionvar,\randomvariable_j^{'})- \expectation\function_j(\functionvar,\randomvariable_j)$ are zero mean random variable.
        Then,
        \begin{align*}
            &\expectation || \frac{1}{n}\sum  \function_j(\functionvar,\randomvariable_j)- \expectation\function_j(\functionvar,\randomvariable_j) || \\&\leq \expectation ||\frac{1}{n}\sum  \function_j(\functionvar,\randomvariable_j)- \expectation\function_j(\functionvar,\randomvariable_j)-\frac{1}{n}\sum  \function_j(\functionvar,\randomvariable_j^{'})+ \expectation\function_j(\functionvar,\randomvariable_j)|| \\
           &\implies  \expectation || \frac{1}{n}\sum  \function_j(\functionvar,\randomvariable_j)- \expectation\function_j(\functionvar,\randomvariable_j) || \leq \expectation ||\frac{1}{n}\sum  \left(\function_j(\functionvar,\randomvariable_j)- \function_j(\functionvar,\randomvariable_j^{'})\right)|| 
        \end{align*}
        (Since $\frac{1}{n}\sum  \function_j(\functionvar,\randomvariable_j)- \expectation\function_j(\functionvar,\randomvariable_j)$ and $\frac{1}{n}\sum  \function_j(\functionvar,\randomvariable^{'}_j)- \expectation\function_j(\functionvar,\randomvariable_j)$ are independent)
        \begin{align*}
            &\expectation || \frac{1}{n}\sum  \function_j(\functionvar,\randomvariable_j)- \expectation\function_j(\functionvar,\randomvariable_j) || \leq \expectation ||\frac{1}{n}\sum  \left(\function_j(\functionvar,\randomvariable_j)- \function_j(\functionvar,\randomvariable_j^{'})\right)||  \\ &\implies \expectation ||\frac{1}{n}\sum  \radamacher_j \left(\function_j(\functionvar,\randomvariable_j)- \function_j(\functionvar,\randomvariable_j^{'})\right)||  \leq  2  \expectation ||\frac{1}{n}\sum  \radamacher_j \left(\function_j(\functionvar,\randomvariable_j)\right)||
        \end{align*}
        (Symmetric random variables~\citep{Vershynin_2018} since $\randomvariable_j$ and $\randomvariable_j'$ have same distribution; followed by Triangular Inequality).
    \end{proof}
    \begin{lemma}(Concentration using McDiarmid's inequality)
    If $|\function_j(\functionvar)| \leq \constant_{13}$ a.s.. And suppose $\randomvariablealt = \sup_{\functionvar\in F} \sum_{j=1}^{\numsamples}\function_j(\functionvar,\randomvariable_j)$, where $\randomvariable_1,\dots,\randomvariable_\numsamples$ are sampled uniformly without replacement from some set.  If $\expectation \randomvariablealt \leq 2\delta\numsamples$, then the following holds,  \begin{align*}
        \probability(\randomvariablealt \geq 4\delta\numsamples) \leq \exp\left(\frac{-8\delta^2\numsamples}{\constant_{13}^2}\right) 
    \end{align*} 
    \end{lemma}
    \begin{proof}
        We prove the result by using McDiarmid's inequality. 
        First we bound the quantity,
        \begin{align*}
           \sup_{\randomvariable^{'}_{i}} &| \sup_{\functionvar\in F} \sum_{j=1}^{\numsamples}\function_j(\functionvar,\randomvariable_j) - \sup_{\functionvar\in F} \left(\sum_{j=1,j\neq i}^{\numsamples}\function_j(\functionvar,\randomvariable_j)+\function_i(\functionvar,\randomvariable_i^{'})\right) | \\
           \leq & \sup_{\randomvariable^{'}_{i}} | \sup_{\functionvar\in F} \sum_{j=1,j\neq i}^{\numsamples}\function_j(\functionvar,\randomvariable_j) + \sup_{\functionvar\in F} \function_j(\functionvar,\randomvariable_i)  - \sup_{\functionvar\in F} \sum_{j=1,j\neq i}^{\numsamples}\function_j(\functionvar,\randomvariable_j) -\inf_{\functionvar\in F}\function_i(\functionvar,\randomvariable_i^{'}) |
           \\
           \leq & \sup_{\randomvariable^{'}_{i}} | \sup_{\functionvar\in F} \function_j(\functionvar,\randomvariable_i) | \leq \constant_{13}
        \end{align*}
        We use the following version of McDiarmid's concentration inequality for random variable without replacement with $t = 2\delta\numsamples$ to obtain the result.
        The condition that needs to be verified is that $\randomvariablealt$ is symmetric under permutations of the individual $\function_j(\functionvar,\randomvariable_j)$. This is obviously true since this is a unweighted sum of the individual $\function_j(\functionvar,\randomvariable_j)$. We next state McDiarmid's concentration inequality without replacement from~\citet{concwithoutreplacement}, 
        \begin{lemma}
        Suppose $\randomvariablealt = \sup_{\functionvar\in F} \sum_{j=1}^{\numsamples}\function_j(\functionvar,\randomvariable_j)$, where $\randomvariable_1,\dots,\randomvariable_\numsamples$ are sampled uniformly without replacement from some set. Then, 
         \begin{align*}
             \probability(\randomvariablealt- \expectation \randomvariablealt \geq t) \leq \exp\left( \frac{-2t^2}{ \numsamples\constant_{13}^2}\right).
         \end{align*}
        \end{lemma}
        The probability is $\exp\left( \frac{-8\delta^2\numsamples}{ \constant_{13}^2}\right) \leq \exp\left( \frac{-\delta^2\numsamples}{ 6\constant_{13}^2}\right) \leq \exp\left( \frac{-\delta^2\numsamples\mineigvalue^2(\sparsity_1,\covariancematrix^{1/2})}{ 6\armbound^4\sparsityaux^2}\right) \leq  \exp\left( \frac{-\delta^2\numsamples\mineigvalue^2(\sparsity_1,\covariancematrix^{1/2})}{ 6\armbound^2\sparsityaux}\right)$, which is same as that in the original theorem from~\cite{reconstructionitit} except an additional $\mineigvalue(\sparsity_1,\covariancematrix^{1/2})$ which is reflected  in our Theorem statement. 
    \end{proof}
    \newcommand{\functionvaralt}{z}
\textit{Comment on Dudley's inequality}: Theorem 23 also uses Dudley's inequality, but there the $\Psi_1,\dots,\Psi_\numsamples$ are treated as deterministic and so the proof goes through in our sampled without replacement case as well. 
\end{proof}

We can now complete the proof by computing the probability of the following event,
\begin{align*}
    \event &= \left\{ \restrictedeigenvalue (\sparsityk,4+4\restrictedconeconstant,\frac{\designmatrix}{\sqrt{\numsamples}} )  \geq  \restrictedeigenvalue (\sparsityk,4+4\restrictedconeconstant, {\covariancematrix}^{1/2}) \geq (\mineigvalue^*)^{-1/2}, \right.\\&\left.  \Vert\parameter-\parameterestimate\Vert_1 =  \widetilde{\orderof}\left( \sigma \frac{(5+4\restrictedconeconstant)^2\sparsityk}{\restrictedeigenvalue (\sparsityk,4+4\restrictedconeconstant,\frac{\designmatrix}{\sqrt{\numsamples}} )}\sqrt{\frac{2\log{\dimarm}}{\numsamples}}\right) \right\}
\end{align*}
has the probability, $\probability(\event) \geq 1-\dimarm^{1-2\constant_2^2}-\exp(-\constant_5\numsamples)$. which completes the proof. 

\subsection{Proof of Theorem~\ref{th:randomizedrounding}}

For vectors $\vectorsym$ and $\randomvectoralt$, define $\randomvariable_\vectorsym(\randomvectoralt) = \randomvectoralt^\transpose \vectorsym\vectorsym^\transpose\randomvectoralt$. 

Then the minimum eigenvalue for (covariance matrix of) a set of vectors $\armsubset$ is given by, $\mineigvalue(\armsubset)=\min_{\randomvectoralt\in\unitball}\frac{1}{|\armsubset|}\sum_{\vectorsym\in\armsubset}\randomvariable_\vectorsym(\randomvectoralt)$.   

Let randomized rounding be run with $\searchbound$, and $\hat{\armsubset}$ be the sampled set of arms. Of-course $|\hat{\armsubset}|$ need not be equal to $\searchbound$. However, we first assume that the denominator of $\mineigvalue(\hat{\armsubset})$ is equal to $\searchbound$. We later show that this assumption worsens the approximation guarantees by $2$ with high probability. Under this assumption by construction of the randomized rounding procedure, the expected minimum eigenvalue of the sampled set is equal to the minimum eigenvalue corresponding to the solution of the convex optimization problem, since, $\expectation_{\samplingdist}\randomvariable_\vectorsym(\randomvectoralt) = \samplingdist_\vectorsym\randomvariable_\vectorsym(\randomvectoralt) $. 

\textbf{Step 1: }We therefore prove the following result to bound the approximation error between the $\mineigvalue(\hat{\armsubset})$ obtained from the randomized rounding solution and the optimal solution (off by a factor of $\frac{\searchbound}{|\hat{\armsubset}|}$)of the convex relaxation from~\eqref{opt:relaxed} run using $\searchbound$.  

\begin{lemma}\label{lemma:probboundeig}
Let $\armset$ be a set of $\numarms$ arms where each arm is $\arm \in \unitball$ and let $\randomvariable_\vectorsym(\randomvectoralt) = \randomvectoralt^\transpose \vectorsym\vectorsym^\transpose\randomvectoralt$. Let $\samplingdist$ be the solution of the convex relaxation of~\eqref{opt:relaxed} at $\searchbound$ and $\hat{\armsubset}$ be the set sampled using randomized rounding (Step 18-20 in Alg.~\ref{alg:greedy}). Then for some constant $\constant_{10}$ the following holds , 
  \begin{align}\label{eq:approximationerrorbasic}
  \begin{split}
        \probability&\left(\left| \inf_{\vectorsymalt \in \unitball}\frac{1}{\searchbound}\sum_{\vectorsym\in\hat{\armsubset}} \randomvariable_\vectorsym(\randomvectoralt) - \inf_{\vectorsymalt \in \unitball}\sum_{\vectorsym\in\armset}\samplingdist_\vectorsym\randomvariable_\vectorsym(\randomvectoralt)\right| \geq \frac{\constant_{10}\sqrt{\dimarm}\log\numarms}{\sqrt{|\searchbound|}}\right) \leq \frac{1}{\log \numarms}
        \end{split}
    \end{align}
\end{lemma}

\begin{proof}

    By symmetrization over the sum of independent random variables $\{ \randomvariablealt_i,  i \in [\numarms] \}$ each of which is  $\randomvariable_{\vectorsym_i}(\randomvectoralt)$ with probability $\searchbound\samplingdist_{\vectorsym_i}$ and 0 otherwise. It can be seen that  $\expectation\{\frac{1}{\searchbound}\sum_{i \in [\numarms]} \randomvariablealt_i\} = \sum_{\vectorsym\in\armset} \samplingdist_\vectorsym \randomvariable_\vectorsym(\randomvectoralt)$
    \begin{align}~\label{eq:bound on expectation of maximum discrepancy between sampled and exact } 
    \begin{split}
         \expectation &\left[ \sup_{\vectorsymalt \in \unitball}\left| \frac{1}{\searchbound}\sum_{i \in [\numarms]} \randomvariablealt_i -  \sum_{\vectorsym\in\armset} \samplingdist_\vectorsym\randomvariable_\vectorsym(\randomvectoralt) \right| \right]  \leq 2  \expectation \left[ \sup_{\vectorsymalt \in \unitball}\left| 
         \frac{1}{\searchbound}\sum_{i \in[\numarms]} \radamacher_i \randomvariablealt_i \right| \right] = \frac{2}{\searchbound} \expectation \left[ \sup_{\vectorsymalt \in \unitball}\left| 
  \sum_{i \in[\numarms]} \radamacher_i \randomvariablealt_i \right| \right]
  \end{split}
    \end{align}
    where $\radamacher_i$ are i.i.d. Rademacher random variables.
    
Now using Dudley's integral inequality on sum of independent RVs, 
    \begin{align*}
      \expectation \left[ \sup_{\vectorsymalt \in \unitball}\left| \sum_i \radamacher_i \randomvariablealt_i \right| \right] \leq \constant_{11} \subgaussianbound \log^{1/2}(\frac{3}{\epsilon})\sqrt{\dimarm}
    \end{align*}
where $\subgaussianbound$ is the constant which satisfies for all $\randomvectoralt_1$ and $\randomvectoralt_2$, 
\begin{align*}
   ||\sum_{i\in[\numarms]} &\radamacher_i \randomvariablealt_i(\randomvectoralt_1)-\sum_{i\in[\numarms]} \radamacher_i \randomvariablealt_i(\randomvectoralt_2)||_{\subgaussiansym}  \leq \subgaussianbound ||\vectorsymalt_1-\vectorsymalt_2||_2 \leq \sqrt{2}\subgaussianbound,
\end{align*}
 and where $\Vert \cdot \Vert$ is the sub Gaussian norm.

Now w.l.o.g.,
\begin{align*}
      ||\sum_{i\in[\numarms]} \radamacher_i \randomvariablealt_i(\randomvectoralt_1)-\sum_{i\in[\numarms]} \radamacher_i \randomvariablealt_i(\randomvectoralt_2)||_{\subgaussiansym}  &\leq  2||\sum_{i\in[\numarms]} \radamacher_i \randomvariablealt_i(\randomvectoralt_1)||_{\subgaussiansym}  \leq 2||\sum_i \randomvariablealt_{i}(\randomvectoralt_1)||_{\subgaussiansym}.
\end{align*}
The last inequality follows since the $\radamacher_i\randomvariablealt_{i}$ are bounded by $\randomvariablealt_{i}$. 
Now from the definition of sub Gaussian norm, 
\begin{align*}
    ||\sum_{i\in[\numarms]}  \randomvariablealt_{i}||_{\subgaussiansym} = \inf\{t: \exp{\frac{(\sum  \randomvariablealt_{i})^2}{t^2}}\leq 2\}
\end{align*}
Now, 
\begin{align*}
 \exp{\frac{\sum (W_i)^2}{t^2}} \leq \exp{\frac{4\armbound^4\dimarm}{\searchbound t^2}}  \ {(a)}
\end{align*}
(a) follows with high probability from the following argument (using Hoeffding's inequality on Bernoulli random variable $\indicator(\vectorsym)$ (if $\vectorsym$ is present or not with probability $\searchbound\samplingdist_\vectorsym$) with deviation equal to the mean), 
\begin{align}\label{eq:eqn with hoeffdings}
        &\sum_{\vectorsym\in\hat{\armsubset}} (\randomvariablealt_i)^2 \leq  \sum_{\vectorsym\in\armset} 2\expectation[\indicator(\vectorsym)] (\randomvariable_\vectorsym(\randomvectoralt_1))^2  &\leq  \sum_{\vectorsym\in\armset} 2\searchbound\samplingdist_\vectorsym \frac{(\randomvectoralt_1^\transpose\vectorsym \vectorsym^\transpose\randomvectoralt_1)^2}{\searchbound^2}  \leq \frac{2\armbound^4\dimarm^2}{\searchbound}\sum \samplingdist_\vectorsym \leq \frac{2\armbound^4\dimarm^2}{\searchbound},
\end{align}
with probability $1-\exp(-2\numarms)$. 

To find $\inf\{t:\exp(\frac{\cdot}{t^2})\leq 2\}$,
\begin{align*}
    &\exp{\frac{4\armbound^4\dimarm^2}{\searchbound t^2}} \leq 2 \implies  \sqrt{\frac{4\armbound^4\dimarm^2}{\searchbound\ln2}} \leq t
\end{align*}
Therefore 
$\subgaussianbound \leq \sqrt{\frac{4\armbound^4\dimarm^2}{\searchbound\ln2}}$ and letting $\constant_{10}  = 4\constant_{11}\sqrt{\frac{\log(\frac{3}{e})}{\ln (2)}}$, plugging this into~\eqref{eq:bound on expectation of maximum discrepancy between sampled and exact } we obtain, 
\begin{align*}
  \expectation &\left[ \sup_{\vectorsymalt \in \unitball}\left| \frac{1}{\searchbound}\sum_{i \in [\numarms]} \randomvariablealt_i -  \sum_{\vectorsym\in\armset} \samplingdist_\vectorsym\randomvariable_\vectorsym(\randomvectoralt) \right| \right] \leq  2\constant_{10} \sqrt{\frac{\dimarm^3\armbound^4}{\searchbound^{3}}}
\end{align*}

Using Markov's inequality along with the union bound on the the probability of \eqref{eq:eqn with hoeffdings}, 
  \begin{align*}
        \probability(\sup_{\vectorsymalt \in \unitball}\left| \frac{1}{\searchbound}\sum_{i \in [\numarms]} \randomvariablealt_i -  \sum_{\vectorsym\in\armset} \samplingdist_\vectorsym\randomvariable_\vectorsym(\randomvectoralt) \right|&\geq 2\constant_{10}\frac{\sqrt{\dimarm}t}{\sqrt{\searchbound^3}}) \leq \frac{1}{t} + \exp(-2\numarms)
    \end{align*}
    One can can set $t$ appropriately.

Also finally note that showing the bound on the supremum implies the bound we set to show in equation (11). Note that $\frac{1}{\searchbound}\sum_{i \in [\numarms]} \randomvariablealt_i = \frac{1}{\searchbound}\sum_{\vectorsym\in\hat{\armsubset}} \randomvariable_{\vectorsym}$
 Now if, 
    \begin{align*}
        \sup_{\vectorsymalt \in \unitball}\left| \frac{1}{\searchbound}\sum_{\vectorsym\in\hat{\armsubset}} \randomvariable_{\vectorsym} - \sum_{\vectorsym\in\armset} \samplingdist_\vectorsym\randomvariable_\vectorsym(\randomvectoralt) \right| \leq \epsilon
    \end{align*}
    is true with probability $1-\delta$. 
    Then $\forall \vectorsymalt \in \unitball$, 
    \begin{align*}
   \sum_{\vectorsym\in\armset} \samplingdist_\vectorsym\randomvariable_\vectorsym(\randomvectoralt)  - \epsilon &\leq   \sum_{\vectorsym\in\hat{\armsubset}}\frac{1}{\searchbound}\randomvariable_\vectorsym(\randomvectoralt)\\
   \inf_{\vectorsymalt \in \unitball}\expectation  \sum_{\vectorsym\in\armset} \samplingdist_\vectorsym\randomvariable_\vectorsym(\randomvectoralt)  - \epsilon &\leq   \inf_{\vectorsymalt \in \unitball}\frac{1}{\searchbound}\sum_{\vectorsym\in\hat{\armsubset}}\randomvariable_\vectorsym(\randomvectoralt)\\
    \end{align*}
    with probability $1-\delta$. Similar to the other direction, we obtain the desired result. 
    \end{proof}
\textbf{Step 2: }Then, we bound the size of the actual sampled set with respect to $\searchbound$, at which the convex relaxation is computed. We derive the following result which gives a probability bound on the number of arms sampled. 
\begin{lemma}\label{lemma:probboundsize}
Let $\searchbound$ be the subset size that the randomized rounding is run with (Line 20 in Alg.~\ref{alg:greedy}) and let $\hat{\armsubset}$ be the true number of sampled arms. Then the following probability holds, 
    \begin{align}
        \probability(\frac{\searchbound}{2}\leq |\hat{\armsubset}|\leq 2\searchbound) \geq 1 -  2(\frac{2}{e})^{\frac{\searchbound}{2}}
    \end{align}
\end{lemma}

\begin{proof}
        We prove the following two tail bounds and then take the union bound over them both, 
    \begin{align*}
         \probability(|\hat{\armsubset}|\geq 2\searchbound) \leq (\frac{e}{3})^{\searchbound}, \probability(|\hat{\armsubset}|\leq \searchbound/2) \leq (\frac{2}{e})^{\frac{\searchbound}{2}}
    \end{align*}
First the size of the sampled subset is the sum of independent Bernoulli random variables, $|\hat{\armsubset}| = \sum \indicatorelement_j $ where each $\indicatorelement_j = Ber(\samplingdist_j\searchbound)$. Using tail bound from Chernoff bound,
\begin{align*}
    \probability(|\hat{\armsubset}|\geq 2\searchbound) &\leq \inf_{t>0}\exp(-t2\searchbound)\expectation[\exp(t|\hat{\armsubset}|)] = \inf_{t}\exp(-t2\searchbound)\prod_{j} \expectation[\exp(t\indicatorelement_j)] \  \text{(independent rv})\\
   \inf_{t}&\exp(-t2\searchbound)\prod_{j} \expectation[\exp(t\indicatorelement_j)] \leq \inf_{t}\exp(-t2\searchbound)\prod_{j} \exp(\searchbound\samplingdist_j(\exp(t)-1)) \\&= \inf_{t}\exp(-t2\searchbound)\exp(\searchbound(\exp(t)-1))
\end{align*}
Achieves infinum for $t = \ln2$, 
\begin{align*}
    \probability(|\hat{\armsubset}|\geq 2\searchbound) \leq \exp(-2\searchbound\ln2  + \searchbound) = (\frac{3}{e})^{\searchbound}
\end{align*}
Using a similar left tail bound, 
\begin{align*}
    \probability(|\hat{\armsubset}|\leq \searchbound/2) &\leq \inf_{t<0}\exp(-t\searchbound/2)\expectation[\exp(t|\hat{\armsubset}|)] \\&= \inf_{t}\exp(-t\searchbound/2)\prod_{j} \expectation[\exp(t\indicatorelement_j)] \ \text{(independent rv})\\&
  = \inf_{t}\exp(-t\searchbound/2)\prod_{j} \expectation[\exp(t\indicatorelement_j)] \\&\leq \inf_{t}\exp(-t\searchbound/2)\prod_{j} \exp(\searchbound\samplingdist_j(\exp(t)-1)) = \inf_{t}\exp(-t\searchbound/2)\exp(\searchbound(\exp(t)-1))
\end{align*}
Achieves infinum for $t = -\ln2$,
\begin{align*}
    \probability(|\hat{\armsubset}|\leq \searchbound/2) \leq \exp(\searchbound(-\frac{1}{2}) + \ln2\searchbound/2) = (\frac{2}{e})^{\frac{\searchbound}{2}} 
\end{align*}
Now for $\searchbound \geq 1$, $(\frac{2}{e})^{\frac{\searchbound}{2}} \geq  (\frac{3}{e})^{\searchbound}$, and therefore applying the union bound we obtain the required result. 
\end{proof}
Therefore the above lemma helps us prove the following statement, 
\begin{align*}
    \probability(\frac{1}{2} \leq \frac{\min_{\vectorsymalt\in\unitball}\sum_{\vectorsym} \frac{1}{|\hat{\armsubset}|}\indicatorprob_{\vectorsym}\vectorsymalt^\transpose \vectorsym\vectorsym^\transpose \vectorsymalt}{\min_{\vectorsymalt\in\unitball}\sum_{\vectorsym\in\armset} \frac{1}{\searchbound}\indicatorprob_{\vectorsym}\vectorsymalt^\transpose \vectorsym\vectorsym^\transpose \vectorsymalt} \leq 2)\geq 1 -  2(\frac{2}{e})^{\frac{\searchbound}{2}}
\end{align*}

\textbf{Step 3: }The above two lemmas help us prove the approximation error of the randomized rounding with respect to a fixed parameter $\searchbound$. However, the approximation errors need to be with respect to the optimal choice of $\searchbound$, $\optsubsetsize$, which is the size of the optimal subset from~\eqref{eq:discreteopt}.

We claim the following about the solution of the convex relaxation at $\searchbound$ and $\constant_{12}\dimarm$,
\begin{align}\label{eq:mineigvaluetrick}
\begin{split}
   &\mineigvalue^*(\constant_{12}\dimarm) \leq  \underset{{\samplingdist \in \probabilityspace(\armset);\Vert\samplingdist\Vert_{\infty} \leq \frac{1}{\constant_{12}\dimarm}}}{\arg\max} \ \inf_{\vectorsymalt \in \unitball}\sum_{\vectorsym\in\armset}\samplingdist_\vectorsym\randomvariable_\vectorsym(\randomvectoralt)     \leq\underset{{\samplingdist \in \probabilityspace(\armset);\Vert\samplingdist\Vert_{\infty}\leq \frac{1}{\constant_{12}\dimarm}}}{\arg\max} \ \inf_{\vectorsymalt \in \unitball}\sum_{\vectorsym\in\armset}\frac{1}{\constant_{12}\dimarm}\randomvariable_\vectorsym(\randomvectoralt)\\& \leq 
   \frac{\searchbound}{\constant_{12}\dimarm}\underset{{\samplingdist \in \probabilityspace(\armset);\Vert\samplingdist\Vert_{\infty} \leq \frac{1}{\dimarm}}}{\arg\max} \ \inf_{\vectorsymalt \in \unitball}\sum_{\vectorsym\in\armset}\frac{1}{\searchbound}\randomvariable_\vectorsym(\randomvectoralt)   \leq \frac{\searchbound}{\constant_{12}\dimarm}\underset{{\samplingdist \in \probabilityspace(\armset);\Vert\samplingdist\Vert_{\infty} \leq \frac{1}{\searchbound}}}{\arg\max} \ \inf_{\vectorsymalt \in \unitball}\sum_{\vectorsym\in\armset}\frac{1}{\searchbound}\randomvariable_\vectorsym(\randomvectoralt)\\& \leq \frac{\searchbound}{\constant_{12}\dimarm}\underset{{\samplingdist \in \probabilityspace(\armset);\Vert\samplingdist\Vert_{\infty} \leq \frac{1}{\searchbound}}}{\arg\max} \ \inf_{\vectorsymalt \in \unitball}\sum_{\vectorsym\in\armset}\frac{1}{\searchbound}\randomvariable_\vectorsym(\randomvectoralt)  \leq \frac{\searchbound}{\constant_{12}\dimarm}\underset{{\samplingdist \in \probabilityspace(\armset);\Vert\samplingdist\Vert_{\infty} \leq \frac{1}{\searchbound}}}{\arg\max} \ \inf_{\vectorsymalt \in \unitball}\sum_{\vectorsym\in\armset}\samplingdist_\vectorsym\randomvariable_\vectorsym(\randomvectoralt)  \\&\leq \frac{\searchbound}{\constant_{12}\dimarm}\mineigvalue^*(\searchbound)
   \end{split}
\end{align}
The last inequality follows from the fact that $\frac{1}{\searchbound}$ lies is in the feasibility set. 

Let $\optsubsetsize$ be the size of the optimal subset. 
If $\optsubsetsize>\constant_{12}\dimarm$,  (otherwise we do a subset search), the convex relaxation at $\optsubsetsize$ is more constrained than the one at $\constant_{12}\dimarm$,
\begin{align*}
& \mineigvalue^*(\optsubsetsize) \leq  \mineigvalue^*(\constant_{12}\dimarm) \leq  \frac{\searchbound}{\constant_{12}\dimarm}\mineigvalue^*(\searchbound) \implies \mineigvalue^*(\searchbound) \geq  \frac{\constant_{12}\dimarm}{\searchbound}\mineigvalue^*(\optsubsetsize) 
\end{align*}

\textbf{Putting Everything Together: }We can now combine the two lemmas and the equation above to say that for an error $\epsilon  = \constant_{10}\frac{t\sqrt{\dimarm}}{\sqrt{\searchbound}^3}$, to obtain the following, 
\begin{align*}
   \hat{\lambda}_{\min} \geq \frac{\constant_{12}\dimarm}{2\searchbound}\mineigvalue^*(\optsubsetsize)  - 2\constant_{10} \sqrt{\frac{\dimarm^3\armbound^4}{\searchbound^{3}}} 
\end{align*}
Set $\constant_{12}$ to be a large enough constant, $\searchbound = \constant_{12} \dimarm$ and $t = 2$, 
\begin{align*}
      \hat{\lambda}_{\min} \geq \frac{\mineigvalue^*(\optsubsetsize) }{2} - 2\constant_{10} \sqrt{\frac{1}{\constant_{12}^3}} 
\end{align*}
with probability $\frac{1}{2}-\exp(-2\numarms) - 2(\frac{2}{e})^{\frac{\searchbound}{2}}$.

If a lower bound $\mineigvalue^l$ of $\mineigvalue^*(\optsubsetsize)$ is known, we set $\constant_{12}$ to be large enough such that $2\constant_{10} \sqrt{\frac{1}{\constant_{12}^3}} 
 \leq \frac{\mineigvalue^l}{4}$, making the inequality,
\begin{align*}
 \hat{\lambda}_{\min} &\geq \frac{1}{2}\mineigvalue^*(\optsubsetsize) - \frac{\mineigvalue^l}{4} 
  \implies   \hat{\lambda}_{\min} \geq \frac{1}{4}\mineigvalue^*(\optsubsetsize)
\end{align*}
Therefore we need $\constant_{12}\geq \frac{{64}^{2/3}(\constant_{10})^{2/3}}{(\lambda_{\min}^l)^{2/3}}$

To increase the probability of success , we can then run this multiple times and then take the maximum of the minimum eigenvalues of the different sampled subsets, as is standard when using randomized algorithms. 

\subsection{Proof of Theorem~\ref{th:regret}}\label{app:th_regret}
\begin{proof}
    The regret bound can be proved in 3 steps. First, we decompose the regret, apply Corollary~\ref{coro:empiricallasso}, and then optimize the exploration period.

  \textbf{Step 1. Regret Decomposition: }
  $\armbound = \sup_{\arm\in\armset}\|\armset\|_\infty$. 
Define the maximum reward as, 
$\rmax = \max_{\arm\in\armset}|\parameter^\transpose \arm|$ and $\arm^*$ as the corresponding arm. The regret can be decomposed as, 
\begin{align*}
        \expectedregret = &\expectation_{\parameter}\left[ \sum_{\timeindex=1}^{\numrounds}\left<\parameter,\arm^{\permutation(\timeindex)} - \arm_\timeindex\right>\right] = \expectation_{\parameter}\left[ \sum_{\timeindex=1}^{\numrounds}\left<\parameter,\arm^{\permutation(\timeindex)} - \arm_\timeindex\right>\right] \\&= \expectation_{\parameter}\left[ \sum_{\timeindex=1}^{\numroundsexplore}\left<\parameter,\arm^{\permutation(\timeindex)} - \arm_\timeindex\right>\right] + \expectation_{\parameter}\left[ \sum_{\timeindex=\numroundsexplore+1}^{\numrounds}\left<\parameter,\arm^{\permutation(\timeindex)} - \arm_\timeindex\right>\right]
        \\&\leq 2\numroundsexplore\rmax + \expectation_{\parameter}\left[ \sum_{\timeindex=\numroundsexplore+1}^{\numrounds}\left<\parameter,\arm^{\permutation(\timeindex)} - \arm_\timeindex\right>\right]
\end{align*}
The decomposition step requires extra care since the regret is with respect to the top-$\numroundsexplore$ arms.
 In the exploitation stage, the arms are selected such that the top $(\numrounds-\numroundsexplore)-$arms are played according to $\parameterestimate$ and are indexed by the permutation $\permutationapprox$ (that is $\arm_\timeindex = \arm^{\permutationapprox(\armidx)}$). We next bound the regret for the $\armidx^{\text{th}}$ selected arm. 
    
    \begin{enumerate}
        \item If $\parameter^{\transpose}\arm^{(\permutationapprox(\armidx))} \geq \parameter^{\transpose}\arm^{(\permutation(\armidx))}$, then the regret for the $\armidx^{\text{th}}$ selected arm is negative. 
        \item If not i.e., $\parameter^{\transpose}\arm^{(\permutationapprox(\armidx))} \leq \parameter^{\transpose}\arm^{(\permutation(\armidx))}$, then there exists an arm index, $\armidx_1$ in the permutation $\permutation$ such that $\armidx_1$ is shifted to the left in $\permutationapprox$. This implies that $\parameter^{\transpose}\arm^{(\permutationapprox(\armidx_1))} \geq \parameter^{\transpose}\arm^{(\permutation(\armidx_1))}$. Note that by our estimation guarantees of estimating $\parameter$ upto $\epsilon$ accuracy with high probability, $(\parameterestimate^{\transpose}\arm^{(\permutationapprox(\armidx))}
-\parameter^{\transpose}\arm^{(\permutationapprox(\armidx))} ) \leq \|\parameterestimate^{\transpose} -  \parameter^{\transpose}\|_1\|\arm^{(\permutationapprox(\armidx))}\|_\infty\leq \armbound\epsilon$  with probability $1-\frac{1}{\dimarm^{\constant_2-1}}$ ($\constant_2$ is a large enough constant).  We decompose the regret for this case with respect to this index and bound the error: 
    \end{enumerate}
    \begin{align*}
\parameter^{\transpose}\arm^{(\permutation(\armidx))}-\parameter^{\numrounds}\arm^{(\permutationapprox(\armidx))} &= 
\underbrace{(\parameter^{\transpose}\arm^{(\permutation(\armidx))}-
\parameter^{\transpose}\arm^{(\permutationapprox(\armidx_1))})}_{\leq 0}
 +\underbrace{(\parameter^{\transpose}\arm^{(\permutationapprox(\armidx_1))}
-\parameterestimate^{\transpose}\arm^{(\permutationapprox(\armidx_1))})}_{\leq \armbound\epsilon}
\\&
+\underbrace{(\parameterestimate^{\transpose}\arm^{(\permutationapprox(\armidx_1))}
-(\parameterestimate^{\transpose}\arm^{(\permutationapprox(\armidx))})}_{\leq  0}
 +\underbrace{(\parameterestimate^{\transpose}\arm^{(\permutationapprox(\armidx))}
-\parameter^{\numrounds}\arm^{(\permutationapprox(\armidx))} )}_{\leq \armbound\epsilon}
\\&\leq \left<\parameter^{\numrounds} 
-
\parameterestimate^{\numrounds}, \arm^{(\permutation(\armidx))} \right> + \left<\parameter^{\numrounds} 
-
\parameterestimate^{\numrounds}, \arm^{(\permutation(\armidx_1))} \right>
\leq 2\armbound\epsilon
    \end{align*}
with probability $\probabilityexploresucc= 1-\frac{1}{\dimarm^{1-\constant_2}}-\exp(-\constant_5\numsamples)$. 

We therefore obtain the following, 
     \begin{align*}
       \expectedregret \leq &2\armbound\numroundsexplore\rmax + 2(\numrounds-\numroundsexplore)\probabilityexploresucc\epsilon + \armbound(\numrounds-\numroundsexplore)(1-\probabilityexploresucc)\rmax,
   \end{align*}
   For a high enough $\constant_3$, $1-\probabilityexploresucc = o(1)$ and the last term is very small. 
   Further $\numrounds \gg \numroundsexplore$ (the number of exploration rounds is sublinear in $\numrounds$) we obtain,

  \begin{align*}
    \expectedregret = &\expectation_{\parameter}\left[ \sum_{\timeindex=1}^{\numrounds}\left<\parameter,\arm^{\permutation(\timeindex)} - \arm_\timeindex\right>\right]   = \Tilde{\orderof}\left(\rmax\numroundsexplore + \armbound\|\parameter-\parameterestimate\|_1\numrounds \right)
\end{align*}

\textbf{Step 2. Fast Sparse Learning: }
We use Theorem \ref{lemma:fastsparselearning}, which is proved in the appendix, to obtain an estimation guarantee in terms of the number of exploration rounds. 
And we now apply the bound from Corollary~\ref{coro:empiricallasso} and obtain the following (Making an assumption similar to~\citet{hao2020high} on the exploration rounds $ \numroundsexplore>\orderof(\sparsityk\mineigvalue^{-4}$).

\textbf{Step 3. Exploration Period Optimization: }
 ( The probability of error terms ($1-\probabilityexploresucc$) are left out in the expression.)
 $\epsilon= \widetilde{\orderof}\left( \sigma \frac{(5+4\restrictedconeconstant)^2\sparsityk}{\mineigvalue^*}\sqrt{\frac{1}{\numroundsexplore}}\right)$
~\eqref{eq:regretbound} can then be bounded as,
    \begin{align*}
    \expectedregret = \widetilde{\orderof}\left(\rmax\numroundsexplore + \numrounds (\sigma\armbound \frac{(5+4\restrictedconeconstant)^2\sparsityk}{\mineigvalue^*}\sqrt{\frac{1}{\numroundsexplore}})\right)
    \end{align*}
Setting $\numroundsexplore =  \rmax^{-2/3}\hat{\mineigvalue}^{-2/3}\sparsityk^{2/3}\numrounds^{2/3}$ we obtain the following,
   \begin{align*}
    \expectedregret = \widetilde{\orderof}\left(\rmax^{1/3}\left(\frac{1}{\hat{\mineigvalue}^{-2/3}} + \sigma\armbound \frac{(5+4\restrictedconeconstant)^2\hat{\mineigvalue}^{1/3}}{\mineigvalue^*}\right)\sparsityk^{2/3}\numrounds^{2/3}\right)
    \end{align*}

Further using guarantees from Theorem~\ref{th:randomizedrounding}, we use $\hat{\mineigvalue} = \Omega(\mineigvalue^*)$ with high probability and $\armbound=1$,
   \begin{align*}
    \expectedregret = \widetilde{\orderof}\left(\rmax^{1/3}\sigma{(5+4\restrictedconeconstant)^2{\mineigvalue^*}^{-2/3}}\sparsityk^{2/3}\numrounds^{2/3}\right).
    \end{align*}

   \end{proof}
We can also obtain a regret bound for the case of hard sparsity which is of the same order as~\citet{hao2020high}. 
\begin{coro}\label{coro:hard}
   Let $\parameter$ be $\sparsityk$-sparse, $\Vert\parameter\Vert_0 \leq \sparsityk$ in the sparse linear bandits framework of Theorem~\ref{th:regret}. Let $\mineigvalue^*$ be the minimum eigenvalue from~\eqref{eq:discreteopt} with the same assumptions as Theorem~\ref{th:regret}. Then Algorithm \texttt{BSLB} with
exploration period $\numroundsexplore = \orderof(\sparsityk^{\frac{2}{3}}\numrounds^{\frac{2}{3}})$, achieves a regret guarantee of   
$\expectedregret = \orderof((\mineigvalue^* )^{-1}\sparsityk^{\frac{2}{3}} \numrounds^{\frac{2}{3}}).$
\end{coro}

\subsection{Alternative to Subset Algorithm}\label{sec:appendix on avoiding brute force}
We can reduce the time complexity of Algorithm~\ref{alg:getgoodsubset} by considering a bound on the eigenvalue of the sampled vectors in terms of a lower bound. The tradeoff here is in the regret gaurantee - the denominator gets an additional factor of $(\mineigvalue^l)^2$. First note the following decomposition, 
\begin{align*}
 &\frac{\mineigvalue(\armsubset^*)}{\constant_{12}}\leq \frac{\dimarm\mineigvalue(\armsubset^*)}{\constant_{12}\dimarm} \leq \frac{\optsubsetsize\mineigvalue(\armsubset^*)+(\searchbound-\optsubsetsize)\mineigvalue(\mathcal{J})}{\searchbound}  
 \\&\leq  \mineigvalue(\armsubset^* \cup \mathcal{J}) \leq \mineigvalue^*(\searchbound) \overset{(a)}{\leq} 2(\hat{\mineigvalue}+2\constant_{10} \sqrt{\frac{\dimarm^3}{\searchbound^{3}}} ) \\
 &\implies \frac{\mineigvalue(\armsubset^*)}{2\constant_{12}} -\constant_{10} \sqrt{\frac{1}{\constant_{12}^{3}}} \leq \hat{\mineigvalue}
\end{align*}
where (a) follows from the approximation gaurantees of the previous section. 

Now plugging in $\constant_{12} = \frac{16\constant_{10}^2}{(\mineigvalue^l)^2}$ we obtain, 
\begin{align*}
     \frac{\mineigvalue^*(\optsubsetsize)(\mineigvalue^l)^2}{32\constant_{10}^2} &-\frac{(\mineigvalue^l)^3}{64\constant_{10}^{2}} \leq \hat{\mineigvalue} \\ 
     &\implies  \hat{\mineigvalue} \geq \frac{(\mineigvalue^l)^2(2\mineigvalue^*(\optsubsetsize)-\mineigvalue^l)}{64\constant_{10}^2} \geq \frac{(\mineigvalue^l)^2(\mineigvalue^*(\optsubsetsize))}{64\constant_{10}^2} 
\end{align*}
Therefore this choice of $\constant_{12}$ (and hence $\hat{u}$, leads to a linear time algorithm however the regret bound is now $(\mineigvalue^*)^{-1}(\mineigvalue^l)^{-2}$ instead of $(\mineigvalue^*)^{-1}$. 
\subsection{Subset Search Algorithm for Searching the Optimal Subset}\label{sec:bruteforce}

From the previous previous proof we can set $\searchbound = \constant_{12}\dimarm$, and then search for subsets in the range $[\dimarm,\searchbound]$ to obtain a minimum eigenvalue $\hat{\mineigvalue}$. We obtain the approximation guarantee, $\hat{\mineigvalue}\geq\frac{1}{2}\mineigvalue^*$ w.h.p., since we are only using the approximation guarantee from Lemma~\ref{lemma:probboundeig} and Lemma~\ref{lemma:probboundsize}, and not from~\eqref{eq:mineigvaluetrick} because we are already searching the space $<\searchbound$. Since the search space is dependent on $\textsf{Poly}(\dimarm)$, the time complexity of the subset search algorithm, Algorithm \ref{alg:bruteforce} follows. This time complexity is substantially smaller than the complexity over the search over all arms which is of the order $\orderof(\exp(\numarms))$.

\begin{algorithm}
    \begin{algorithmic}
    \STATE \textit{Input} Approximation Factor $\epsilon$, Search Bound $\searchbound$
    \STATE \textit{Output} Subset $\Bar{\armsubset}$
        \STATE Set $\Bar{\mineigvalue} = 0$, $\Bar{\armsubset} = \phi$
        \FOR{$\Bar{\dimarm}$ in $\{\dimarm,\dots,\searchbound\}$}
            \FOR{$\armsubset'$ in  $\{\armsubset \subseteq \armset; |\armset| = \Bar{\dimarm} \}$ }
            \IF{$\mineigvalue(\sum_{\arm \in \armsubset'}|\armsubset'^{-1}|\arm\arm^\transpose)>\bar{\mineigvalue}$}
            \STATE Set $\bar{\mineigvalue} = \mineigvalue(\sum_{\arm \in \armsubset'}|\armsubset'^{-1}|\arm\arm^\transpose)$, $\Bar{\armsubset} = \armsubset'$ 
            \ENDIF{}
            \ENDFOR{}
        \ENDFOR{}
    \end{algorithmic}
    \caption{SubsetSearch : Subset Search for Optimal Subset}
    \label{alg:bruteforce}
\end{algorithm}
\textit{What if the maximum minimum eigenvalue $\mineigvalue^*$ is not known ?}
We can use a lower bound on the $\mineigvalue^*$. This is easy to obtain: Randomly sample subsets of the arm set $\ca{A}$ and compute the objective value in Equation \ref{eq:discreteopt} for each subset - the lower bound can be the maximum objective value across the sampled subsets.

\textit{What can the practitioner do to select a good subset size empirically?}
Additionally, if a practitioner wants to test out a particular choice of $\searchbound$, the worst-case error can be empirically calculated (the difference between the convex relaxation at $\dimarm$ and averaged across multiple randomized rounding runs for different values of $\searchbound$). This is possible because the randomized rounding in \textsc{GetGoodSubset}($\armset$) can be run offline.

\subsection{Why Does the Average Minimum Eigenvalue Constraint Not Satisfy Matroid Constraints}\label{sec:matroid}
Existing work in experiment design work with objective functions which often satisfy the matroid, submodularity or cardinality constraints to perform experiment design~\citep{Allen-Zhu2021-yl,brown2024maximizingminimumeigenvalueconstant}. However we need to optimize the minimum eigenvalue averaged across the subset (because we want to avoid dependence of $\numarms$ in the regret term and also use the RE condition). Our objective clearly does not satisfy the cardinality constraint, and the feasible sets don't satisfy a matroid constraint (since removing a vector might improve the minimum eigenvalue averaged across the set, so there is no clear partitioning/structure to the feasible set). Finally, we tried to prove submodularity but were unable to do so for our objective function especially because the denominator is dependent on the subset size. 
\subsection{Details on corralling}~\label{app:corralling}
\textit{Brief Overview on Corralling: }The algorithm CORRAL~\citep{agarwal2017corralling} is a meta-bandit algorithm that uses online mirror descent and importance sampling to sample different bandit algorithms that receive the reward. Using the rewards updates the probabilities used for sampling. The main objective is to achieve a regret which is as good as if the best base algorithm was run on its own.

To setup the context, we exactly reproduce the following excerpt, definition and theorems have been taken from~\citet{agarwal2017corralling}:
 
\textit{For an environment $\mathcal{E}$, we define the environment $\mathcal{E}'$ induced by importance weighting, which is the environment that results when importance weighting is applied to the losses provided by environment $\mathcal{E}$. More precisely, $\mathcal{E}'$ is defined as follows. On each round $t = 1, \dots, T$,}

\textit{1. $\mathcal{E}'$ picks an arbitrary sampling probability $p_t \in (0, 1]$ and obtains $(x_t, f_t) = \mathcal{E}(\theta_1, \dots, \theta_{t-1})$.}

\textit{2. $\mathcal{E}'$ reveals $x_t$ to the learner and the learner makes a decision $\theta_t$.}

\textit{3. With probability $p_t$, define $f_t'(\theta, x) = f_t(\theta, x) / p_t$ and $\theta_t' = \theta_t$; with probability $1 - p_t$, define $f_t'(\theta, x) = 0$ and $\theta_t' \in \Theta$ to be arbitrary.}

\textit{4. $\mathcal{E}'$ reveals the loss $f_t'(\theta_t, x_t)$ to the learner, and passes $\theta_t'$ to $\mathcal{E}$.}
\begin{definition}\label{def:stable}\citep{agarwal2017corralling}
    For some $\alpha \in (0, 1]$ and non-decreasing function $R: \mathbb{N}_+ \rightarrow \mathbb{R}_+$, an algorithm with decision space $\mathcal{O}$ is called $(\alpha, R)$-stable with respect to an environment $\mathcal{E}$ if its regret under $\mathcal{E}$ is $R(T)$, and its regret under any environment $\mathcal{E}'$ induced by importance weighting is
\[
\sup_{\theta \in \Theta} \mathbb{E} \left[ \sum_{t=1}^T f_t(\theta, x_t) - f_t(\theta, x_t) \right] \leq \mathbb{E}[\rho^\alpha] R(T) \qquad (2)
\]where $\rho = \max_{t \in [T]} 1/p_t$ (with $p_t$ as in the definition of $\mathcal{E}'$ above), and all expectations are taken over the randomness of both $\mathcal{E}'$ and the algorithm.
\end{definition}
Similar too most reasonable Base Algorithms it can be seen that the BSLB algorithm satisfies is $(1,\expectedregret)$-stable by rescaling the losses.  

We 
\begin{theorem}(Theorem 4 in \citep{agarwal2017corralling})
    For any $i \in [M]$, if base algorithm $\mathcal{B}_i$ (with decision space $\mathcal{O}_i$) is $(\alpha_i, R_i)$-stable (recall Defn.~\ref{def:stable}) with respect to an environment $\mathcal{E}$, then under the same environment CORRAL satisfies 
\begin{align*}
\sup_{\theta \in \Theta,   
 \pi \in \Pi} &\mathbb{E} \left[ \sum_{t=1}^T f_t(\theta_t, x_t) - f_t(\theta, x_t) \right] \\&= \widetilde{O} \left( \frac{M}{\eta} + T \eta \frac{|\mathcal{O}_{\pi_t}|}{\eta} + \frac{\alpha_i}{\eta \beta} R_i(T) \right), \qquad (3)
\end{align*}
where all expectations are taken over the randomness of CORRAL Algorithm, the base algorithms, and the environment. 
\end{theorem}
\begin{theorem}(Theorem 5 in \citep{agarwal2017corralling})\label{th:corrallingog}
    Under the conditions of Theorem 7, if $\alpha_i = 1$, then with $\eta = \min \left\{ \frac{1}{4 \theta R_i(T) \ln T}, \sqrt{\frac{1}{T}} \right\}$
CORRAL satisfies: 
$\sup_{\theta \in \Theta, \pi \in \Pi} \left[ \sum_{t=1}^T f_t(\theta_t, \pi_t) - f_t(\theta, \pi_t) \right] = \widetilde{O} \left( \sqrt{MT} + MR_i(T) \right)$.
\end{theorem}

\subsection{Proof of Theorem~\ref{th:corralling}}
\label{sec:corralling}
\begin{algorithm}
    \begin{algorithmic}[1]
        \STATE Input: Dimension $\dimarm$, Total Number of Rounds $\numrounds$, Regret Bound of Best Algorithm $R_{\text{best}}$ 
        \STATE Set Learning rate $\eta = \min\left(\frac{1}{40 \numrounds R_{\text{best}}},\sqrt{\frac{\lfloor \log_2(\dimarm)\rfloor}{\numrounds}}\right)$
        \STATE Set Exponential Grid $\sparsityk \in [1,2,\dots 2^{\lfloor \log_2(\dimarm)\rfloor}] $
        \STATE Initialize $\lfloor \log_2(\dimarm)\rfloor+1$ Base Algorithms one  for each sparsity parameter on an exponential grid, \texttt{BSLB}($\numroundsexplore = \constant\sparsityk^{1/3}\numrounds^{2/3}$)
        \STATE Sample $\numarms_{\text{sampled}} = \constant\dimarm^{1/3}\numrounds^{2/3}$ arms without replacement to be used as proxy samples.
        \STATE Run Corral$\left(\lfloor \log_2(\dimarm)\rfloor+1 \ \texttt{BSLB} \ \text{algorithms},\eta\right)$ from~\citet{agarwal2017corralling} with Base Algorithms and time horizon $\numrounds$. If an arm is suggested which is already pulled, pull an arm from the remaining set of arms uniformly at random. 
    \end{algorithmic}
    \caption{Corralling with Blocked Linear Sparse Bandits (\texttt{C-BSLB})}
    \label{alg:corral}
\end{algorithm}

Before presenting the proof, we clarify what we mean by the exponential scale with an example. For dimension $\dimarm = 1024$, the exponential scale will be $\sparsityk \in \{1,2,4,8,16,32,64,128,256,512,1024\}$, and we initialize a base algorithm each with the exploration period set according to the $\sparsityk$. 

\textit{Remark: }Also Step 5 and Step 6 in Algorithm~\ref{alg:corral} needs explanation. Note that the CORRAL algorithm as a whole has to respect the \textit{blocking} constraint. Even though CORRAL does not require the bandit algorithms to be run independently, we want to avoid changing CORRAL or the base algorithm. Instead we just change the arms that are available to sample by exploiting the fact that our base algorithm is a two step algorithm and each of the steps can be performed offline. For each base algorithm: 
\begin{enumerate}
    \item For the explore phase of the base algorithm we take arms from the intersection of the subset sampled in step 5 and the subset of arms which have not been sampled. 
    \item For the exploit phase if the chosen base algorithm provides an arm which has already been sampled by CORRAL. Then we provide the feedback corresponding to that arm. And then we pull an arm from the remaining set of arms without replacement.
\end{enumerate}

Note that with this modification, the exploration phase of each of the algorithm runs as if the algorithm was being run independently.
Hence the regret bounds for each individual base algorithm still holds. We can now prove Theorem~\ref{th:corrallingog}.

\newcommand{\sparsitys}{s}

We now use Theorem~\ref{th:corrallingog} with $M = \log_2 \lceil\dimarm \rceil+1$ algorithms. However if we simply apply the theorem we can bound with respect to the sparsity parameter which lies on the grid, $\sparsitys \in \{2^i\}_{i=0}^{\log_2 \lceil\dimarm \rceil}$,
\begin{align*}
\expectedregret \leq \Tilde{\orderof}(\sqrt{\log_2 \lceil\dimarm \rceil\numrounds} + \log_2 \lceil\dimarm \rceil\expectedregret_{\sparsitys})
\end{align*}
But the optimal sparsity parameter $\sparsity$ may not lie on the grid and we need to bound $\expectedregret_{\sparsity}$ in terms of $\expectedregret_{\sparsitys}$. To that end we prove the following lemma, 
\begin{lemma}
Let $\sparsity$ be the sparsity parameter at which the regret bound of~\ref{eq:regretbound}
is minimized. And let $\sparsitys \in \{2^i\}_{i=0}^{\log_2 \lceil\dimarm \rceil} $ be the parameter on grid which is closest to $\sparsity$ in absolute distance. Then the following holds (where $\probabilityexploresucc$ is the probability of exploration round succeeding at sparsity level $\sparsity$),
\begin{align*}
     \expectedregret_{\sparsitys} &\leq \sqrt{2\sparsityk^*}\expectedregret_{\sparsity} + \log_2(2\sparsityk^*)\orderof(1-\probabilityexploresucc)\\
\end{align*}
    
\end{lemma}
\begin{proof}
Let the bound on the expected regret for sparsity level $\sparsityk$ be given by $\expectedregret_\sparsityk$. 

From the statement of the theorem, we assume that for the optimal sparsity parameter $\sparsity$, the nearest parameter (on the exponential scale) is $\sparsitys$. 
$\sparsity$ lies in the interval $\left[\lceil\sparsitys/2\rceil, 2\sparsitys]\right]$ (otherwise $\sparsitys$ would not be the closest parameter on the exponential scale.). Therefore if we were perform a binary search for $\sparsity$, we would need at most $Y=\lfloor\log_2(4\sparsity)\rfloor$ queries to search for $\sparsity$. Let $\sparsityk_1^*,\sparsityk_2^*,\dots,\sparsityk_{Y}^*$ be the mid-points of these queries, where $\sparsityk_{Y}^*=\sparsity$. Now each of them is such that $\sparsityk_j^* = \alpha \sparsityk_{j-1}^*$, where $\alpha \in [0.75,1.25]$. 

First consider the case when $\alpha \in [0.75,1]$, then by substituting $\sparsityk = \lfloor\alpha\sparsityk\rfloor$,
in the regret bound of Theorem~\ref{th:regret}, the following inequalities can be obtained , 
\begin{align*}
   \expectedregret_{\lfloor\alpha\sparsityk\rfloor} &\leq  (\frac{1}{\alpha})^{1/2}\expectedregret_{\sparsityk} + (1-\probabilityexploresucc)\orderof(\numrounds) \leq \sqrt{2}\expectedregret_{\sparsityk} + (1-\probabilityexploresucc)\orderof(\numrounds).
\end{align*}

Now for the case, $\alpha\in [1,1.25]$, we substitute for $\sparsityk = \lceil\alpha\sparsityk\rceil$
\begin{align*}
    \begin{split}
         \expectedregret_{\lceil \alpha \sparsityk\rceil} &\leq (\alpha)^{1/3}\expectedregret_{\sparsityk} + (1-\probabilityexploresucc)\orderof(\numrounds)  \leq  \sqrt{2}\expectedregret_{\sparsityk} + (1-\probabilityexploresucc)\orderof(\numrounds)
    \end{split}
\end{align*}
The probability of success of each of them is $1-o(1)$ and $\log(4\sparsityk^*)$ times the probability of error is still $o(1)$.

Now we can take a cascade of products by decomposing $\expectedregret_{\sparsityk^*}$ using the above inequality in the direction of $\sparsityk_1^*,\sparsityk_2^*,\dots,\sparsityk_{Y}^*$. (i.e. we can decompose $\sparsity = \alpha_1 \alpha_2 \dots\alpha_Y \sparsityk$, 
\begin{align*}
    \expectedregret_{\sparsitys} &\leq \alpha  \expectedregret_{\sparsityk_1^*} + O(1-\probabilityexploresucc) \leq \dots \leq (\sqrt{2})^{Y}\expectedregret_{\sparsity}  + Y\orderof(1-\probabilityexploresucc)\\
     \expectedregret_{\sparsitys} &\leq 2\sqrt{\sparsity}\expectedregret_{\sparsity} + \log_2(4\sparsity)\orderof(1-\probabilityexploresucc)\\
\end{align*}
\end{proof}